\documentclass{article}

\usepackage[english]{babel}

\usepackage[letterpaper,top=2cm,bottom=2cm,left=3cm,right=3cm,marginparwidth=1.75cm]{geometry}

\usepackage{mathtools}
\usepackage{amsmath, amssymb}
\usepackage{graphicx}
\usepackage{amsthm}
\usepackage{amsfonts}
\usepackage[colorlinks=true, allcolors=blue]{hyperref}
\newcommand{\Rn}{\mathbb R^n}
\newcommand{\R}{\mathbb R}

\newcommand{\gamxnot}{\gamma_t(x_0)}

\newcommand{\DE}{\textrm{DE}}
\newcommand{\D}{\textrm{Diag}}
\newcommand{\logdet}{\mathrm{logdet}}
\newcommand{\tr}{\textrm{Tr}}

\newcommand\numberthis{\addtocounter{equation}{1}\tag{\theequation}}
\newcommand{\grad}{\nabla}

\DeclareMathOperator*{\E}{\mathbb{E}}
\newtheorem{theorem}{Theorem}
\newtheorem{corollary}[theorem]{Corollary}

\newtheorem{lemma}[theorem]{Lemma}

\newtheorem{remark}{Remark}
\newtheorem*{remark*}{Remark}
\newtheorem{definition}{Definition}

\title{Convergence of the Riemannian Langevin Algorithm}

\author{Khashayar Gatmiry\thanks{gatmiry@mit.edu} \and Santosh S. Vempala\thanks{vempala@gatech.edu}}
\begin{document}
\maketitle

\begin{abstract}

We study the Riemannian Langevin Algorithm for the problem of sampling from a distribution with density $\nu$ with respect to the natural measure on a manifold with metric $g$. We assume that the target density satisfies a log-Sobolev inequality with respect to the metric and prove that the manifold generalization of the Unadjusted Langevin Algorithm converges rapidly to $\nu$ for Hessian manifolds.  
This allows us to reduce the problem of sampling non-smooth (constrained) densities in $\R^n$ to sampling smooth densities over appropriate manifolds, while needing access only to the gradient of the log-density, and this, in turn, to sampling from the natural Brownian motion on the manifold. 
Our main analytic tools are (1) an extension of self-concordance to manifolds, and (2) a stochastic approach to bounding smoothness on manifolds. A special case of our approach is sampling isoperimetric densities restricted to polytopes by using the metric defined by the logarithmic barrier.   
\end{abstract}

\section{Introduction}

Sampling is a fundamental problem with connections to many areas of mathematics and algorithms. A natural approach to sampling from a desired distribution is {\em diffusion}, i.e., a stochastic process that converges to the target density in continuous time. This is a subject of classical study with many powerful and surprising theorems~\cite{bakry2014}. Sampling algorithms, which are inherently discrete, are also a major subject due to their generality and plethora of applications~\cite{heirendt2019creation,chalkis2020volesti,cousins2016practical}. The connections between the two are still being developed and the following is not well-understood:
\begin{center}
{\em When does the rapid convergence of a diffusion process result in an efficient sampling algorithm?}
\end{center}

In recent years, progress has been made on this via the Langevin Diffusion and associated Langevin algorithm, both in Euclidean space. For a target density proportional to $e^{-f(x)}$, Langevin Diffusion (LD) is the following stochastic process\footnote{We use $D$ for Euclidean derivative and reserve $\nabla$ for manifold derivative; when the manifold is $\R^n$, they coincide.}: 
\begin{equation}\label{LangevinSDE}
dX_t = -Df(X_t)dt + \sqrt{2}dW_t,
\end{equation}
where $W_t$ is the standard Wiener process.
LD converges to the target density in continuous time. The rate of convergence in relative entropy (or KL divergence) can be bounded  for densities satisfying a Log-Sobolev Inequality (LSI, defined in section~\ref{sec:isoperimetry}) with  parameter $\alpha > 0$:
\[
H_{\nu}(\rho_t) \le e^{-2\alpha t} H_{\nu}(\rho_0).
\]
When the target $e^{-f}$ is strongly logconcave and gradient Lipshitz, the work of ~\cite{roberts1996exponential,dalalyan2017theoretical} showed that the simple Euler discretization converges rapidly. This is given by
\[
x_{k+1} = x_{k} - h Df(x_k) + \sqrt{2h}Z 
\]
where $Z \sim N(0,I)$ is a standard Gaussian. The convergence rate was refined and improved in many subsequent papers~\cite{durmus2017nonasymptotic,durmus2019analysis}. 
It was later shown that this discretization gives an efficient algorithm for any density proportional to $e^{-f}$ where $Df$ is $L$-Lipshitz, with the rate of convergence polynomial in $L, n$ and $1/\alpha$~\cite{VW19}. Li and Erdogdu~\cite{LiErdogdu2020} extended this framework to the manifold given by a product of spheres, and ~\cite{WLP20} to more general manifolds under some assumptions. These results suggest a more comprehensive picture is waiting to be discovered. 

Diffusion can be generalized naturally to Riemannian manifolds. Given a manifold ${\mathcal M}$ defined by a local metric $g(x)$ for each $x\in {\mathcal M}$, 
one can define Riemannian Langevin Diffusion (RLD) in Euclidean coordinates as follows:
\begin{align}\label{EuclideanRLD}
    dX_t = (D \cdot (g^{-1} (X_t)) - g^{-1}(X_t)Df(X_t))dt + \sqrt{2g^{-1}(X_t)} dB_t,
\end{align}
where $dB_t$ is the standard Wiener process. If $f = 0$, then $X_t$ solving the above SDE is the canonical Brownian motion on the manifold with metric $g$. On the other hand, the above equation (including $f$) has stationary distribution $e^{-f}dx$. The benefit of this process is that one can choose the local metric to facilitate convergence and hence potentially extend the method to more general densities. This leads to our main motivating questions:
\begin{enumerate}
    \item Under what conditions does the discretization of RLD, the Riemannian Langevin Algorithm, converge?
    \item Can non-smooth and/or constrained distributions be sampled by mapping to appropriate manifolds and using RLA?  This is analogous to optimization where non-smooth convex optimization can be mapped to smooth convex optimization via self-concordant barriers (the Interior-Point Method~\cite{nesterov1994interior}). 
\end{enumerate}

For the first question, it is useful to consider RLD as a process on a manifold (the equivalence is shown formally in Lemma~\ref{lem:Euclidean-manifold}).
\begin{definition}[Riemannian Langevin Diffusion (RLD)]
Given a metric $g$ in the Euclidean chart $x \in \mathbb R^n$, let $\mathcal M$ be the manifold with measure $dv_g(x)$, whose density with respect to the Lebesgue measure in the Euclidean chart is $\sqrt{|\det(g)|}$.  Given a distribution with density $\nu = e^{-F}$, Riemannian Langevin Diffusion is the solution to the following SDE whose stationary distribution is $\nu dv_g$. 
\begin{align*}
    d\tilde X_t = (\nabla \cdot (g^{-1}(\tilde X_t)) - \nabla F(\tilde X_t))dt + \sqrt{2g^{-1}(\tilde X_t)}dB_t.
\end{align*}
\end{definition}
Here $\nabla \cdot$ denotes the divergence with respect to the manifold and it is applied separately to each row of a matrix 
(see Sec.~\ref{sec:prelims} for background). 

Convergence in continuous time holds for any target density satisfying LSI and for any Riemannian metric.
\begin{theorem}[RLD Convergence]\cite{bakry2014}
Suppose that the density $e^{-f(x)}dx = 
e^{-F}dv_g(x)$ has log-Sobolev constant $\alpha > 0$ with respect to the metric $g$. Then, the distribution $\rho_t$ obtained by RLD at time $t \ge 0$ satisfies
\[
H_{\nu}(\rho_t) \le e^{-2\alpha t} H_{\nu}(\rho_0). 
\]
\end{theorem}
When the metric $g$ is defined by the Hessian of a strictly convex function $\phi$ as $g(x) = D^2\phi(x)$, then the corresponding process is also called {\em Mirror Langevin Diffusion}, as it can be viewed as passing to the dual via the {\em mirror} map $y = D\phi(x)$. This is discussed in several recent papers~\cite{ZPFP20,AC21,LTVW2021}, which prove convergence under rather strong assumptions. 

The simple discretization of RLD is the Riemannian Langevin Algorithm, defined next.
\begin{definition}[Riemannian Langevin Algorithm (RLA)]
To sample from the distribution $\nu dv_g(x)$ over the manifold $\mathcal M$, we define the following iterative process for a fixed step size $\epsilon$: 
\begin{align*}
    &y =  \exp_{x_k}(-\epsilon\nabla F(x_k))\\
    &x_{k+1} = \mathfrak B(y, \epsilon) 
\end{align*}
where $\mathfrak B(x_0,t)$ samples from Brownian motion on the manifold, starting from $x_0$ after time $t$.
\end{definition}
In words, at any point $x$, the algorithm travels for some time $\epsilon$ along the geodesic with initial gradient $-\nabla F(x)$ to reach a point $y$, and then performs a Brownian motion on the manifold starting at $y$ for time $\epsilon$ (we describe Brownian motion on a manifold in Section~\ref{sec:brownian}). We note that in this paper we assume that we have access to an oracle that can sample Brownian motion on a manifold starting from a given point up to a given time increment (similar to~\cite{AC21}.)

Such ``geodesic walks" have been considered in the literature, e.g., in ~\cite{lee2017geodesic,GS2019,WLP20}, under various additional assumptions or with more algorithmic modifications. \cite{lee2017geodesic} applies a geodesic walk to sample uniformly from a polytope by mapping to the manifold given by a strictly convex barrier function, but their process requires a Metropolis step which explicitly computes the probability of transition and uses a filter to ensure the desired stationary distribution. \cite{GS2019} show that geodesic ``ball" walk converges rapidly to the manifold measure for geodesically strictly convex subsets of a positive curvature manifold subject to additional conditions. ~\cite{LiErdogdu2020} obtains strong estimates of the rate of convergence for some classes of functions on the manifold defined by a product of spheres. 

The discretization of Mirror Langevin Diffusion has been an active research topic in recent years. It was first proposed as a sampling algorithm by ~\cite{ZPFP20} and later studied by ~\cite{Hsieh18, AC21, jiang2021, LTVW2021}. The results of \cite{Hsieh18,LTVW2021} assume that the convex function whose Hessian defines the metric satisfies a property they call {\em modified self-concordance}. Roughly speaking it requires that the square-root of the metric changes slowly (in Frobenius norm), namely the change is bounded by a constant times the change in the gradient in Euclidean norm. They also require the function of interest to be smooth and strongly convex. Even with these assumptions, convergence to the target distribution with vanishing bias was established very recently~\cite{LTVW2021}, using an analysis technique developed in~\cite{LZT21}. Unfortunately, the modified self-concordance parameter is not affine-invariant (even though the process itself is!) and is unbounded for the log barrier, even in two dimensions~\cite{LTVW2021}. The work of \cite{AC21} shows how to use self-concordance of the convex barrier function together with the assumptions that the target $f$ is convex and Lipshitz to reduce the sampling problem to that of simulating Brownian motion in the barrier metric. Relaxing their assumptions, while maintaining an affine-invariant analysis is our motivating open problem. 

A natural candidate for the condition on the metric is self-concordance, as usually defined in optimization, i.e., the third directional derivative of the convex function can be bounded by a constant times the second directional derivative to the power of $3/2$, an affine-invariant condition~\cite{nesterov1994interior}. Unfortunately, this does not imply that the square-root of the inverse of the metric satisfies self-concordance (or changes slowly), leading to the above mentioned modified self-concordance assumption. We get around this by a natural extension of standard self-concordance to manifolds.

Here we prove a general theorem based on self-concordance parameters of the metric $g$ defining the manifold and smoothness parameters of $F$. We define {\em second-order self-concordance} for a metric, an extension of the standard definition to one higher derivative. 
Our main result can be interpreted as a reduction from the problem of sampling a general density on a manifold to sampling Brownian motion on the manifold. As a special case of this theorem, it follows that for smooth functions $f$ satisfying LSI with respect to the log barrier metric, RLA can sample from the density proportional to $e^{-f}$ restricted to polytopes with the Langevin algorithm. Previous work was able to sample smooth, strongly logconcave densities or smooth, unconstrained densities satisfying isoperimetry, or convex and Lipshitz densities. To describe the results more precisely, we first define manifold self-concordance. Recall that the local norm of a vector $v$ in a manifold with metric $g$ is given by $\|v\|_g^2 = v^Tgv$.  

\begin{definition}
(Second-order self-concordance) Let $\phi:\Rn \rightarrow \R$ be a thrice-differentiable function and $g(x)=D^2\phi(x)$. 
Then $g$ is $\gamma_1$-self-concordant if
\begin{align*}
    -\gamma_1 \|v\|_g g \preceq Dg(v) \preceq \gamma_1 \|v\|_g g.
\end{align*}

We say that $g$ is $(\gamma_1, \gamma_2, \gamma_3)$-second-order self-concordant if in addition to the above (standard) self-concordance, it satisfies
\begin{align*}
    \langle Dg(v), Dg^{-1}(w)\rangle \preceq n \gamma_2^2 \|v\|_g \|w\|_g
\end{align*}
and
\begin{align*}
    -\gamma_3^2 \|v\|_g\|w\|_g g \preceq  D^2g(v,w) \preceq \gamma_3^2 \|v\|_g\|w\|_g g.
\end{align*}
\end{definition}

\begin{remark}
For the manifold defined by the log barrier in a polytope, the parameters $\gamma_1, \gamma_2, \gamma_3$ are all bounded by absolute constants.
\end{remark}

We can now state our main theorem. For convenience, we assume that all self-concordance and Lipshitzness parameters are at least $1$.
The assumption $\alpha \le 1$ below is for simplicity of exposition, and one can simply replace $\alpha$ with $\min\{1, \alpha\}$. 
\begin{theorem}[General Hessian Manifolds]\label{thm:generalmanifold}
Let $\mathcal M$ be a $(\gamma_1, \gamma_2, \gamma_3)$-second-order self-concordant Hessian manifold. Consider 
the distribution $d\nu(x)$ on $\mathcal M$ with density $e^{-F}$ w.r.t to the manifold volume measure. Assume that $\nu$ satisfies the log-Sobolev inequality with parameter $\alpha \le 1$. Let the function $F$ be $L_2$-gradient Lipschitz, $L_3$-Hessian Lipschitz on $\mathcal M$. Then, there exist universal constants $c_1, c_2$, such that after $k$ steps of the Riemannian Langevin Algorithm with step-size (time parameter) $\epsilon$, starting from a distribution with density $\rho_0$ which is square integrable with respect to the manifold measure and has finite second moment (with respect to an arbitrary, fixed point on $\mathcal M$), for any $\epsilon \leq 1$, the distribution $\rho_k$ satisfies  
\begin{align}
H_\nu(\rho_{k}) &\leq e^{-\frac{3}{16}\alpha \epsilon k} H_\nu(\rho_0) + c_1\Big[ n^{5/2} (\gamma_1^2 + \gamma_2^2 + \gamma_3^2)L_2
     + \sqrt nL_2^2 + nL_3 \Big]\epsilon/\alpha.
\end{align}
 Hence, for any target accuracy $\delta > 0$, after $ k = O(\frac{1}{\alpha \epsilon} \log(2H_\nu(\rho_0)/\delta)))$ steps of RLA with small enough step-size $\epsilon$, namely 
\begin{align}
    \epsilon  \le c_2 
    \frac{\delta \alpha}{ n^{5/2} (\gamma_1^2 + \gamma_2^2 + \gamma_3^2)L_2
     + \sqrt nL_2^2 + nL_3 },
\end{align}
the distribution $\rho_k$ satisfies $H_\nu(\rho_k) \le \delta$.
\end{theorem}

The theorem recovers as special cases some results in the literature\footnote{The main theorem of \cite{WLP20} claims a result in a similar spirit with additional parameters $K_1,K_2,K_3,K_3$ as well as curvature bounds; we avoid these parameters here and do not see how they could be bounded.}. It also extends the reach of Langevin style algorithms to constrained distributions as illustrated in the corollaries below. 
We note that our gradient and Hessian Lipschitness assumptions are the natural ones in the manifold.

\begin{remark}
In general, we work with Gibbs form of the distribution $e^{-F} dv_g(x)$, which can also be presented by its density with respect to the Lebesgue measure in the Euclidean chart as $e^{-f}dx$, where we have the relation 
\begin{align*}
    f = F - \frac{1}{2}\log \det(g(x)).
\end{align*}
Note that logarithmic Sobolev parameter is determined by the distribution and the metric $g$, hence remains the same property for either of the representations with $f$ or $F$ with respect to the same metric $g$.
\end{remark}

\paragraph{Polytope with log barrier. }
We apply our result to the special case when the manifold is the open set inside a polytope $Ax \geq b$, and the metric at point $x$ is given by the Hessian $D^2 \phi(x)$, where $\phi$ is the logarithmic barrier function: 
\begin{align*}
    \phi(x) = - \sum_i \log(a_i^T x - b_i).
\end{align*}
The Hessian can be written as $D^2\phi(x) = A_x^T A_x$, where $A_x$ is equal to the matrix $A$ whose $i$th row is reweighted by $1/(a_i^Tx - b_i)$. We denote the leverage score vector of $A_x$ by $\sigma_x$, i.e., 
\[
\sigma_x = \mbox{diag}(A_x(A_x^\top A_x)^{-1}A_x^\top)
\]
Then, Brownian motion with this geometry is given by
\begin{align*}
    dX_t = \sqrt{2(A_t^T A_t)^{-1}} dB_t - \frac{1}{2}(A_t^T A_t)^{-1} A_t^T \sigma_t dt,
\end{align*}
where we denoted $A_{X_t}$ and $\sigma_{X_t}$ by $A_t$ and $\sigma_t$ for brevity.

\begin{corollary}
Let the manifold $\mathcal M$ be the open set inside a polytope with the log barrier metric. Consider the distribution with density $e^{-f}$ with respect to the Lebesgue measure inside the polytope s.t. the corresponding function $F$ defining the density $e^{-F}$ on $\mathcal M$ is $L_2$-gradient Lipschitz and $L_3$-Hessian Lipschitz  (i.e. taking the metric into account), and also satisfies a log-Sobolev inequality with constant $\alpha$ with respect to the log barrier metric. Assuming $L_2, L_3 \geq 1$ and $\alpha, \delta \leq 1$ for simplicity of exposition, there is a constant $c$ s.t., RLA with step-size  
\begin{align*}
    \epsilon \leq 
    \frac{\delta \alpha}{ n^{5/2}L_2
     + \sqrt nL_2^2 + nL_3},
\end{align*}
starting from a distribution $\rho_0$, samples from a distribution whose KL-divergence with respect to the target $\nu$ is at most $\delta$, and the number of iterations bounded by
\begin{align*}
    O\Big( \frac{1}{\alpha \epsilon}\log(2H_\nu(\rho_0)/\delta) \Big).
\end{align*}
\end{corollary}

To make the statement above explicit in terms of the original function $e^{-f}$, we bound all the required smoothness parameters for the log barrier in a polytope and obtain the following result. We note that while we focus on a polytope, the techniques extend to more general convex bodies with analogous barriers (e.g., as investigated in~\cite{narayanan2016randomized}).

\begin{corollary}
Suppose the target distribution is $d\nu(x) = e^{-f} d(x)$ for $f$ inside a polytope $Ax \geq b$ of diameter $R$, assuming $R \geq 1$ for simplicity, such that $f$ is $\ell_2$-gradient Lipschitz, and $\ell_3$-Hessian Lipschitz. Moreover, suppose $\nu$ satisfies a log-Sobolev inequality with constant $\alpha \le 1$ with respect to the log barrier metric. For any $\delta < 1$, the query complexity of RLA to reach at most $\delta$ error in KL divergence starting from density $\rho_0$ is $O\Big(\frac{1}{\alpha \epsilon}\log \frac{2H_\nu(\rho_0)}{\delta}\Big)$, for step size 
\begin{align*}
    \epsilon \leq 
    \frac{\delta\alpha}{n^{7/2} + n^{5/2}\ell_2R + \sqrt n \ell_2^2 R^2 + n\ell_3 R^{3/2}}.
\end{align*}
\end{corollary}

To prove these corollaries, we bound the self-concordance parameters of the log barrier metric and bound the manifold smoothness parameters of $F$ in terms of the smoothness parameters of the given function $f$ (Section~\ref{app:euclideantomanifold}). 

\subsection{Approach}

Langevin diffusion, even for Riemannian manifolds, converges exponentially fast to the target distribution in KL divergence for any target satisfying LSI. To turn this into an algorithm, we bound the ``discretization error", i.e., how much the time-discretized algorithm deviates from the continuous process in terms of relative entropy. The discretization proceeds by taking a step defined by the gradient of $F$ at the current point for some positive time, and then performing a Brownian motion for the same length of time.

In the Euclidean setting, the error of discretization can be bounded using smoothness of the target (gradient of $f$ is Lipshitz) and LSI itself~\cite{VW19}. In the Riemannian setting, there is a new source of error, namely the gradient step is along the geodesic with initial velocity determined by the gradient of $F$. The error analysis of ~\cite{VW19} was cleanly generalized by \cite{LiErdogdu2020} with an additional Riemannian discretization error term. They then focused on the manifold given by a product of spheres, where quantities of interest can be written quite explicitly.   

The general case presents formidable further challenges. This is already illustrated in the papers on Mirror Langevin, which introduce the modified self-concordance condition making the applicability of the analysis rather restricted. Our goal here is to avoid such a condition and preserve the affine-invariance of the process in the analysis.

The discretization error (Lemma ~\ref{lem:dis-error}) can be viewed as follows for $h(x) = \log(\rho_t(x)/\nu(x))$
\begin{equation}\label{eq:error}
\E_{x_0 \sim \rho_0}\left(\E_{x \sim \rho_{t|0}} \langle \grad_x h(x), \nabla F(x) \rangle
      - \langle \grad_{\gamxnot} \E_{\rho_{t|0}} h(x), b(t,x_0) \rangle\right)
\end{equation}
where $\rho_{t|0}$ is the distribution obtained by running Brownian motion for time $t$ starting at $\gamma_t(x_0)$, the point obtained by going along the geodesic starting at $x_0$ with initial velocity $-\nabla F(x_0)$ for time $t$.
The vector $b(t,x_0)$ is the parallel transport of $\grad F(x_0)$ from $x_0$ to $\gamma_t(x_0)$. In $\R^n$, it would simply be $\grad F(x_0)$ and could be factored. Here we will have to bound its change along the geodesic.

Next, note that in the above expression, it would be useful to exchange the gradient and the inner expectation in the second term. In $\R^n$, we can write $\nabla \E(h(x)) = \E(\nabla h(x))$ for an  arbitrary function $h$, but this is false in general on a manifold, since the (i) gradient is locally defined and (ii) expectation doesn't make sense as $\nabla h(x)$ is a point in the tangent space at $x$. 

To address this, we consider a map $H_t:{\mathcal M} \rightarrow {\mathcal M}$,  $H_t(x_0) = X_t^{x_0}$ is a (random) map from $x_0$ to $X_t$, the value of the Brownian process started at $x_0$ at time $t$. Recall that
\[
dX_t = A(X_t)dB_t + Z(X_t)dt
\]
where $A(x) = \sqrt{2g^{-1}(x)}$ and $Z(x) = \nabla \cdot (g^{-1}(x))$ are defined by the metric. 

Then we write
\[
 d\E_{\rho_{t|0}}(h(x)) = \E_{\rho_{t|0}} dh(x)DH_t(x_0),
\]
where $DH_t(x)$ is the Jacobian of the random map $H$. To analyze this, we use an extension of a theorem due to \cite{elworthy1994formulae}, which will imply that  
\begin{align*}
    \langle \grad_{\gamxnot} \E_{\rho_{t|0}} h(x), b(t,x_0) \rangle = \mathbb E_{V_t} \langle \grad h(x) , V_t\rangle,
\end{align*}
for a stochastic process $V_t$ with SDE
\begin{align}
    dV_t = DA(X_t)(V_t) dB_t + DZ_t(V_t) dt.\label{eq:Vsde}
\end{align}
where $V_0 = b(t,x_0)$.

Here $V_t$ will be the action of the Jacobian of the random map $H_t$ on $b(t,x_0)$. 

Applying this, we get that the discretization error is bounded by 
\begin{align*}
\E_{\rho_0}\E_{\rho_{t|0}} \langle \grad h(X_t),V(X_t)-\nabla F(X_t)\rangle
&\leq \E_{\rho_t} \frac{1}{4} \|\nabla h(X_t)\|^2  +  \E_{\rho_t} \|V(X_t) - \nabla F(X_t)\|^2
\end{align*}
using Cauchy-Schwarz and the AM-GM inequality. The first term, with our choice of $h(x) = \log(\rho_t(x)/\nu(x))$ will be a quarter of the relative Fisher information, which is good --- the continuous process decreases the relative entropy by the relative Fisher information and so we still get a decrease overall. Now the second term needs to be bounded, this takes up most of our effort, and will crucially use the self-concordance of $g$ as well as the smoothness of $F$. We will do this by showing that the squared norm of $V$ grows slowly and so does that of $\nabla F$.

\section{Preliminaries}\label{sec:prelims}

\subsection{Notation} 
We use $D$ for the Euclidean derivative, $\partial, \partial_j$ for Euclidean partial derivatives, and $\nabla$ (or $\grad$) for manifold derivative. So, e.g., $D \log\det g(x) = \langle g^{-1}(x), Dg(x) \rangle$ where $g(x)$ is a matrix defined by $x$ and $Dg$ is a third-order tensor, so that the RHS becomes a vector. The directional derivative in the direction $y$ would be
$D\log \det g(x)(y)  = Tr(g^{-1}(x) Dg(x)(y))$,
We interchangeably refer to the gradient and Hessian of $F$ over the manifold by $\grad F, \text{Hess}f$ and $\nabla F, \nabla^2 F$ respectively. 

We use $\nu$ for the steady state or target distribution. 
The $KL$-divergence (relative entropy) of a distribution $\rho$ with respect to another distribution $\nu$ is denoted by $H_\nu(\rho)$ and defined as 
\[
H_\nu(\rho) = \E_{x\sim \rho}(\log(\rho(x)/\nu(x))) = \int \rho(x)\log\frac{\rho(x)}{\nu(x)} \, dv_g(x).
\]
The relative Fisher information is denoted by $I_\nu(\rho)$ and defined as
\begin{align}
    I_\nu(\rho) = \int \rho \|\grad \log(\frac{\rho}{\nu})\|^2 dv_g(x).
\end{align}
We use $\gamma_t(x)$ to denote the point on the manifold obtained after going on a geodesic $\gamma$ starting from $x$ after time $t$. We also use $\gamma_t^*(\rho)$ to denote the push forward measure of $\rho$ under the map $\gamma_t$.

\subsection{Basic Manifold definitions}
\paragraph{Abstract Manifold and Charts.}
 A manifold $\mathcal M$ is a topological space such that for each point $p \in \mathcal M$, there is in an open neighborhood $W \subset \mathcal M$ of $p$ and a homeomorphism $x:W\rightarrow U$, where $U$ is an open subset of $\mathbb R^n$. Such a specific map is called a chart, usually denoted as $x(p) \in \mathbb R^n$. Here we refer to a point on the manifold simply by its representation in its Euclidean chart. 
For each point $p \in \mathcal M$, the tangent space $T_p(\mathcal M)$ is an $n$-dimensional vector space whose elements are partial derivative operators of functions over $\mathcal M$,  or more precisely, linear functionals that obeys Leibniz' rule, i.e. for $V \in T_p(\mathcal M)$ and $f, g$ arbitrary functions on $\mathcal M$
\begin{align*}
   v(\alpha f + \beta g)(p) = \alpha v(f)(p) + \beta v(g)(p),\\
   v(fg)(p) = v(f)(p)g(p) + f(p)v(g)(p).
\end{align*}
For any local chart $x$ around $p$, the set of partial derivatives $\{\partial x_i\}_{i=1}^n$ in that chart is a natural basis for $T_p(\mathcal M)$.

\paragraph{Cotangent Space.}
 The space of linear functionals over $T_p(\mathcal M)$ is the dual or cotangent space denoted by $T^*_p(\mathcal M)$. To every basis $\{e_i\}_{i=1}^n$ of $T_p(\mathcal M)$, relates the dual basis $\{e^*_i\}_{i=1}^n$ which is a basis for $T_p^*(\mathcal M)$, defined by
\begin{align*}
    \forall i, j: \ e^*_i(e_j) = \delta_{ij}.
\end{align*}
The dual basis of $\{\partial x_i\}$ is denoted by $\{dx_i\}_{i=1}^n$.

\paragraph{Differential.}
For a map $f: {\mathcal M} \rightarrow \mathcal N$ between two manifolds, the differential $df_p$ at some point $p \in \mathcal M$ is a linear map from $T_p(\mathcal M)$ to $T_{f(p)}(\mathcal N)$ with the property that for any curve $\gamma(t)$ on $\mathcal M$ with $\gamma(0) = p$, we have
\begin{align}
    df(\frac{d}{dt}\gamma(0)) = \frac{d}{dt}f(\gamma)(0).\label{eq:prop}
\end{align}
. As a special case, for a function $f$ over the manifold, the differential $df$ at some point $p \in \mathcal M$ is a linear functional over $T_p(\mathcal M)$, i.e. an element of $T^*_p(\mathcal M)$. Writing~\eqref{eq:prop} for curve $\gamma_i$ with $\frac{d}{dt}\gamma_i(0) = \partial x_i$, testing property~\eqref{eq:prop}, we see
\begin{align*}
    df(\partial x_i) = \frac{d}{dt} f(\gamma_i(t))\Big|_{t = 0} 
    = \frac{\partial f}{\partial x_i}(\gamma_i(0)).
\end{align*}
for every $i$, one can see $df = \sum_i \frac{\partial f}{\partial x_i}dx_i$.

\paragraph{Vector field.} 
A vector field $V$ is a smooth choice of a vector $V(p) \in T_p(\mathcal M)$ in the tangent space for all $p \in \mathcal M$. 

\paragraph{Metric and inner product.} 
A metric tensor on the manifold $\mathcal M$ is simply a smooth choice of a symmetric bilinear map over $\mathcal M$. Alternatively, the metric or dot product $\langle ,\rangle$ can be seen as a bilinear map over the space of vector fields with the tensorization property, i.e. for vector fields $V, W, Z$ and functions $\alpha, \beta$ over $\mathcal M$:
\begin{align}
    &\langle V+W, Z\rangle = \langle V , Z\rangle + \langle W, Z\rangle,\\
    &\langle \alpha V, \beta W\rangle = \alpha \beta \langle V, W\rangle.
\end{align}

\subsection{Manifold Derivatives, Geodesics, Parallel Transport}
\subsubsection{Covariant derivative}
Given two vector fields $V$ and $W$, the covariant derivative, also called the Levi-Civita connection $\nabla_{V}W$ is a bilinear operator with the following properties:
\begin{align*}
    \nabla_{\alpha_1 V_1 + \alpha_2V_2}W = \alpha_1\nabla_{V_1}W + \alpha_2\nabla_{V_2}W,\\
    \nabla_{V}(W_1 + W_2) = \nabla_V(W_1) + \nabla_V(W_2),\\
    \nabla_V(\alpha W_1) = \alpha \nabla_V(W_1) + V(\alpha)W_1.
    \end{align*}
Importantly, the property that differentiates the covariant derivative from other kinds of derivaties over manifold is that the covariant derivative of the metric is zero, i.e., 
$\nabla_V g = 0$ for any vector field $V$. In other words, we have the following intuitive rule:
\begin{align*}
    \nabla_V \langle W_1, W_2 \rangle 
    = \langle \nabla_V W_1, W_2\rangle + \langle  W_1, \nabla_V W_2\rangle.
\end{align*}
Moreover, the covariant derivative has the torsion free property, meaning that for vector fields $W_1, W_2$:
\begin{align*}
    \nabla_{W_1}W_2 - \nabla_{W_2}W_1 = [W_1, W_2],
\end{align*}
where $[W_1, W_2]$ is the Lie bracket of $W_1, W_2$ defined as the unique vector field that satisfies
\begin{align*}
    [W_1, W_2]f = W_1(W_2(f)) - W_2(W_1(f))
\end{align*}
for every smooth function $f$.

In a local chart with variable $x$, if one represent $V = \sum V_i \partial x_i$, where $\partial x_i$ are the basis vector fields, and $W = \sum W_i \partial x_i$, the covariant derivative is given by
\begin{align*}
    \nabla_V W & = \sum_i V^i\nabla_i W =
    \sum_i V^i \sum_j \nabla_i (W^j \partial x_j) \\
    & = \sum_i V^i \sum_j \partial_i(W^j) \partial x_j) + \sum_i V^i \sum_j W^j \nabla_i \partial x_j)\\
    & =  \sum_j V(W^j) \partial x_j) + \sum_i \sum_j V^i W^j \sum_k \Gamma_{ij}^k \partial x_k)
    = \\
    & = \sum_k (V(W^k) + \sum_i \sum_j V^i W^j \Gamma_{ij}^k )\partial x_k.
\end{align*}
where Christoffel symbols $\Gamma_{ij}^k$ are the representations of the Levi-Cevita derivatives of the basis $\{\partial x_i\}$:
\begin{align*}
    \nabla_{\partial x_j} \partial x_i = \sum_{k} \Gamma_{ij}^k \partial x_k
\end{align*}
and are given by the following formula:
\begin{align*}
    \Gamma_{ij}^k = \frac{1}{2}\sum_{m} g^{km}(\partial_j g_{mi} + \partial_i g_{mj} - \partial_m g_{ij}).
\end{align*}
Above, $g^{ij}$ refers to the $(i,j)$ entry of the inverse of the metric. 

\subsubsection{Parallel Transport}
The notion of parallel transport of a vector $V$ along a curve $\gamma$ can be generalized from Euclidean space to a manifold. On a manifold, parallel transport is a vector field restricted to $\gamma$ such that $\nabla_{\gamma'}(V) = 0$. By this definition, for two parallel transport vector fields $V(t), W(t)$ we have that their dot product $\langle V(t), W(t)\rangle$ is preserved, i.e., $\frac{d}{dt} \langle V(t), W(t)\rangle = 0$. 

\subsubsection{Geodesic}
A geodesic is a curve $\gamma$ on $\mathcal M$ is a ``locally shortest path", i.e., the tangent to the curve is parallel transported along the curve: $\nabla_{\dot \gamma}\dot \gamma = 0$ ($\dot \gamma$ denotes the time derivative of the curve $\gamma$.) Writing this in a chart, one can see it is a second order nonlinear ODE 
which locally has a unique solution given initial location and speed. 

\begin{align}
    \frac{d^2\gamma_k}{dt^2}(t) = -\sum_{i,j} \frac{d\gamma_i}{dt} \frac{d \gamma_j}{dt} \Gamma_{ij}^k, \ \ \forall k.
\end{align}

\subsubsection{Exponential Map}
The exponential $\exp_p(v)$ at point $p$ is a map from $T_p(\mathcal M)$ to $\mathcal M$, defined as the point obtained on a geodesic starting from $p$ with initial speed $v$, after time $1$. We use $\gamma_t(x)$ to denote the point after going on a geodesic starting from $x$ with initial velocity $\nabla F$, after time $t$. 

\subsubsection{Gradient}
For a function $F$ over the manifold, it naturally acts linearly over the space of vector fields: given a vector field $V$, the mapping $V(F)$ is linear in $V$. Therefore, by Riesz representation theorem, there is a vector field $W$ where
\begin{align*}
    V(F) = \langle W,V \rangle 
\end{align*}
for any vector field $V$. This vector field $W$ is defined as the gradient of $F$, which we denote here by $\nabla F$ or $\grad F$. To explicitly derive the form of gradient in a local chart, note that
\begin{align*}
    V(F) = (\sum_{i} \partial x_i V_i)(F) = \sum_i V_i \partial_i F = \sum_i V_i  \sum_{j,m,i} V_j g_{jm} g^{mi}\partial_i F = \langle V, \sum_{i}g^{mi} \partial_i F\rangle.
\end{align*}
The above calculation implies
\begin{align*}
    \nabla F = \sum_m (\sum g^{mi} \partial x_i F)\partial x_m.
\end{align*}
So $ g^{-1}DF$ is the representation of $\nabla F$ in the Euclidean chart with basis $\{\partial x_i\}$.

\subsubsection{Manifold Divergence} 
For the divergence of a vector field $F$ on a manifold, we have
\begin{align*}
    \nabla \cdot F &= \sum_i dx_i(\nabla_i F)\\
    &= \sum_i dx_i(\sum_j \partial_i F_j\partial x_j) + \sum_i dx_i((\sum_j F_j  \Gamma^k_{ij})\partial x_k)\\
    &= \sum_i \partial_i F_i + \sum_i (\sum_j F_j  \Gamma^i_{ij}).    
\end{align*}
Note that by the formula for the Christoffel symbol $\Gamma^{k}_{ij}$ for $k=i$ we get
\begin{align*}
\sum_i\Gamma^{i}_{ij} &= \sum_i \frac{1}{2}\sum_m g^{im} \partial_j g_{im}\\
&= \frac{1}{2} \tr(g^{-1}\partial_j g)\\
&= \frac{1}{2} \frac{\partial_j \det g}{\det g}. 
\end{align*}
Putting this back in the formula for divergence, we get
\begin{align*}
\nabla \cdot F &= \sum_i \partial_i F_i + \sum_j F_j  (\sum_i\Gamma^i_{ij})\\
& = \frac{1}{\sqrt{|g|}} \sum_i \left(\sqrt{|g|}\, \partial_i F_i +   \frac{\partial_i |g|}{2\sqrt{|g|}}  F_i\right)\\
&= \frac{1}{\sqrt{|g|}}\sum_{i}\partial_i(\sqrt{|g|}\,  F_i).
\end{align*}

\subsubsection{Higher order differentiated tensors}
As we described in the previous part, gradient of function $F$ can be regarded as an element of the cotangent space $T^*_x(\mathcal M)$. Now one can differentiate this tensor w.r.t the Levi-Civita connection to the Hessian which is a 0-2 tensor:
\begin{align*}
\nabla^2F(V_1, V_2) = \text{Hess}(F)(V_1, V_2) = V_1(\langle \grad F, V_2\rangle) - \langle \grad F, \nabla_{V_1} V_2\rangle  = \langle \nabla_{V_1} \grad F, V_2\rangle.
\end{align*}
Repeating the same process, one can differentiate the Hessian of $F$ to obtain. a 0-3 order tensor, which can be regarded as the manifold analog of the third order tensor of partial derivatives of $F$ in the Euclidean space:
\begin{align*}
    \nabla^3 F(X_1, X_2, X_3) = X_1(\nabla^2F(X_2, X_3)) - \nabla^2 F(\nabla_{X_1}X_2, X_3) - \nabla^2F(X_2, \nabla_{X_1}X_3).
\end{align*}

In general, we think of $F$ as a function on the manifold, and we assume bounds on the tensors $\nabla F$, $\nabla^2 F$, and $\nabla^3 F$ as for all vector fields $V_1, V_2, V_3$:
\begin{align*}
    &|\nabla F(V_1)| \leq L_1\|V_1\|,\\
    &|\nabla^2F(V_1, V_2)| \leq L_2\|V_1\|\|V_2\|\\
    & |\nabla^3F(V_1, V_2, V_3)| \leq L_3\|V_1\|\|V_2\|\|V_3\|.
\end{align*}
Note that the above norms are all the norm of a vector in the tangent space of the manifold, hence taking into account the metric.

\subsection{Brownian Motion on Manifold}\label{sec:brownian}
The metric $g$ in the Euclidean chart defines a manifold $\mathcal M$. The Brownian motion over $\mathcal M$ is formally defined in a local chart as the solution to the following   It\^{o} Stochastic differential Equation:
\begin{align}
    dX_t =  Z(X_t) dt + A(X_t) dB_t \label{eq:manifold_Brownian},
\end{align}
for 
\begin{align*}
    A(x) &= \sqrt{2g^{-1}(x)},\\
    Z(x) &= \nabla \cdot (g^{-1}).
\end{align*}
$dX_s = \sqrt{2g^{-1}}dB_s + \nabla \cdot (g^{-1})dt$ is the Brownian motion over the manifold. Note that Brownian motion on $\mathcal M$ is just the RLD for $F = const$. 
The symbol $\nabla \cdot$ above refers to divergence over the manifold applied to each row of the metric $g$, 
which in the Euclidean chart, for any vector field $V = \sum_i V_i \partial_{x_i}$, is given by
\begin{align}
    \nabla \cdot V = \sum_j \frac{1}{\sqrt{|g|}}\partial_j(\sqrt{|g|} V_j)
\end{align}
where $|g| = |\det g|$. Applied to $g^{-1}$, we have
\begin{align}
    \nabla \cdot (g^{-1})_i = \sum_j \frac{1}{\sqrt{|g|}}\partial_j(\sqrt{|g|} g^{ij}).
\end{align}
The drift term $Z(x)$ can alternatively be written as~\cite{hsu1988brownian}
\begin{align}\label{eq:zequationn}
    Z(x) = \nabla \cdot (g^{-1}) = -\sum_{ij} \Gamma_{i,j}^k g^{ij},
\end{align}
The symbols $\Gamma_{ij}^k$ refer to the Christoffel symbols regarding the covariant derivative over the manifold that are discussed in the preliminary section.

Brownian motion over the manifold is closely related to the heat equation over the manifold; namely, the Kolmogorov forward and backward equations for this process both reduces to the heat equation over $\mathcal M$. We should note that the conventional way to define the Brownian is in a way that the infinitesimal operator of the process becomes half of the Laplace Beltrami operator in $\mathcal M$, which differs from our definition by removing the factor $\sqrt 2$ and adding a factor of $1/2$ to the drift.  In particular, the infinitesimal operator of this process is the Laplace-Beltrami over $\mathcal M$, which is stated and proved in the following Lemma.

We start by recalling the definition of infinitesimal operator.
\begin{definition}
For a Markov chain with semigroup $(P_t)_{t\geq 0}$, one can associate an operator $L$, called the infinitesimal operator, which is the generalization of the Laplacian in the discrete case, defined on function $f$ as
\begin{align*}
    Lf = \lim_{t \rightarrow 0} \frac{1}{t}(P_tf -f). 
\end{align*}
The domain of the operator $L$ is exactly the set functions that we are working, in domain of $Q_t$ (usually bounded functions in some $\ell_p$ space) for which the above limit is defined. It is often the case that the domain of $L$ does not include the whole domain of $Q_t$. Also, the operator $L$ is usually unbounded, while $(Q_t)_t$ are bounded. As an example, the infinitesimal generator of Brownian motion is the Laplacian operator (we skip the detailed discussion on the domain of $L$). 
\end{definition}

Now we state a lemma about the infinitesimal operator of Brownian motion on a manifold.

\begin{lemma}
The infinitesimal operator of the above process is the Laplace-Beltrami operator $\Delta$ over the manifold $\mathcal M$, which is defined by
\begin{align}
    \Delta f = \nabla \cdot (\nabla f) = \frac{1}{\sqrt {|g|}}\sum_{i} \partial_{i}(\sum_j \sqrt{|g|} g^{ij} \partial_j f),
\end{align}
where $|g|$ is the determinant of $g$.
\begin{proof}
Using   It\^{o}'s Lemma, for a twice-differentiable function $h$, we have
\begin{align*}
    dh(X_t) & = DhdX_t + \frac{1}{2}\sum_{ij }D^2h(X_t)_{ij}[{X}_i, {X}_j]_t  dt\\
    & = (Dh)^TdB_t + (Dh)^T \nabla \cdot (g^{-1}) dt + \langle g^{-1},  D^2h(X_t)\rangle dt,
\end{align*}
which by taking expectation implies
\begin{align*}
    \frac{d}{dt}\E h(X_t) \Big|_{t=0}= \Delta h(x).
\end{align*}
Hence, $\Delta$ is the infinitesimal operator.
\end{proof}
\end{lemma}

\begin{lemma}[RLD in the Euclidean chart]\label{lem:Euclidean-manifold}
Riemmanian Langevin diffusion in the Manifold sense for the distribution $e^{-F}dv_g(x)$ is the same as it is in the Euclidean definition (\ref{EuclideanRLD}) for the distribution $e^{-f}$ for
\begin{align*}
    f = F - \frac{1}{2}\log \det(g(x)).
\end{align*}
\end{lemma}
\begin{proof}
We have
\begin{align*}
    \nabla \cdot (g^{-1})_i & = \sum_{j}\frac{1}{\sqrt{|g|}}\partial_j (\sqrt{|g|} g^{ij})\\
    & = \sum_j \partial_j(g^{ij}) + \sum_j g^{ij} \frac{1}{\sqrt{|g|}}\partial_j (\sqrt{|g|})\\
    & = \sum_j \partial_j(g^{ij}) + \sum_j g^{ij}\langle \partial_j g, g^{-1}\rangle.
\end{align*}
On the other hand,
\begin{align*}
    g^{-1}Df = g^{-1}(DF - \frac{1}{2}D \logdet(g(x))) = g^{-1}DF - g^{-1}(\frac{1}{2}D\logdet(g(x))).
\end{align*}
Noting that 
\begin{align*}
g^{-1}D\logdet(g(x))_i =  \sum_{j}g^{ij} \partial_j (\logdet(g(x))) = \sum_j g^{ij} \langle g^{-1}, \partial_j g\rangle    
\end{align*}
completes the proof. 

\end{proof}

\subsection{  It\^{o} calculus}
Given continuous semimartingales $X_t$ and $A_t$, one can integrate one with respect to another to obtain a new semimartingale, namely
\begin{align*}
    Y_t = \int_{s_0}^t X_s dA_s,
\end{align*}
which uses the   It\^{o}'s construction for stochastic integration.  The It\^{o} integral obeys a slightly different chain rule compared to the ordinary calculus. Namely, if $X_t$ is a vector of $n$ semimartingales and given a continuously differentiable function $f$ over $\mathbb R^n$,
\begin{align*}
    f(X_t) = \sum_i \int \partial_i f(X_t)d{{X}_t}_i +\frac{1}{2}\sum_{i,j}\int \partial_{i,j} f(X_t) d[{X}_i, {X}_j]_t,
\end{align*}
where $[{X_t}_i, {X_t}_j]$ is the quadratic variation of ${X_t}_i$ and ${X_t}_j$.

\subsection{Isoperimetry}\label{sec:isoperimetry}

\begin{definition}
We say the distribution $\nu$ satisfies a Log-Sobolev inequality with constant $\alpha$ with respect to the metric $g$ 
if for every measure with density $\rho$ with respect to $\nu$:
\begin{align*}
    H_\nu(\rho) \leq \frac{1}{2\alpha} I_\nu(\rho),
\end{align*}
where $H_\nu(\rho)$ is the KL divergence between $\rho$ and $\nu$
and $I_\nu(\rho)$ is the Fisher information of the measure $\rho$ with respect to $\nu$. Note that the metric $g$ only affects the Fisher information $I_\nu(\rho)$.
\end{definition}

Log-Sobelov inequalities in manifolds are a well-studied topic, see e.g., ~\cite{bakry2014} for the compact setting and \cite{Wang97} for the the noncompact Riemannian setting. We note however that LSI in Euclidean space does not translate automatically to LSI with a manifold metric; for example, the uniform density over the unit interval does not have a bounded LSI constant with respect to the metric defined by the log barrier over the interval.

\section{The discretization error and a related stochastic process}\label{sec:discretizationerror}

We use the same notation as e.g. in Erdogdu et al.~\cite{LiErdogdu2020}: We denote the current step of the algorithm by $x_0$ and its distribution by $\rho_0$. $\rho_t$ is the distribution after taking one discrete step with time parameter $t$. 

For a point $x \in \mathcal M$, let $\gamma_t(x)$ be the exponential map at point $x$ applied to the vector $\nabla F(x)$, i.e.
\begin{align*}
    \gamma_t(x) = \exp_x(\nabla F).
\end{align*}
Let $b(t,x) \in T_{\gamma_t(x)}(\mathcal M)$ be the parallel transport of the vector $\nabla F(x)$ from $x$ to $\gamma_t(x)$, i.e.
\begin{align*}
    b(t,x) = P_{x}^{\gamma_t(x)}(\nabla F).
\end{align*}
 Given $x_0$, let $\rho_{t|0}$ be the density of the Brownian motion on manifold at time $t$ with respect to $v_g dx$, starting from $\gamma_t(x_0)$, and let $\rho_{t0}$ be the joint density of $x_0$ and $x_t$. 

To derive the discretization error, we first compute the derivative of the density $\rho_t$ with respect to time in the following lemma, similar to \cite{LiErdogdu2020}.

\begin{lemma}\label{lem:dis-error}
For the derivative of the relative entropy,
\begin{align*}
    \partial_t H_{\nu}(\rho_t) = -I_{\nu}(\rho_t) +   \E_{\rho_0}\E_{x \sim \rho_{t|0}} \langle \nabla_x \log(\frac{\rho_t}{\nu}), \nabla F(x) \rangle
      - \E_{\rho_0} \langle \nabla_{\gamxnot} \E_{\rho_{t|0}} \log(\frac{\rho_t}{\nu}), b(t,x_0) \rangle.
\end{align*}
\end{lemma}

\begin{proof}[Proof of Lemma~\ref{lem:dis-error}]
Denoting the density of Brownian motion by $p_t(x_0, x)$, from the chain rule

\begin{align*}
    \frac{d}{dt} \rho_{t|0}(x) = \frac{d}{dt} p_t(\gamma_t(x_0), x) 
    & = dp_t(\gamma_t(x_0), x)(\partial_t \gamma_t(x_0), \partial_t)
     \\
     & = \langle \grad_{\gamma_t(x_0)} p_t(\gamma_t(x_0), x), \partial_t \gamma_t(x_0)\rangle + \partial_t p_t(\gamma_t(x_0), x)\\
    & = -\langle \grad_{\gamma_t(x_0)} p_t(\gamma_t(x_0), x), b(t,x_0)\rangle + \Delta_x p_t(\gamma_t(x_0), x).\numberthis\label{eq:densityderivative}
\end{align*}
Now taking the time derivative of the KL divergence to the stationary distribution at time $t$,  applying Equation~\eqref{eq:densityderivative}, and integrating by parts, we have:
\begin{align*}
    \partial_t H_\nu(\rho_t) & = \int \frac{d}{dt} \rho_{t}(x) \log(\frac{\rho_t}{\nu}) dv_g(x)\\
    & = \int \int \frac{d}{dt} \rho_{t|0}(x) \rho_0(x_0) \log(\frac{\rho_t}{\nu}) dv_g(x)dv_g(x_0)  \tag{by definition of $\rho_{t|0}$} \\
    & = \int \int \Delta_x \rho_{t|0}(x) \rho_0(x_0) \log(\frac{\rho_t}{\nu}) dv_g(x)dv_g(x_0) \tag{Chain rule and Fokker Plank Equation}\\
    & \indent -  \int \int \langle \grad_{\gamma_t(x_0)} p_t(\gamma_t(x_0), x), b(t,x_0)\rangle \rho_0(x_0) \log(\frac{\rho_{t}}{\nu}) dv_g(x)dv_g(x_0) \\
    &  = -\int \int \langle \grad \rho_{t|0}(x), \grad \log(\frac{\rho_t}{\nu})\rangle \rho_0(x_0) dv_g(x)dv_g(x_0) \tag{integration by parts}\\
    & \indent - \int \langle \grad_{\gamma_t(x_0)} \E_{\rho_{t|0}}\log(\frac{\rho_t}{\nu}), b(t,x_0)\rangle \rho_0(x_0) dv_g(x_0) \tag{Leibniz rule}\\
    &  = -\int \langle \rho_t(x)\grad \log(\rho_t), \grad \log(\frac{\rho_t}{\nu})\rangle dv_g(x)\tag{Leibniz rule, Chain rule}\\
    & \indent - \int \langle \grad_{\gamma_t(x_0)} \E_{\rho_{t|0}}\log(\frac{\rho_t}{\nu}), b(t,x_0)\rangle \rho_0(x_0) dv_g(x_0)\\
    &  = -\int \langle \rho_t(x)\grad \log(\frac{\rho_{t}(x)}{\nu}) + \rho_t\grad \log(\nu), \grad \log(\frac{\rho_t}{\nu})\rangle dv_g(x)\tag{Add and Subtract}\\
    & \indent - \int \langle \grad_{\gamma_t(x_0)} \E_{\rho_{t|0}}\log(\frac{\rho_t}{\nu}), b(t,x_0)\rangle \rho_0(x_0) dv_g(x_0)\\
    & = -\int \rho_t(x)\|\grad \log(\frac{\rho_{t}(x)}{\nu})\|^2 dv_g(x) + \int \langle \grad F, \grad \log(\frac{\rho_t}{\nu})\rangle \rho_t(x) dv_g(x)\\
    & \indent - \int \langle \grad_{\gamma_t(x_0)} \E_{\rho_{t|0}}\log(\frac{\rho_t}{\nu}), b(t,x_0)\rangle \rho_0(x_0) dv_g(x_0)\\
    & = -I_\nu(\rho_t) + 
    \E_{\rho_0}\E_{x \sim \rho_{t|0}} \langle \grad_x \log(\frac{\rho_t}{\nu}), \grad F(x) \rangle - \E_{\rho_0} \langle \grad_{\gamxnot} \E_{\rho_{t|0}} \log(\frac{\rho_t}{\nu}), b(t,x_0) \rangle.\numberthis\label{eq:similarterms}
\end{align*}
\end{proof}

Therefore, defining the discretization error as
\begin{align}
    \DE_t & =
      \E_{\rho_0}\E_{x \sim \rho_{t|0}} \langle \grad_x \log(\frac{\rho_t}{\nu}), \grad F(x) \rangle
      - \E_{\rho_0} \langle \grad_{\gamxnot} \E_{\rho_{t|0}} \log(\frac{\rho_t}{\nu}), b(t,x_0) \rangle \label{eq:secondterm} 
\end{align}
we have
\begin{align}
    \partial_t H_{\nu}(\rho_t) = -I_{\nu}(\rho_t) + \DE.\label{eq:klderivative}
\end{align}

To bound the discretization error, note that for a fixed $x_0$, the distribution $\rho_{t|0}$ is a time $t$ Brownian increment on the manifold starting from $\gamma_{t}(x_0)$. Therefore, we consider the SDE for Brownian motion on $\mathcal M$ starting from $x_0$, which is given in  It\^{o} form as
\begin{align}
    dX_s = A(X_s)dB_s + Z(X_s)dt.\label{eq:coresde}
\end{align}
Recall from~\eqref{eq:manifold_Brownian}, $A = \sqrt{2 g^{-1}}$ and $Z(x) = \nabla \cdot (g^{-1}) = \Gamma_{i,j}^k g^{ij}$,
and $\nabla \cdot$ is the row-wise divergence over the manifold, defined as $\nabla \cdot (g^{-1}) = \frac{1}{\sqrt{|g|}}\sum_j \partial_j (\sqrt{|g|}g^{ij}).$

Taking a stochastic analysis point of view, we wish to employ a formula from ~\cite{elworthy1994formulae} which enables us to write the discretization error as an expectation over the randomness of the underlying Brownian motion. Using the notation $h = \log(\rho_t/\nu)$, our aim is to write the first term of the discretization error as:
\begin{align}
    \langle \nabla_{\gamxnot} \E_{x\sim \rho_{t|0}} h(x), b(t,x_0) \rangle = \mathbb E_{X_t, V_t} \langle \nabla h(X_t) , V_t\rangle,\label{eq:basic1}
\end{align}
for some process $V_s$, called ``the derivative flow,'' which obeys the following joint stochastic differential equation with $X_t$:
\begin{align}
    dV_t = DA(X_t)(V_t) dB_t + DZ_t(V_t) dt.\label{eq:vequation}
\end{align}
with initial condition $V_0 = b(t,x_0)$, where we define $b(t,x_0)$ to be $\partial_t \gamma_t(x_0)$. We think of $V_t$ as an element in the tangent space $T_{X_t}(\mathcal M)$ written in the Euclidean chart.  
Note that $b(t,x_0)$ is the same as the parallel transport of $\nabla F(x_0)$ from $x_0$ to $\gamma_t(x)$. 

To introduce this process, we use the notation $H_t(x)$ to denote the solution flow of the process $X_t$ over the manifold; namely, 

$H_t(x_0)$ is a ``random function'' which maps the initial point $x_0 \in \mathcal M$ to the value of the process at time $t$, namely $X_t$. (recall that $X_t$ is the solution to~\eqref{eq:manifold_Brownian}). Denoting $\gamma_t(x_0)$ by $x_0'$, one can rewrite the first term of the discretization error as

\begin{align}
    \langle \grad_{\gamma_t(x_0)} (\E h(X_t)), b(t,x_0)\rangle = d(\E h(X_t))(b(t,x_0)) = d(\E h(H_t(x_0')))(b(t,x_0)).\label{eq:above}
\end{align}
where $d(\E h(H_t(x_0))) \in T^*_{\gamma_t(x_0)}(\mathcal M)$ is an element in the cotangent space at $x_0$, acting upon $b(t, x_0) \in T_{\gamma_t(x_0)}(\mathcal M)$. We remind the reader that $\grad$ is the gradient on the manifold.
Now if $(1)$ the random map $H$ was differentiable, $(2)$ the chain rule holds, $(3)$ and one could exchange differentiation and expectation in Equation~\eqref{eq:above}, then
\begin{align}
    d(\E h(H_t(x_0')))(b(t,x_0)) = \E d(h(H_t(x_0')))(b(t,x_0)) = \E dh(DH_t(x_0')(b(t,x_0))), ,\numberthis\label{eq:derivativeofexp}
\end{align}
 where $DH_t(x_0')$ is the Jacobian of $H_t(x_0')$ and $dh \in T^*_{H_t(x_0')}(\mathcal M)$ is the derivative of $h$. Denoting $DH_t(x_0')(b(t,x_0))$ by $V_t$, we then obtain Equation~\eqref{eq:basic1}. Here, the notion of differentiability that we work with is that the limit defining the derivative of the flow $H_t(.)$ (in the space variable) exists in probability. In order to make the above derivation~\eqref{eq:derivativeofexp} feasible, we impose a technical condition on the diffusion~\ref{eq:coresde} called ``strong 1-completeness,
which guarantees for every curve $\sigma(r): [0,1] \rightarrow \mathcal M$ starting from $\sigma(0) = x_0'$, that $H_t(\sigma(r))$ is differentiable in $r$ almost surely, and the derivative in some fixed direction $V_0$ satisfies the SDE in~\eqref{eq:vequation} with initial condition $V_0$ (recall that the derivative $DH_t(x_0')(V_0)$ is a random variable itself).

 Furthermore, in order to be able to change the order of differentiation and integration as we did in~\eqref{eq:derivativeofexp}, we need an additional integrability assumption.

In particular, we need to work with potentially unbounded functions $f$, while the result in~\cite{li1994strong} is stated for bounded $f$ with bounded derivative. Hence, we go through the proof of a result in~\cite{li1994strong} regarding Equation~\eqref{eq:derivativeofexp} and meanwhile relax the boundedness condition to an $\ell^2$ integrability condition.

\begin{lemma}\label{lem:exchange}~\cite{li1994strong}[Theorem 9.1]
Suppose that $h$ is continuously differentiable and we have the strong 1-completeness assumption of the Brownian motion on $\mathcal M$, as defined in~\eqref{eq:coresde}. If we further assume that $\mathcal M$ is compact then $f$ and $\|dh\|$ are continuous and bounded, for any vector $V_0 \in T_{x_0}(\mathcal M)$:
\begin{align*}
    d(\E_{\rho_{t|0}} h(X_t))(V_0) = \E_{\rho_{t|0}} dh {DH_t}(x_0')(V_0), 
\end{align*}

Furthermore, one can relax the boundedness assumption under the existence of the second moment $\mathbb E_{\rho_{t|0}} \|dh(X_t)\|^2$. Here, $\|.\|$ refers to the manifold norm of the cotangent element $dh(X_t) \in T^*_{X_t}(\mathcal M)$.
\end{lemma}

In Section~\ref{sec:log-barrier-smoothness}, we will verify the strong-1-completeness condition for the log barrier.

Using these results, the discretization error $\DE$ can be written as (stated formally in Lemma~\ref{lem:almostall}):
\begin{align*}
    \DE & = \E_{\rho_0} \E_{X_t} \langle \grad \log(\frac{\rho_t}{v}), V(X_t) - \nabla F(X_t) \rangle\\
    & \leq \E_{\rho_0} (\E_{X_t} \frac{1}{4} \|\grad \log(\frac{\rho_t}{v})\|^2  + \|V(X_t) - \nabla F(X_t)\|^2)\\
    & = \frac{1}{4} \E_{\rho_t} \|\grad \log(\frac{\rho_t}{v})\|^2 +  \E_{\rho_0} \E_{X_t} \|V(X_t) - \nabla F(X_t)\|^2\\
    & = \frac{1}{4}I_v(\rho_t) + \E_{\rho_0} \E_{X_t} \|V(X_t) - \nabla F(X_t)\|^2.\numberthis\label{eq:initialderivation}
\end{align*}
In the next section, we aim to bound the second part, i.e. $\E_{\rho_0} \E_{X_t} \|V(X_t) - \nabla F(X_t)\|^2$.

\begin{proof}[Proof of Lemma~\ref{lem:exchange}]
For simplicity, here we drop the index $\rho_{t|0}$ of the expectations. We follow along the same lines of the proof of Theorem 9.1 in~\cite{li1994strong}. Let $\sigma(r): [0,s] \rightarrow \mathcal M$ be a smooth curve on $\mathcal M$ starting from $\sigma(0) = x_0$, with initial speed $\sigma'(0) = v_0$. From strong 1-completeness, we know the derivative of $H_t(\sigma(r))$ exists almost surely(~\cite{li1994strong}, Proof of Theorem 9.1), i.e.
\begin{align*}
    \frac{h(H_t(\sigma(s))) - h(H_t(\sigma(0)))}{s} = \frac{1}{s}\int_{0}^s dh({DH_t}(\sigma(r))(\sigma'(r)))dr. 
\end{align*}
The strong 1-completeness also implies that ${DH_t}_{\sigma(r)}(\sigma'(r))$ is a.s. continuous in $r$(~\cite{li1994strong}, Proof of Theorem 9.1). Therefore, from the smoothness of $f$, $dh {DH_t}_{\sigma(r)}(\sigma'(r))$ is also continuous, which implies
\begin{align*}
  \lim_{s \rightarrow 0} \frac{1}{s}\int_{0}^s dh({DH_t}(\sigma(r))(\sigma'(r)))dr =    dh({DH_t}(\sigma(0))(\sigma'(0))).
\end{align*}
Now if we have some kind of dominated convergence theorem, we can write
 
\begin{align*}
   d(\E_{\rho_{t|0}} h(X_t))(V_0) & = \lim_{s \rightarrow 0} \frac{\E h(H_t(\sigma(s))) - \E h(H_t(\sigma(0)))}{s}\\
   & = \lim_{s \rightarrow 0} \E \frac{1}{s} \int_{0}^s dh({DH_t}(\sigma(r))(\sigma'(r)))dr\\
   & = \E dh({DH_t}(x_0)(v_0)),
\end{align*}
as desired.
For this, we use the generalized dominated convergence theorem. We use the following upper bound functions:
\begin{align*}
    \frac{1}{s} \int_{0}^s dh({DH_t}_{\sigma(r)}(\sigma'(r)))dr
    & \leq \frac{1}{s} \int_{r=0}^s \|dh(H_t(\sigma(r)))\| \|{DH_t}(\sigma(r))\|dr\\
\end{align*}
where $\|{DH_t}_{\sigma(r)}\|$ is the operator norm of ${DH_t}(\sigma(r))$ and $\|df(H_t(\sigma(r)))\|$ is the norm of a cotangent element in $T^*_{H_t(\sigma(r))}(\mathcal M)$, both taking into account the manifold metric. 
Using these random variables as the upper bound functions in the generalized DCT, it is enough to show the expectation $\E \frac{1}{s}\int_{0}^s \|dh\|\|{DH_t}(\sigma(r))(\sigma'(r))\|$ converges to $\E \|dh\|\|{DH_t}(\sigma(0))(\sigma'(0))\|$. For this, following~\cite{li1994strong} (Proof of Theorem 9.1), we show that these variables are uniformly integrable. Note that $\|dh(H_t(\sigma(r)))\| \|{DH_t}_{\sigma(r)}\|$ converge to $\|dh(H_t(\sigma(0)))\| \|{DH_t}_{\sigma(0)}\|$;  this a.s. convergence plus uniform integrability then would imply convergence in $L_1$. It is left to prove the uniform integrability. For that, it is enough to show for some $\delta > 0$:
\begin{align*}
    \sup_{x \in K} \E(\|dh(H_t(\sigma(r)))\| \|{DH_t}_{\sigma(r)}\|)^{1 + \delta} \leq \infty,
\end{align*}
For any compact neighborhood $K$ of $x_0$. We write
\begin{align*}
    \E(\|dh(H_t(\sigma(x)))\| \|{DH_t}_{\sigma(r)}\|)^{1 + \delta} \leq (\E \|dh(H_t(\sigma(r)))\|^2)^{\frac{1+\delta}{2}} (\E \|{DH_t}(\sigma(r))\|^{\frac{2(1+\delta)}{1-\delta}})^{\frac{1-\delta}{2}}.
\end{align*}
The point is $\E \|dh(H_t(\sigma(x)))\|^2 < \infty$ by assumption, and taking $\delta = \frac{1}{3}$, the second term becomes $(\E \|{DH_t}_{\sigma(x)}\|^{4})^{\frac{1}{3}}$.
Now notice that for any initial vector $v_0$, ${DH_t}_{\sigma(r)}(v_0)$ obeys the SDE in~\eqref{eq:vequation}, i.e. $V_t = {DH_t}_{\sigma(r)}(v_0)$ is a solution to this SDE. we show in Equation~\eqref{eq:vbounded} that $\E \|V_t\|^2$ is bounded for all $t$. To exploit this result, letting $\{v_i\}_{i=1}^n$ be an orthonormal basis at $\sigma(r)$, we use Equation~\eqref{eq:V4bound}:
\begin{align*}
    \E \|{DH_t}(\sigma(r))\|^4 & \leq \E \|{DH_t}(\sigma(r))\|_{F}^4 = \E (\sum_i  \|{DH_t}(\sigma(r))(v_i)\|^2)^2 \\
    & \lesssim n\E \sum_i  \|{DH_t}(\sigma(r))(v_i)\|^4\\
    & \lesssim
    (\sum_i\|v_i\|_g^4)\exp\{tC\} \\
    & \lesssim
    (\sum_i\|v_i\|_g^2)^2\exp\{tC\} \\
    & = n^2 \exp\{tC\} < \infty.
\end{align*}
which completes the argument.
\end{proof}

\subsection{Changing Differentiation and Expectation}
In this section, we prove a lemma that enables us to write the second term of the discretization error as an expectation over the process $V_t$.
\begin{lemma}\label{lem:almostall}
Let $\mathcal M$ be a manifold on which the Brownian motion satisfies the strong 1-completeness, which has Ricci curvature bounded from below, and bounded sectional curvature from above $K$. Suppose we start the algorithm from a distribution $\rho_0^*$ which has $\ell_2$ bounded density with respect to the volume of the  manifold, and bounded second moment $\E_{x_0 \sim \rho^*_0} d(x_0, p)^2 < \infty$ ($p$ is a fixed point and $d$ is Riemannian distance). Then, at any iteration of the algorithm, if $\rho_0$ is the current density and $\rho_t$ is the density of the next step, we have for $\rho_0$-almost all $x_0$:
\begin{align*}
     \langle \grad_{\gamxnot} \E_{\rho_{t|0}} \log(\frac{\rho_t}{\nu}), b(t,x_0) \rangle = \E_{X_0 = \gamma_t(x_0)} \langle \grad \log(\frac{\rho_t}{v}), V(X_t) - \nabla F(X_t) \rangle,
\end{align*}
For the brownian motion $X_t$ on $\mathcal M$ starting from $\gamma_t(x_0)$ and the derivative flow process $V_t$ as defined in~\eqref{eq:Vsde}.
\end{lemma}
\begin{proof}
From Lemma~\ref{lem:l2bounded2}, we see 
\begin{align*}
     \E_{\rho_0} \E_{\rho_{t|0}} \|\nabla\log(\rho_t/\nu)\|^2 = \E_{\rho_t} \|\nabla\log(\rho_t/\nu)\|^2 < \infty.
\end{align*}
Hence, for $\rho_0$ almost all $x_0$, we have $\E_{\rho_{t|0}} \|\nabla\log(\rho_t/\nu)\|^2 < \infty$, which satisfies the $\ell_2$ bounded condition of Lemma~\ref{lem:exchange}. Combining this with the 1 strongly completeness assumption, Lemma~\ref{lem:exchange} implies that for $\rho_0$ almost all $x_0$, we have the desired exchange property.
\end{proof}

\paragraph{Distance between measures stays bounded.}

In this section, we show that the $\chi^2$ distance of the density $\rho_t$ with respect to the volume of the manifold, denoted by $\chi^2(\rho_0)$ remains bounded. Note that this divergence is nothing but the second moment of the density:
\begin{align*}
\chi^2(\rho_0) = \int \rho_0^2(x_0) dv_g(x_0).   
\end{align*}

\begin{lemma}\label{lem:onestepboundedness}
Under the condition $t \leq \frac{1}{6\sqrt K + 8\sqrt L_2}$, starting from a distribution $\rho_0$ where $\chi^2(\rho_0) < \infty$, after taking a geodesic step with some duration $t$ we still have $\chi^2(\gamma^*_t) < \infty$.
\end{lemma}
\begin{proof}
  Note that that
\begin{align*}
    \gamma_t^*(\rho_0)(\gamma_t(x_0)) = \rho_0(x_0) \logdet(J(\gamma_t)^{-1}(y)),
\end{align*}
where $J(\gamma_t)(y)$ is the differential of the map $\gamma_t$ which is from $T_y(\mathcal M)$ to $T_{\gamma_t(y)}(\mathcal M)$. Note that we are overloading notation and refer to the density of the push forward map $\gamma_t^*(\rho_0)$ using the same notation.
But Lemma~\ref{lem:jacobibound} shows that the operator norm of $J(\gamma_t)^{-1}(y)$ is upper bounded by $\sqrt{1 - (3\sqrt K + 4\sqrt{L_2})t}$ and similar to~\eqref{eq:similard}:
\begin{align*}
     \E_{y \sim \rho_0} \logdet(J(\gamma_t)^{-1}(y)) \leq - \frac{n}{2}\log(1 - (3\sqrt K + 4\sqrt{L_2})t) \leq (3\sqrt K + 4\sqrt{L_2})nt \leq n/2.
\end{align*}
where we are assuming the condition $t \leq \frac{1}{6\sqrt K + 8\sqrt L_2}$. We then have
\begin{align*}
    \gamma_t^*(\rho_0)(\gamma_t(x_0)) \leq \rho_0(x_0)e^{n/2},
\end{align*}
which combined with casting a push forward measure integral into the integral of the measure itself implies
\begin{align*}
    \chi^2(\gamma_t^*(\rho)) & = \E_{\gamma_t^*(\rho_0)}\gamma_t^*(\rho_0)(x_0)^2 \\
    & = \E_{\rho_0}\gamma_t^*(\rho_0)(\gamma_t(x))^2\\
    & \leq \E_{\rho_0} \rho_0(x_0)^2e^{n}\\
    & = e^{n}\chi^2(\rho_0) < \infty,
\end{align*}
which completes the proof.
\end{proof}

\begin{lemma}\label{lem:helper1}
Under the condition $t \leq \frac{1}{6\sqrt K + 8\sqrt L_2}$, assuming $\chi^2(\rho_0) < \infty$, for $\rho_t$  the distribution after one step of the Riemannian Langevin algorithm starting from $\rho_0$, we have (recall that $K$ denote the sectional curvature upper bound)
\begin{align*}
    &\chi^2(\rho_t) < \infty,\\
    &\E_{\rho_t} \|\nabla\log(\rho_t)\|^2 < \infty.
\end{align*}
\end{lemma}
\begin{proof}
  From Lemma~\ref{lem:onestepboundedness} we know $\chi^2(\gamma_t^*(\rho_0)) < \infty$. On the other hand, since Brownian motion is self-adjoint (with symmetric semigroup), Then $\rho_t = P_t\gamma_t^*(\rho_0)$, and since $P_t$ is contraction in $\ell_2$~\cite{bakry2014}, we get  
  \begin{align*}
      \chi^2(\rho_t) = \int \rho_t(x)^2 dv_g(x) \leq \int \gamma_t^*(\rho_0)(x)^2 dv_g(x) < \infty.
  \end{align*}
  On the other hand, since the initial density $\gamma_t^*(\rho_0)$ (just before running brownian motion) has bounded $\ell_2$ norm, it is in the domain of the infinitesimal operator of the brownian motion. As a result, the entropy functionsl along the evolving density of the brownian motion run is differentiable, and in particular its derivative, i.e. the fisher information is bounded~\cite{bakry2014}. In other words, we have
  \begin{align*}
      I_{v_g}(\rho_t) = \E_{\rho_t} \|\nabla\log(\rho_t)\|^2 < \infty,
  \end{align*}
  which completes the proof. Note that $v_g$ refers to the volume measure of the manifold.
\end{proof}

\begin{lemma}\label{lem:helper2}
Under the condition $t \leq \frac{1}{6\sqrt K + 8\sqrt L_2}$, Assuming the relative likelihood $\E_{\rho_0}\|\nabla\log(\nu)\|^2$ is bounded for the initial distribution $\rho_0$, then it remains bounded after one iteration, namely $\E_{\rho_t}\|\nabla\log(\nu)\|^2$ is bounded.
\end{lemma}
\begin{proof}
  For a fixed $x_0$, we have
  \begin{align*}
      \E_{\rho_{t|0}} \|\nabla\log(\nu)\|^2&  = 
      \E_{\rho_{t|0}} \|\nabla F(x)\|^2 \\
      & = \E \|\nabla F(X_t)\|^2 \leq \|\nabla F(X_0)\|^2 + (\beta_0 + \beta_1)t,
  \end{align*}
  where we used Equation~\eqref{eq:ODEbound}. Now taking expectation with respect to $\rho_0$:
  \begin{align*}
      \E_{\rho_t} \|\nabla\log(\nu)\|^2 & = \E_{\rho_0} \E_{\rho_{t|0}} \|\nabla\log(\nu)\|^2 \\
      & \leq \E \|\nabla F(X_0)\|^2 + (\beta_0 + \beta_1)t\\
      & \lesssim \E_{x_0 \sim \rho_0}\|\nabla F(x_0)\|^2 + t^2L_2^2 + (\beta_0 + \beta_1)t\\
      & = \E_{\rho_0} \|\nabla \log(\nu)\|^2
      < \infty,
  \end{align*}
  where we applied Lemma~\ref{lem:paralleldiffpart}.
\end{proof}

\begin{lemma}\label{lem:l2bounded}
We have
\begin{align*}
    \E_{\rho_t} \|\nabla\log(\rho_t/\nu)\|^2 < \infty.
\end{align*}
\end{lemma}

\begin{proof}
\begin{align*}
  \E_{\rho_t} \|\nabla\log(\rho_t/\nu)\|^2 & \leq
  \E_{\rho_t}\|\nabla \log(\rho_t) + \nabla(\log(\nu))\|^2\\
  & \lesssim \E_{\rho_t}\|\nabla \log(\rho_t)\|^2 + \E_{\rho_t}\| \nabla(\log(\nu))\|^2,  
\end{align*}
which is bounded by combining Lemmas~\ref{lem:helper1} and~\ref{lem:helper2}. 
\end{proof}

\begin{lemma}\label{lem:l2bounded2}
If we start the algorithm from a distribution $\rho_0^*$ whose density with respect to the manifold volume is in $\ell_2$, and the second moment bounded $\E_{x_0 \sim \rho_0^*} d(p,x_0)^2 < \infty$ ($p$ is a fix point on manifold and $d$ is the manifold distance), then at any iteration of the algorithm , the density $\rho_k$ satisfies
\begin{align*}
    \E_{\rho_k} \|\nabla\log(\rho_k/\nu)\|^2 < \infty.
\end{align*}
\begin{proof}
  Directly from Lemmas~\ref{lem:l2bounded},~\ref{lem:helper1}, and~\ref{lem:helper2}.
\end{proof}
 \end{lemma}

\section{Bounding the discretization error}

\subsection{Overview}

Note that $g^{-1/2}$ is the term appearing in the Langevin equation. Hence, because the time derivative of the evolving density w.r.t to the stationary measure $(\rho_t/\nu)$ is equal to $L$ (the infinitesimal operator of the Markov chain), and $L$ involves the term $\langle g^{-1/2}, D^2(\rho_t/\nu) \rangle$, one might hope to control this term to bound the discretization error. However, this is not possible for two reasons: First, $g^{-1/2}$ is not self-concordant in general; in fact, its self-concordance itself is not an affine-invariant property. Second, the matrix $D^2(\rho_t/\nu)$ is not necessarily positive definite.
Instead of directly trying to control the derivative of the evolving density, we observe that the second and third terms in Equation~\eqref{eq:similarterms} are similar but different in the fact that the order of computing the gradient and taking the expectation with respect to the conditional distribution $\rho_{t|0}$ are switched. In the Euclidean space, this change of derivation and integration could be easily handled, but here we need to use a more sophisticated machinery, in its core is the process $V(t)$ that we define. As we discuss in section~\ref{sec:discretizationerror}, this allows us to write the discretization term as the expected norm squared between $V$ and $\nabla F$ in Equation~\eqref{eq:initialderivation}. 

The overall approach from here onward is to bound the discretization error $\E_{\rho_0} \E_{X_t} \|V(X_t) - \nabla F(X_t)\|^2$.

Notably, we bypass the argument about self concordance of $g^{-1/2}$ in Lemma~\ref{lem:critical} by a weaker argument which turns out to be sufficient for our needs.
We start by bounding the norm of $V(t)$ for small enough times $t$. To this end, we write the squared norm of the $V_t$ process as a stochastic integral along the stochastic curve $X(t)$. More specifically, we use Lemma~\ref{lem:normexpantion} which is directly It\^{o}'s rule.

\begin{lemma}\label{lem:normexpantion}
Given stochastic vector field $V(t)$ on the stochastic curve $X_t$, we have
\begin{align*}
    d\langle V, V\rangle_g(t)
    &= 2\langle dV(t) ,V\rangle_g
    + \sum_{i,j,p}V_i \partial_p g_{ij} V_j dX_p+ \sum_{i,j,p,r}\frac{1}{2}V_i V_j\partial_{p,r}g_{ij} d[X_p, X_r] + 
     \sum_{i,j,p}\partial_p g_{ij} d[V_i, X_p] V_j + \sum_{i,j}\frac{1}{2} g_{ij} d[V_i, V_j].
\end{align*}
\end{lemma}
\begin{proof}
Using It\^{o}'s rule,
\begin{align*}
    d\langle V,V\rangle_{g}(t) &=  2\langle dV(t) ,V\rangle_g
    + \sum_{i,j}V_i dg_{ij} V_j + \sum_{i,j}d[ V_i, g_{ij}] V_j + \sum_{i,j}\frac{1}{2}g_{ij} d[ V_i, V_j]\\
    &= 2\langle dV(t) ,V\rangle_g
    + \sum_{i,j,p}V_i \partial_p g_{ij} V_j dX_p+ \sum_{i,j,p,r}\frac{1}{2}V_i V_j\partial_{p,r}g_{ij} d[ X_p, X_r] + 
     \sum_{i,j,p}\partial_p g_{ij} d[ V_i, X_p] V_j + \sum_{i,j}\frac{1}{2} g_{ij} d[ V_i, V_j]\\
     & = 2\langle dV(t) ,V\rangle_g + (*).
\end{align*}
\end{proof}

Our job is to then bound its growth using the self-concordance of the metric. We will see that our generalized self-concordance plays nicely with the terms that we obtain in the stochastic integral using It\^{o}'s rule. Notably, here we also prove a regularity result on the derivative of the square-root of the metric based only on the self-concordance assumption for the metric itself, proved in Section~\ref{app:criticlem}.
\begin{lemma}\label{lem:critical}
Suppose the metric $g$ is $\gamma_1$-self-concordant (in the standard sense). Then, for the square root matrix $A(x) = \sqrt{2g^{-1}}$, we have
\begin{align*}
    &\langle g, (DA(V))^2\rangle \leq n\gamma_1^2\|V\|_g^2,\\
    &\langle g^{1/2}, DA(V)\rangle \leq \sqrt n\gamma_1\|V\|_g.
\end{align*}
\end{lemma}

This bound shows how the self-concordance assumption can be leveraged. In general, the square-root of the metric is {\em not} self-concordant~\cite{LTVW2021}.

We prove that the local martingale part is indeed a martingale, and then take expectation and obtain a simple ODE which reveals an upper bound on $\E \|V_t\|^2$. The following Lemma illustrates our final result for control over $\E \|V_t\|^2$, proved in 
Section~\ref{app:vbound}.

\begin{lemma}\label{lem:Vbound}
For 
\begin{align}
    t \lesssim \frac{1}{2n^{3/2}(\gamma_1^2 + \gamma_2^2 + \gamma_3^2)}\label{eq:criticcondition},
\end{align}
we have
\begin{align}
    \E_{\rho_{t|0}} \|V_t\|_g^2 \leq \|\nabla F(x_0)\|_g^2 (1 + tn^{3/2}n(\gamma_1^2 + \gamma_2^2 + \gamma_3^2)) \leq 2 \|\nabla F(x_0)\|_g^2.\label{eq:vbounded}
\end{align}
\end{lemma}

Next, we control the squared norm of the manifold gradient $\nabla F$ along the stochastic curve. Here, we first extract useful information from the gradient and Hessian Lipschitzness assumption of $f$ on $\mathcal M$  using again the self-concordance of the metric (see Section~\ref{app:derivativetensorbounds}). We then use this information to derive a bound on the average squared norm of $J_t = \nabla F(X_t)$, which is stated in the following lemma:
\begin{lemma}\label{lem:Jbound}
For $t \leq \frac{1}{\xi}$, we have
\begin{align}
    \E_{\rho_{t|0}} \|J_t\|^2 \lesssim \|J_0\|^2 + \beta t,\label{eq:ODEbound}
\end{align}
where $\xi = n(\gamma_1^2 + \gamma_3^2) + L_2\gamma_1n^{3/2}$.
\end{lemma}

Finally, we use these methods to analyze the change of the norm of the difference of $V$ and $\nabla F$ by an It\^{o} expansion which is more challenging as the processes $V$ and $J$ have nonzero quadratic variation with respect to each other. We remind the reader that $V_t$ and $\nabla F(X_t)$ are elements in the tangent space of $X_t$ and their norms are computed based on the metric on $\mathcal M$, so in all the calculations we need to take into account the derivatives of the metric itself. There, we use the norm bounds on $V_t$ and $J_t$ from Lemmas~\ref{lem:Vbound} and~\ref{lem:Jbound}. Moreover, we also employ a generalization of the Talagrand Wasserstein-KL inequality to Riemannian manifolds by authors in~\cite{otto2000generalization}. The final bound is stated below.
\begin{lemma}\label{lem:VJbound}
    
    For 
    $$\epsilon \leq \frac{1}{n^{3/2}(\gamma_1^2 + \gamma_2^2 + \gamma_3^2 + nL_3 + L_2n\sqrt n\gamma_1) + L_2^2},$$
    we have 
    \begin{align*}
    \E_{\rho_{t|0}} \|V_t - J_t\|^2 \lesssim \omega + \omega_0 H_\nu(\rho_0). 
\end{align*}
for 
\begin{align}
    & \omega_0 = c_1'(n^{3/2} (\gamma_1^2 + \gamma_2^2 + \gamma_3^2))\epsilon,\\
    & \omega = c_2'\Big[ (n^{3/2} (\gamma_1^2 + \gamma_2^2 + \gamma_3^2))nL_2
     + \sqrt nL_2^2 + nL_3 \Big]\epsilon.
\end{align}

\end{lemma}

Finally, we combine this with  Equations~\eqref{eq:initialderivation} and~\eqref{eq:klderivative} to obtain a differential inequality for relative entropy. The final result of this section is presented in Lemma~\ref{lem:main}, which is used to prove Theorem~\ref{thm:generalmanifold} in Section~\ref{sec:rate}.

\begin{lemma}\label{lem:main}
For time at most 

$$\epsilon \leq \frac{1}{n^{3/2}(\gamma_1^2 + \gamma_2^2 + \gamma_3^2 + nL_3 + L_2n\sqrt n\gamma_1) + L_2^2},$$
we have
\begin{align*}
    \partial_t H_\nu(\rho_t) \leq -\frac{3}{4}I_\nu(\rho_t) + \omega + \omega_0 H_\nu(\rho_0),
\end{align*}
\end{lemma}

\subsection{Bound on the norm of the stochastic process}\label{app:vbound}

The goal of this section is to prove Lemma~\ref{lem:Vbound}.
Recall that $Z(X_t)$ is the drift term in the SDE corresponding the Brownian motion on manifold, which can be written as $Z_k(x) = -\sum_{i,j} \Gamma_{ij}^k g^{ij}(x)$ as stated in section~\ref{sec:brownian}. Now
We start by calculating the derivative $DZ(V)$ using the ordinary chain rule. We will use the following simplification of gradients which comes from the assumption that our manifold is a Hessian manifold. 
\begin{lemma}[Hessian property]\label{lem:Hessian_prop}
On a Hessian manifold, we have
\begin{align*}
    & \partial_i g_{qj} + \partial_j g_{qi} - \partial_q g_{ij} = \partial_q g_{ij} = (Dg)_{qij},\\
    & \partial_p\partial_i g_{mj} + \partial_p \partial_j g_{mi} - \partial_p \partial_m g_{ij} = (D^2g)_{ijmp}
\end{align*}
\end{lemma}
\begin{proof}
Writing the metric in terms of the barrier $g = D^2\phi$, for the first equation,
\begin{align*}
     \partial_i g_{qj} + \partial_j g_{qi} - \partial_q g_{ij} =  \partial_{ijq}\phi +  \partial_{ijq}\phi -  \partial_{ijq}\phi =  \partial_{qij}\phi = (Dg)_{qij},
\end{align*}
where $\partial_{ijq}$ is the short form of $\partial_i \partial_j \partial_q$. Similarly for the second equation,
\begin{align*}
   \partial_p\partial_i g_{mj} + \partial_p \partial_j g_{mi} - \partial_p \partial_m g_{ij} = \partial_{ijmp} \phi = (D^2g)_{ijmp}. 
\end{align*}
\end{proof}

\begin{lemma}\label{lem:dZ(V)}
\begin{align*}
    DZ_k(V)  = & \frac{1}{2} D(g^{-1})(V) \langle g^{-1}, Dg\rangle
    + \frac{1}{2} g^{-1} \langle \partial_{V}g^{-1}, Dg\rangle
    -  \frac{1}{2} g^{-1} \langle DDg(V), g^{-1} \rangle.\numberthis\label{eq:firstpart}
\end{align*}

\end{lemma}

\begin{proof}
\begin{align*}
    DZ_k(V) &= -\sum_{p,i,j} V_pD_p \Gamma_{i,j}^k g^{i,j} - \sum_{p,i,j}V_p \Gamma_{i,j}^k D_p(g^{i,j}) \tag{Using Equation~\ref{eq:zequationn} and Euclidean chain rule}\\
    & = -\frac{1}{2}\sum_{p,m,i,j}V_pD_p(g^{km})(\partial_i g_{mj} + \partial_j g_{im} - \partial_m g_{ij})g^{ij} \\
    & - \sum_{p,m,i,j}\frac{1}{2}V_pg^{km}(\partial_p(\partial_i g_{mj} + \partial_j g_{mi} - \partial_m g_{ij})g^{ij}) \tag{Definition of Christoffel symbols}\\
    & +  \sum_{p,\ell,i,j,r}V_p \Gamma_{i,j}^k g^{i,\ell} D_p g_{\ell,r} g^{rj}\\
    & = \sum_{p,m,s,q,j,i}\frac{1}{2}V_p g^{km} \partial_p g_{ms} g^{sq} (\partial_i g_{qj} + \partial_j g_{qi} - \partial_q g_{i,j})g^{i,j}\\
    & - \sum_{p,m,i,j}\frac{1}{2} V_p g^{km}(\partial_p \partial_i g_{mj} + \partial_p \partial_j g_{mi} - \partial_p \partial_m g_{i,j}) g^{i,j}\\
    & + \sum_{p,m,i,j, \ell, r}\frac{1}{2} V_p g^{km}(\partial_i g_{mj} + \partial_j g_{mi} - \partial_m g_{i,j}) g^{i\ell} D_p g_{\ell,r} g^{rj},
\end{align*}
where note that in the above we dropped the time index of the process for simplicity. Now using the fact that the manifold is Hessian, we can further simplify the above. 
For the first term, the Hessian property (Lemma~\ref{lem:Hessian_prop}), we have $\partial_i g_{qj} + \partial_j g_{qi} - \partial_q g_{ij} = \partial_q g_{ij} = (Dg)_{qij}$. Now summing over $i,j$, we see
\begin{align*}
   \sum_{i,j}(\partial_i g_{qj} + \partial_j g_{qi} - \partial_q g_{i,j})g^{i,j} = \langle Dg, g^{-1}\rangle.
\end{align*}
In the above notation, $Dg$ refers to the three-dimensional tensor of derivatives of the matrix $g$ and
 by $\langle Dg, g^{-1}\rangle$ we mean the $n$-dimensional vector whose $i$th entry is $\langle D_ig, g^{-1}\rangle$. On the other hand, the first part $V_pD_p(g^{km})$ is nothing but $(Dg^{-1}(V))_{km}$. Hence, overall the first term can be written as $\frac{1}{2} D(g^{-1})(V) \langle g^{-1}, Dg\rangle$. For the second term, again using the fact that the manifold is hessian, we have 
$$\sum_p V_p (\partial_p\partial_i g_{mj} + \partial_p \partial_j g_{mi} - \partial_p\partial_m g_{ij}) = \sum_p V_p \partial_p\partial_m g_{ij} = (D^2g(V))_{ijm}.$$
summing this over $i,j$:
$$\sum_{ijp}V_p(\partial_p\partial_i g_{mj} + \partial_p \partial_j g_{mi} - \partial_p \partial_m g_{ij})g^{ij} = \langle D^2g(V), g^{-1}\rangle.$$
Finally, summing over $m$ gives that the second term is equal to 
$$-\frac{1}{2}g^{-1}\langle D^2g(V), g^{-1}\rangle.$$
For the third term, note that summing over $p$:
$$\sum_p V_pg^{i\ell}D_pg_{\ell,r}g^{rj} = \sum_p V_p D_p(g^{ij}) = (Dg^{-1}(V))_{ij}.$$
Now summing over $i,j$ and using the Hessian manifold property again
\begin{align*}
    \sum_{p,i,j}(\partial_i g_{mj} + \partial_j g_{mi} - \partial_m g_{ij})V_pD_p(g^{ij}) = 
    \langle Dg, Dg^{-1}(V)\rangle,
\end{align*}
where in the above, we mean the $m$th entry of the right hand side is equal to the left hand side. Overall, the last term is (by summing the above over $m$)
$$
\frac{1}{2}g^{-1}\langle Dg^{-1}(V), Dg\rangle.
$$

Combining the above, we have the lemma.
\end{proof}

\begin{lemma}\label{lem:dvsquared}
\begin{align*}
     d\langle V, V\rangle_g(t) & =  M_t dB_t + R_t dt
\end{align*}
where
\begin{align*}
    & M_t = 2V^TgDA(V)  +   V^T Dg(V)A \\
    & R_t = 2V^TgDZ(V) + V^T Dg(V) Z 
    V^T\langle D^2g(V), g^{-1}\rangle + 
    V^T \langle Dg,  Dg^{-1}(V)\rangle +
     \langle g, (DA(V))^2\rangle.
         \numberthis\label{eq:secondpart}
\end{align*}

\end{lemma}

\begin{proof}
We starting by expanding $d\|V_t\|_g^2$ using Lemma~\ref{lem:normexpantion}:
\begin{align*}
    d\langle V,V\rangle_g(t)
    = 2\langle dV(t) ,V\rangle_g
    + \sum_{i,j,p}V_i \partial_p g_{ij} V_j dX_p
    + \sum_{i,j,p,r}\frac{1}{2}V_i V_j\partial_{p,r}g_{ij} d[ X_p, X_r] + 
     \sum_{i,j,p}\partial_p g_{ij} d[ V_i, X_p] V_j + \sum_{i,j}\frac{1}{2} g_{ij} d[ V_i, V_j].
\end{align*}
Now opening the quadratic variation parts using the formula for $V$, stated in Equation\eqref{eq:vequation}:
\begin{align}
    &\sum_{i,j,p,r}\frac{1}{2}V_iV_j \partial_{p,r} g_{i,j}d\langle X_p, X_r\rangle = \sum_{i,j,p,r}V_iV_j \partial_{p,r} g_{i,j} g^{pr} dt,\\
    & \sum_{p,i,j}\partial_p g_{ij} V_j d\langle V_i, X_p\rangle = \sum_{p,i,j}\partial_p g_{ij}V_j (DA(V)A)_{ip}dt,\\
    & \sum_{i,j}\frac{1}{2} g_{ij}d\langle V_i, V_j \rangle = \sum_{i,j}\frac{1}{2}g_{ij} {(DA(V))^2}_{ij} dt.\label{eq:simplifications}
\end{align}
An important point above is that since we are on Hessian manifold, regarding the second term we have the symmetry $\partial_p g_{ij} = \partial_i g_{pj}$, so we have that the sum $\partial_p g_{ij}V_j (DA(V)A)_{ip}$ is equal to $$\frac{1}{2}\partial_p g_{ij}V_j \Big((DA(V)A)_{ip} + (ADA(V))_{ip}\Big).$$
 On the other hand, note that $A^2 = 2g^{-1}$ by definition, so from the product rule of differentiation, we have
\begin{align}
DA(V_t) A + A (DA(V_t)) = 2Dg^{-1}(V_t).
\end{align}
Hence, overall, after expanding $dX_t$:
\begin{align*}
    (*)
    & = \sum_{i,j,p,\ell}V_i \partial_p g_{ij} V_j A_{p\ell}dB_\ell + \sum_{i,j,p}V_i \partial_p g_{ij} V_j Z_p dt  + \sum_{i,j,p,r}V_i V_j\partial_{p,r}g_{ij} g^{pr} dt\\ 
    & + \sum_{i,j,p,r}\partial_p g_{ij}V_j Dg^{ip}(V) dt
     + \sum_{i,j}g_{ij} {(DA(V))^2}_{ij} dt.
\end{align*}
On the other hand, for the first term
\begin{align*}
    2\langle dV(t) ,V\rangle & = 2\langle DA(V)dB_t + DZ(V)dt, V\rangle.
\end{align*}

\end{proof}

Here $M_tdB_t$ is the local martingale part and $R_t dt$ is the finite variation part of the process. Now first, we show that the local martingale part is indeed a martingale. To this end, first we bound each of the terms of the quadratic variation part one by one:
\begin{align*}
    |R_t| & \lesssim \|V\|_g\|DZ(V)\|_{g} + \|V\|_g\|g^{-1}Dg(V)Z\|_{g} \\
    & +
    \|V\|_g\|g^{-1}\langle D^2g(V), g^{-1}\rangle\|_g
     + \|V\|_g\|g^{-1}\langle Dg, Dg^{-1}(V)\rangle\|_g\\
     & + \langle g, (DA(V))^2\rangle.\numberthis\label{eq:Vrtterms}
\end{align*}

We state a Lemma which collects our bounds on these terms, but before that, we state a helper Lemma on bounding the drift term $Z(X_t)$.
\begin{lemma}\label{lem:zlemma}
For the drift term $Z$, we have
\begin{align}
    \|Z\|_g \leq \frac{1}{2}\gamma_1 n\sqrt n \label{eq:Zbound}.
\end{align} 
\end{lemma}
\begin{proof}
In a similar way as we described for deriving~\eqref{eq:firstpart}, one can translate the index notation regarding $Z_k = -\sum_{ij}\Gamma_{ij}^k g^{ij}$ into the following matrix form:
$$Z = -\frac{1}{2}g^{-1}\langle g^{-1}, Dg\rangle.$$
Now let $g^{-1} = \sum u_i u_i^T$ be a Cholesky factorization for $g$. We apply the strong-self concordance,
\begin{align*}
    \|Z\|_g^2 = \frac{1}{4}\|\langle g^{-1}, Dg\rangle\|_{g^{-1}}^2 = \frac{1}{4}\sum_i \langle g^{-1}, Dg(u_i)\rangle^2 \leq \frac{1}{4}\gamma_1^2 \sum_i \langle g^{-1}, g\rangle^2 \|u_i\|_g^2 = \frac{1}{4}\gamma_1^2 n^2\sum_i \|u_i\|_g^2 = \frac{1}{4}n^2 \gamma_1^2 \langle g, g^{-1}\rangle \leq \frac{1}{2}\gamma_1^2 n^3,
\end{align*}
which implies 
\begin{align*}
    \|Z\|_g \leq \frac{1}{2}\gamma_1 n\sqrt n.
\end{align*} 
\end{proof}

\begin{lemma}\label{lem:termcollection}
The terms appearing in~\eqref{eq:Vrtterms} can be bounded as follows:
\begin{align*}
    \Big\|g^{-1}Dg(V)Z\Big\| & \leq \gamma_1^2 n \sqrt n \|V\|_g,\\
    \Big\|g^{-1} \langle  Dg^{-1}(V), Dg\rangle\Big\|_g^2 & \leq \gamma_2^2 n\sqrt n \|V\|_g \\
    \Big\| g^{-1} \langle DDg(V), g^{-1}\rangle \Big\|_g^2 & \leq \gamma_3^2 n\sqrt n \|V\|_g\\
    \|DZ(V)\|_g & \leq  \frac{1}{2}(\gamma_1^2 + \gamma_2^2 + \gamma_3^2) n\sqrt n\|V\|.\numberthis\label{eq:dzvbound} 
\end{align*}
\end{lemma}

\begin{proof}
We use the strong self-concordance property to bound these terms. for the first term:
\begin{align}
    \Big\|g^{-1} Dg(V) Z\Big\|_g^2 = \Big\| Dg(V) Z\Big\|_{g^{-1}}^2 = 
    Z^T g^{\frac{1}{2}} \Big(g^{-\frac{1}{2}}Dg(V) g^{-\frac{1}{2}}\Big)^2 g^{\frac{1}{2}} Z
    \leq \gamma_1^2 \|V\|_g^2 \|Z\|_g^2.\label{eq:above1}
\end{align}

Applying Lemma~\ref{lem:zlemma} to Equation~\eqref{eq:above1}, we bound the first term:
\begin{align}
  \Big\|g^{-1} Dg(V) Z\Big\|_g \leq \frac{1}{2}\gamma_1^2 n\sqrt n\|V\|_g.\label{eq:firsttermm}  
\end{align}
For the second term, we use the same trick:
\begin{align*}
    \Big\|g^{-1} \langle  Dg^{-1}(V), Dg\rangle\Big\|_g^2 & = 
    \Big\|\langle Dg^{-1}(V), Dg\rangle\Big\|_{g^{-1}}^2 \\
    & = \sum_i \langle Dg^{-1}(V), Dg(u_i)\rangle^2\\
    &\leq \sum_i\gamma_2^4 n^2\|V\|^2\|u_i\|^2\\
    & = \gamma_2^4 n^2 \|V\|^2 \sum_i \langle u_i u_i^T ,g\rangle = \gamma_2^4 n^2 \|V\|^2 \langle g, g^{-1}\rangle = \gamma_2^4 \|V\|^2 n^3.
\end{align*}

For the third term, we use the self-concordance condition on the fourth tensor:
\begin{align*}
    \Big\| g^{-1} \langle DDg(V), g^{-1}\rangle \Big\|_g^2 & =  \Big\| \langle DDg(V), g^{-1}\rangle \Big\|_{g^{-1}}^2 = 
    \sum_i \langle D^2g(V,u_i), g^{-1}\rangle^2\\
    & \leq \gamma_3^4 \sum_i \langle g, g^{-1}\rangle^2 \|V\|^2 \|u_i\|^2 = \gamma_3^4 n^2 \|V\|^2 \langle g, g^{-1}\rangle = \gamma_3^4 n^3 \|V\|^2.
\end{align*}

For the first term, using Equation~\eqref{eq:firstpart}:
\begin{align*}
    \|DZ(V)\|_{g} & \leq \frac{1}{2} \|D g^{-1}(V) \langle g^{-1}, Dg\rangle\|\\
    & + \frac{1}{2} \|g^{-1} \langle Dg^{-1}(V), Dg\rangle\|\\
    & + \frac{1}{2} \|g^{-1} \langle D^2g(V), g^{-1} \rangle\|.
\end{align*}
\begin{align*}
    \|D g^{-1}(V) \langle g^{-1}, Dg\rangle\|_g & = \|g^{-1} Dg(V) g^{-1} \langle g^{-1}, Dg\rangle\|_g \\
    & \leq \|g^{-1} Dg(V) Z\|_g.
\end{align*}
which we also handled above in~\eqref{eq:firsttermm}. 
Therefore, oevrall
\begin{align*}
   \|DZ(V)\|_{g} \leq  \frac{1}{2}(\gamma_1^2 + \gamma_2^2 + \gamma_3^2) \|V\|_g n \sqrt n.
\end{align*}
This completes the proof of Lemma~\eqref{lem:termcollection}.
\end{proof}

To bound the last finite variation term in~\eqref{eq:secondpart}, we use the key Lemma~\ref{lem:critical}, proved in Appendix~\ref{app:criticlem}, which states some regularity about the square root of the metric, which for sake of convenience we restate here:

\begin{lemma}
Suppose the metric $g$ is $\gamma_1$ normal self-concordant. Then, for the square root matrix $A(x) = \sqrt{2g^{-1}}$, we have
\begin{align*}
    &\langle g, (DA(V))^2\rangle \leq n\gamma_1^2\|V\|_g^2,\\
    &\langle g^{1/2}, DA(V)\rangle \leq \sqrt n\gamma_1\|V\|_g.
\end{align*}
\end{lemma}

\begin{lemma}\label{lem:secondvon}
Suppose the metric $g$ is $\gamma_1$ normal self-concordant. Then, for the square root matrix $A(x) = \sqrt{2g^{-1}}$, we have
\begin{align*}
    V^Tg(DA(V))^2gV \leq n\gamma_1^2\|V\|_g^4.
\end{align*}
\end{lemma}

\begin{proof}[Proof of Lemma~\ref{lem:Vbound}]
Using Lemmas~\ref{lem:dvsquared} and~\ref{lem:termcollection} for the stochastic term, we get:
\begin{align*}
    R_t \lesssim (\gamma_1^4 + \gamma_2^4 + \gamma_3^4) n^3\|V_t\|_g^2.\numberthis\label{eq:rtbound}
\end{align*}

On the other hand, for the local martingale part $M_t$, we have
\begin{align*}
    \|M_t\|^2 \leq 8V^T g (DA(V))^2 g V + 2 V^T Dg(V)A^2Dg(V)V.
\end{align*}
But for the second term:
\begin{align}
    V^T Dg(V) A^2 Dg(V) V =
    V^T g^{\frac{1}{2}} (A Dg(V) A)^2 g^{\frac{1}{2}} V \leq
    \gamma_1^2\|V\|_g^2 V^T g^{\frac{1}{2}} (A g A)^2 g^{\frac{1}{2}} V = \gamma_1^2\|V\|_g^4.\label{eq:afirst}
\end{align}
For the first term, we use Lemma~\ref{lem:secondvon}.
Combining Equation~\eqref{eq:afirst} and Lemma~\ref{lem:secondvon}:
\begin{align*}
    \|M_t\|^2 \leq (n+1)\gamma_1^2 \|V\|_g^4.\numberthis\label{eq:mtbound}
\end{align*}
Now from Ito isometry
\begin{align*}
    \E(\int_{s=0}^t M_s dB_s)^2 & \leq \E \int_{s=0}^t \|M_s\|^2 ds\\
    & \leq \E \int (n+1)\gamma_1^2 \|V_s\|_g^4 ds\\
    & = (n+1)\gamma_1^2\int \E \|V_s\|_g^4 ds.
\end{align*}
Combining Equations~\eqref{eq:rtbound} and~\eqref{eq:mtbound}, and plugging into Equation~\eqref{eq:firstpart}:
\begin{align*}
    \E \|V\|_g^4 & \leq \E \|V_0\|_g^4 + 2  \E(\int_{s=0}^t M_s dB_s)^2 + 2\E(\int R_s ds)^2\\
    & \lesssim (n+1)\gamma_1^2\int \E \|V_s\|_g^4 ds 
    + \E (\int_{s=0}^t (\gamma_1^4 + \gamma_2^4 + \gamma_3^4) n^3\|V\|_g^2)^2 \\
    & \leq (n+1)\gamma_1^2\int \E \|V_s\|_g^4 ds  
    + t (\gamma_1^4 + \gamma_2^4 + \gamma_3^4)^2 n^6\E \int_{s=0}^t \|V\|_g^4\\
    & \leq C \int_{s=0}^t \E \|V_s\|_g^4.
\end{align*}
But this implies
\begin{align*}
    \E \|V\|_g^4 \leq e^{tC} \|V_0\|_g^4 \numberthis\label{eq:V4bound}
\end{align*}
for bounded $t$. 

Now to prove that the local martingale part is actually a martingale, it is enough to show that for every time $t$, the expectation of its quadratic variation is bounded:
\begin{align*}
   \E \Big[ \int M_s dB_s \Big]_t & = 
   \E \int_{s=0}^t \|M_s\|^2 ds\\
   & = \int_{s=0}^t \E \|M_s\|^2 ds\\
   & \leq \int_{s=0}^t (n+1)\gamma_1^2 \E\|V\|_g^4 < \infty.
\end{align*}
where we used Equations~\eqref{eq:mtbound} and~\eqref{eq:V4bound}. Now if we assume a stopping time $\tau = t$ for some fixed time $t$, then $N_t = \int_{s=0}^{t \wedge \tau} M_s dB_s$ is bounded since we proved $\E [ N ] < \infty$, which implies
that $\int_{s=0}^{t} M_s dB_s$ is also a martingale (send $t$ to infinity). 
 Plugging this result back into Equation~\eqref{eq:firstpart}, from the martingale property
\begin{align*}
    \E \|V_t\|_g^2 =  \E \|V_0\|_g^2 + \E \int_{s=0}^t R_s ds & = \int_{s=0}^t \E R_s ds\\
    & \lesssim \int \E(\gamma_1^2 + \gamma_2^2 + \gamma_3^2) n\sqrt n\|V_s\|_g^2ds\\
    & = (\gamma_1^2 + \gamma_2^2 + \gamma_3^2) n\sqrt n\int \E\|V_t\|_g^2ds,
\end{align*}

which implies
\begin{align}
    \E \|V_t\|_g^2 \leq \|V_0\|_g^2\exp\{c_1tn^{3/2}(\gamma_1^2 + \gamma_2^2 + \gamma_3^2)\} & = \|b(t,x_0)\|_g^2\exp\{c_1tn^{3/2}(\gamma_1^2 + \gamma_2^2 + \gamma_3^2)\}\\
    & = \|\nabla F(x_0)\|_g^2\exp\{c_1tn^{3/2}(\gamma_1^2 + \gamma_2^2 + \gamma_3^2)\}.\label{eq:Vbound}
\end{align}
Therefore, assuming
\begin{align*}
    t &\lesssim \frac{1}{c_1n^{3/2}(\gamma_1^2 + \gamma_2^2 + \gamma_3^2)},
\end{align*}
then 
\begin{align}
    \E \|V_t\|_g^2 \leq 2 \|\nabla F(x_0)\|_g^2,
\end{align}
which completes the proof of Lemma~\ref{lem:Vbound}.
\end{proof}

\subsection{Function derivatives}\label{app:derivativetensorbounds}

Next, we want to expand the vector field 
$J(t) = \nabla F(X_t)$ 
along the stochastic curve. Before doing so, we prove some useful assumptions on the derivative tensors of the function $F$. We assume that $F$ is $L_2$ gradient Lipschitz and $L_3$ hessian Lipschitz over the manifold.

The $L_2$ Lipschitz assumption is equivalent to that for all vector fields $W$, we have
\begin{align}
    \|\nabla J(W)\|_g \leq L_2\|W\|_g,\label{eq:covariantprop}
\end{align}
where $\nabla$ is the covariant derivative. Equivalently, in the Euclidean chart,
\begin{align}
    \|DJ(W) + W_i \Gamma_{ij}^k J_j \|_g \leq L_2\|W\|_g.\label{eq:ltwoproperty}
\end{align}
or
\begin{align}
    \|DJ(W) + \frac{1}{2}g^{-1}Dg(J) W\|_g \leq L_2\|W\|_g.\label{eq:anotherform}
\end{align}

Now it is straightforward to translate property~\eqref{eq:anotherform} into a spectral bound property in the Euclidean chart. We state this in the following lemma.

\begin{lemma}\label{lem:opnorm}
The $L_2$ Lipschitz property is equivalent to the following matrix operator norm bound:
\begin{align}\label{lem:lipproperty}
    \|g^{\frac{1}{2}}(DJ + g^{-1}Dg(J))g^{-\frac{1}{2}}\|_{op}^2 \leq L_2.
\end{align}
\begin{proof}
For an arbitrary unit vector $W', \|W'\| = 1$, we have
\begin{align*}
    \|g^{-1/2}W'\|_g = 1.
\end{align*}
Hence, from property~\eqref{eq:anotherform}:
\begin{align*}
    \|\Big(g^{\frac{1}{2}}(DJ + g^{-1}Dg(J))g^{-\frac{1}{2}}\Big)W'\|_2 & =
    \|g^{\frac{1}{2}}(DJ + g^{-1}Dg(J))(g^{-\frac{1}{2}}W')\|_2 & \\
    & = \|(DJ + g^{-1}Dg(J))(g^{-\frac{1}{2}}W')\|_g\\
    & \leq 1.
\end{align*}
where $\|.\|_2$ is the usual Euclidean 2-norm. This completes the proof. 
\end{proof}
\end{lemma}

In the following, we start by translating the gradient and Hessian Lipschitz properties over the manifold to Equivalent norm bounds in the Euclidean chart.

\subsubsection{Gradient Lipschitz property}
\begin{lemma}\label{lemma:derivativebound}
For all vectors $X_1$, 
$$\|DJ(X_1)\| \leq (L_2 + \gamma_1\|J\|)\|X_1\|.$$
\end{lemma}
\begin{proof}
From the gradient Lipschitz condition, using ~eqref{eq:anotherform} and the triangle inequality, we can write
\begin{align*}
    \|DJ(X_1)\| & \leq L_2\|X_1\| + \|g^{-1}Dg(J)X_1\|
    \leq L_2\|X_1\| + \|Dg(J)X_1\|_{g^{-1}} \leq L_2\|X_1\| + \gamma_1\|J\|\|X_1\|.
\end{align*}
\end{proof}

Next, we remark that the $L_3$-Hessian Lipschitz property is equivalent to $|\nabla^3 F(X_1, X_2, X_3)| \leq L_3\|X_1\|\|X_2\|\|X_3\|$, where $\nabla^3$ is the third order derivatives tensor of $f$ on the manifold. Expanding it out:
\begin{align}
    \nabla^3 F(X_1, X_2, X_3) & = X_1(Hess(X_2, X_3)) - Hess(\nabla_{X_1}X_2, X_3) - Hess(X_2, \nabla_{X_1}X_3) \\
    & = X_1(\langle \nabla_{X_2} \grad F, X_3\rangle) - \langle \nabla_{X_3} \grad F, \nabla_{X_1} X_2\rangle - \langle \nabla_{X_2} \grad F, \nabla_{X_1} X_3\rangle\\
    & = \langle X_3, \nabla_{X_1}\nabla_{X_2} \grad F \rangle  - \langle \nabla_{X_3} \grad F, \nabla_{X_1} X_2 \rangle.\label{eq:derivation1}
\end{align}
Next, we prove a Lemma to give a bound on the second derivatives of $F$ using the Hessian Lipschitz condition.

\subsubsection{Hessian Lipschitz}
\begin{lemma}\label{lem:DtwoJlemma}
For $J = \grad F$, we have 
\begin{align}
\|D^2J(X_1, X_2)\| \leq ( \gamma_1L_2 + (\gamma_1^2 + \gamma_3^2) \|J\| + L_3)\|X_1\|\|X_2\|\label{eq:Jhessianbound}    
\end{align}
\end{lemma}
\begin{proof}
For arbitrary vectors $X_1, X_2$ at some point $x$, We set in the derivation in~\eqref{eq:derivation1} $X_2$ to be the parallel transported vector field along an arbitrary curve starting from $x$ with initial speed $X_1$. This way, the second term $\langle \nabla_{X_3} \grad F, \nabla_{X_1} X_2 \rangle$ is zero. 
Also, writing the conditin $\nabla_{X_1} X_2 = 0$ in the Euclidean coordinates, we get
\begin{align*}
    DX_2(X_1) + \frac{1}{2}g^{-1}Dg(X_1)X_2 = 0,
\end{align*}
which implies
\begin{align*}
   \|DX_2(X_1)\|_g & = \frac{1}{2}\|g^{-1}Dg(X_1)X_2\|_g = 
   \frac{1}{2}\|Dg(X_1)X_2\|_{g^{-1}} \\
   & = 
   \frac{1}{2}\sqrt{Tr({X_2}^Tg^{1/2}(g^{-1/2}Dg(X_1)g^{-1/2})^2g^{1/2}X_2)} \\
   & \leq \frac{1}{2}\gamma_1\|X_1\|_g\sqrt{Tr(X_2^T g X_2)} = \frac{1}{2}\gamma_1\|X_1\|_g\|X_2\|_g.\numberthis\label{eq:extraterm}
\end{align*}

Now from the first term we get
\begin{align*}
    |\langle X_3, \nabla_{X_1}\nabla_{X_2}J\rangle| \leq  L_3\|X_1\|\|X_2\|\|X_3\|
\end{align*}
which implies
\begin{align*}
    \|\nabla_{X_1}\nabla_{X_2}J\| \leq L_3\|X_1\|\|X_2\|.
\end{align*}
Now we expand this double covariant derivative:
\begin{align*}
    \nabla_{X_1}\nabla_{X_2} J & = DJ(D(X_2)(X_1)) + D^2J(X_1, X_2) + \frac{1}{2}g^{-1}Dg(X_1)DJ(X_2) - \frac{1}{2}g^{-1}Dg(X_1)g^{-1} Dg(X_2) J \\
    & + \frac{1}{2}g^{-1}Dg(D(X_2)(X_1))J + \frac{1}{2} g^{-1}Dg(X_1, X_2) J \\
    & + \frac{1}{2} g^{-1} Dg(X_2)DJ(X_1) + \frac{1}{4} g^{-1} Dg(X_1)g^{-1} Dg(X_2) J.\numberthis\label{eq:allterms}
\end{align*}
Therefore
\begin{align*}
    \|D^2J(X_1, X_2)\| & \leq  \frac{1}{2}\|g^{-1}Dg(X_1)DJ(X_2)\| + \frac{1}{4}\|g^{-1}Dg(X_1)g^{-1} Dg(X_2) J\| \\
    & + \frac{1}{2} \|g^{-1}Dg(X_1, X_2) J\| + \frac{1}{2} \|g^{-1} Dg(X_2)DJ(X_1)\| + L_3\|X_1\|\|X_2\|\\
    & + \|DJ(D(X_2)(X_1))\| + \frac{1}{2}\|g^{-1}Dg(D(X_2)(X_1))J\|.\numberthis\label{eq:dounblecovariantderivative}
\end{align*}
For the first term, using Lemma~\ref{lemma:derivativebound} and again self-concordance
\begin{align*}
    \|g^{-1}Dg(X_1)DJ(X_2)\| & = \|Dg(X_1)DJ(X_2)\|_{g^{-1}} = 
    DJ(X_2)^T g^{\frac{1}{2}}(g^{-\frac{1}{2}}Dg(X_1)g^{-\frac{1}{2}})^2 g^{\frac{1}{2}}DJ(X_2)\\
    & \leq \gamma_1 \|X_1\| \|DJ(X_2)\|_{g} \leq \gamma_1 \|X_1\| (L_2 + \gamma_1\|J\|)\|X_2\|.
\end{align*}
For the second term, using a similar technique
\begin{align*}
    \|g^{-1}Dg(X_1)g^{-1} Dg(X_2) J\| = \|Dg(X_1)g^{-1} Dg(X_2) J\|_{g^{-1}} \leq 
    \gamma_1 \|X_1\|\|Dg(X_2) J\|_{g^{-1}} \leq \gamma_1^2\|X_1\|\|X_2\| \|J\|.
\end{align*}
For the third term:
\begin{align*}
    \|g^{-1}Dg(X_1, X_2) J\| = \|Dg(X_1, X_2) J\|_{g^{-1}} \leq \gamma_3^2 \|X_1\|\|X_2\|\|J\|. 
\end{align*}
The fourth term is similar to the first term. 
For the fifth term, using Lemma~\ref{lemma:derivativebound}, self concordance and Equation~\eqref{eq:extraterm}:
\begin{align*}
    \|DJ(DX_2(X_1))\| & \leq (L_2 + \gamma_1\|J\|)\|DX_2(X_1)\|\\
    & \leq \frac{1}{2}\gamma_1(L_2 + \gamma_1\|J\|)\|X_1\|\|X_2\|.
\end{align*}
For the sixth term, similarly
\begin{align*}
    \frac{1}{2}\|g^{-1}Dg(D(X_2)(X_1))J\| & = \frac{1}{2}\sqrt{J^Tg^{1/2}(g^{-1/2}Dg(D(X_2)(X_1))g^{-1/2})^2g^{1/2}J}\\
    & \leq \frac{1}{2}\gamma_1\|D(X_2)(X_1)\|\sqrt{J^T g J}\\
    & \leq \frac{1}{2}\gamma_1^2 \|X_1\|\|X_2\|\|J\|.
\end{align*}

So overall, ignoring constants, we get
\begin{align}
    \|D^2J(X_1, X_2)\| \leq ( \gamma_1L_2 + (\gamma_1^2 + \gamma_3^2) \|J\| + L_3)\|X_1\|\|X_2\|.
\end{align}
\end{proof}

\subsubsection{Some bounds on derivatives of the metric and the function}

Next, we prove some useful bounds concerning the derivatives of $F$and the metric $g$ which will be useful later on in the analysis.

\begin{lemma}\label{lem:DtwoJlemma2}
We have
\begin{align*}
    \Big\|\langle D^2J, g^{-1}\rangle\Big\| \leq n( \gamma_1L_2 + (\gamma_1^2 + \gamma_3^2) \|J\| + L_3).
\end{align*}
\end{lemma}
\begin{proof}
\begin{align*}
    \Big\|\langle D^2J, g^{-1}\rangle\Big\| & =
    \Big\|\langle D^2J, \sum u_i u_i^T\rangle\Big\|\\
    & =  \Big\|\sum_i D^2J(u_i,u_i)\Big\|\\
    & \leq  \sum \Big\| D^2J(u_i,u_i)\Big\| \\
    & \leq  ( \gamma_1L_2 + (\gamma_1^2 + \gamma_3^2) \|J\| + L_3)   \sum \|u_i\|^2\\
    & = n( \gamma_1L_2 + (\gamma_1^2 + \gamma_3^2) \|J\| + L_3).
\end{align*}
\end{proof}

\begin{lemma}\label{lem:repeatedterm}
We have
\begin{align}
    \|g^{-1}\langle Dg , DJg^{-1}\rangle\| \leq \gamma_1 n\sqrt n L_2 + \gamma_1^2 n \|J\|.~\label{eq:lemmaresult1}
\end{align}
\end{lemma}
\begin{proof}
But note that
\begin{align*}
    & \|g^{-1}\langle Dg , DJg^{-1}\rangle\|^2  = \|\langle Dg , DJg^{-1}\rangle\|_{g^{-1}}^2 = \sum_i \langle Dg(u_i), DJg^{-1}\rangle^2 \\
    & \leq 2\sum_i \langle Dg(u_i), (DJ + g^{-1}Dg(J))g^{-1}\rangle^2 + 2\sum_i \langle Dg(u_i), g^{-1}Dg(J) g^{-1}\rangle^2\\
    & = 2\sum_i \langle g^{-\frac{1}{2}} Dg(u_i)g^{-\frac{1}{2}}, g^{\frac{1}{2}}(DJ + g^{-1}Dg(J))g^{-\frac{1}{2}}\rangle^2 + 2\sum_i \langle g^{-\frac{1}{2}}Dg(u_i)g^{-\frac{1}{2}}, g^{-\frac{1}{2}}Dg(J) g^{-\frac{1}{2}}\rangle^2\\
    & \lesssim \sum \|g^{-\frac{1}{2}} Dg(u_i)g^{-\frac{1}{2}}\|_1^2 \|g^{\frac{1}{2}}(DJ + g^{-1}Dg(J))g^{-\frac{1}{2}}\|_{op}^2 
    + \sum_i \|g^{-\frac{1}{2}}Dg(u_i)g^{-\frac{1}{2}}\|_F^2 \|g^{-\frac{1}{2}}Dg(J) g^{-\frac{1}{2}}\|_F^2.\\
    & \leq \gamma_1^2 \|g^{\frac{1}{2}}(DJ + g^{-1}Dg(J))g^{-\frac{1}{2}}\|_{op}^2 \sum \|u_i\|_g^2 n^2 + \gamma_1^4 \|J\|^2\sum \|u_i\|_g^2 n\\
    & =  \gamma_1^2\|g^{\frac{1}{2}}(DJ + g^{-1}Dg(J))g^{-\frac{1}{2}}\|_{op}^2 n^3 + \gamma_1^4 \|J(t)\|^2n^2.
\end{align*}
Now note that from the $L_2$ lipschitz property, using Lemma~\ref{lem:opnorm}: 
\begin{align}
\|g^{\frac{1}{2}}(DJ + g^{-1}Dg(J))g^{-\frac{1}{2}}\|_{op} \leq L_2.\label{eq:firstoff}    
\end{align}

Overall, we get
\begin{align*}
    \|g^{-1}\langle Dg , DJg^{-1}\rangle\| \leq \gamma_1 n\sqrt n L_2 + \gamma_1^2 n \|J\|.
\end{align*}
\end{proof}

\begin{lemma}\label{lem:repeatedterm2}
We have
\begin{align*}
   \langle g, DJg^{-1}(DJ)^T\rangle  \leq n((L_2 + \gamma_1\|J\|))^2.
\end{align*}
\end{lemma}
\begin{proof}
Applying Lemma~\ref{lemma:derivativebound}:
\begin{align*}
    \langle g, DJg^{-1}(DJ)^T\rangle & = Tr(g^{1/2}DJg^{-1/2} (g^{1/2}DJg^{-1/2})^T)\\
    & = \|g^{1/2}DJg^{-1/2}\|_F^2\\
    & \leq \sqrt n\|g^{1/2}DJg^{-1/2}\|_{op}^2 \\
    & \leq n((L_2 + \gamma_1\|J\|))^2.
\end{align*}
\end{proof}

\subsection{Bounding the average norm of the gradient}

The goal of this section is to prove Lemma~\ref{lem:Jbound}.
Using Lemma~\ref{lem:normexpantion} applied to the vector field $J$:
\begin{align*}
    d\|J\|_g^2 & = 
     2\langle dJ(t) ,J\rangle_g
    + \sum_{i,j,p} J_i \partial_p g_{ij} J_j dX_p+ \sum_{i,j,p,r}\frac{1}{2}J_i J_j\partial_{p,r}g_{ij} d\langle X_p, X_r\rangle + \sum_{i,j,p} 
     \partial_p g_{ij} d\langle J_i, X_p\rangle J_j + \sum_{i,j}\frac{1}{2} g_{ij} d\langle J_i, J_j\rangle.\numberthis\label{eq:f1}\\
     & = N_t^T dB_t + Q_t dt.
\end{align*}

Similar to~\eqref{eq:simplifications}, we have the simplifications
\begin{align}
    &\frac{1}{2}J_iJ_j \partial_{p,r} g_{i,j}d\langle X_p, X_r\rangle = J_iJ_j \partial_{p,r} g_{i,j} g^{pr} dt,\label{eq:aval}\\
    & \partial_p g_{ij} J_j d\langle J_i, X_p\rangle = \partial_p g_{ij}J_j \partial_\ell J_i \langle dX_\ell, dX_p \rangle dt = 
    2\partial_p g_{ij}J_j \partial_\ell J_i g^{\ell p} dt
    ,\label{eq:dovom}\\
    & \frac{1}{2} g_{ij}d\langle J_i, J_j \rangle = \frac{1}{2} g_{ij} \partial_\ell J_i \partial_r J_j d\langle X_r, X_\ell\rangle = g_{ij} \partial_\ell J_i \partial_r J_j g^{r\ell}dt.\label{eq:simplifications2}
\end{align}

We start by giving a bound on the finite variation parts $Q_t$. 
The first term can be expanded as
\begin{align*}
     2\langle dJ(t) ,J\rangle & = 2\langle DJdX_t + \sum_{\ell, r}\frac{1}{2}\partial_{\ell, r} J_i g^{\ell r} dt, J\rangle\\
     & = 2\langle DJ AdB_t + DJ Z dt + \sum_{\ell, r}\frac{1}{2}\partial_{\ell, r} J_i g^{\ell r} dt, J\rangle\numberthis\label{eq:firstterme}
\end{align*}
We bound the first $dt$ term above with another term later. 
For the second term
\begin{align*}
    \frac{1}{2} \langle \partial_{\ell, r} J_i g^{\ell r}, J\rangle  & = \frac{1}{2}\langle J, \langle D^2J, g^{-1} \rangle\rangle\\
    & \leq \|J\|\Big\| \langle D^2J, g^{-1} \rangle\Big\|\\
    & \leq n( \gamma_1L_2 + (\gamma_1^2 + \gamma_3^2) \|J\| + L_3) \|J\|.\numberthis\label{eq:1term}
\end{align*}

For the second term
\begin{align*}
\sum_{i,j,p} J_i \partial_p g_{ij} J_j dX_p & = J^T Dg(J)AdB_t + J^T Dg(J)Z_t dt.
\end{align*}
For its quadratic variation part, we combine it with term $\langle DJZ, J\rangle dt$ from the first term in~\eqref{eq:firstterme} and apply~\eqref{eq:anotherform} plus~\eqref{eq:Zbound}:
\begin{align}
  \langle 2DJZ + g^{-1}Dg(J)Z, J \rangle \leq \|J\|_g 2\|DJZ + \frac{1}{2}g^{-1}Dg(J)Z\| \leq 2L_2\|J\|_g\|Z\|_g \leq 2L_2 \gamma_1 n\sqrt n \|J\|_g .\numberthis\label{eq:2term}
\end{align}
For the third term in~\eqref{eq:f1}, according to~\eqref{eq:aval}, is equal to:
\begin{align}
    \sum_{i,j,p,r}J_iJ_j \partial_{p,r} g_{i,j} g^{pr} & = J^T \langle D^2g, g^{-1}\rangle J 
     \nonumber\\
    & \leq \frac{1}{2}\sum_i  J^TD^2g(u_i,u_i)J \nonumber\\
    & \leq  \frac{1}{2}\gamma_3^2 (\sum_i\|u_i\|_g^2) \|J\|_g^2  = \frac{1}{2} n\gamma_3^2 \|J\|_g^2.\label{eq:3term}
\end{align}
Next, for the forth term in~\eqref{eq:f1}, using~\eqref{eq:dovom}:
\begin{align*}
    \sum_{i,j,p,\ell} 2\partial_p g_{ij}J_j \partial_\ell J_i g^{\ell p} = 2J^T \langle Dg, DJg^{-1} \rangle \leq 2\|J\| \|g^{-1}\langle Dg, DJg^{-1} \rangle\|_g.
\end{align*}
Further applying Lemma~\ref{lem:repeatedterm}:
\begin{align}
    \sum_{i,j,p,\ell} 2\partial_p g_{ij}J_j \partial_\ell J_i g^{\ell p} \leq \gamma_1 n\sqrt n L_2\|J\| + \gamma_1^2 n \|J\|^2.\numberthis\label{eq:4term}
\end{align}
Finally for the last term in~\eqref{eq:f1}, again using the reformulation in~\eqref{eq:simplifications2}:
\begin{align}
    \sum_{i,j,p,\ell}2g_{ij} \partial_\ell J_i \partial_r J_j g^{r\ell} = 2\langle g, DJg^{-1}(DJ)^T \rangle \leq n((L_2 + \gamma_1\|J\|))^2,\label{eq:5term} 
\end{align}
where we used Lemma~\ref{lem:repeatedterm2}.
Combining Equations~\eqref{eq:1term},~\eqref{eq:2term},~\eqref{eq:3term},~\eqref{eq:4term},~\eqref{eq:5term}:

\begin{align*}
    |Q_t| & \leq n( \gamma_1L_2 + (\gamma_1^2 + \gamma_3^2) \|J\| + L_3) \\
    & + 2L_2 \gamma_1 n\sqrt n \|J\|_g\\
    & + \frac{1}{2} n\gamma_3^2 \|J\|_g^2\\
    & + \gamma_1 n\sqrt n L_2\|J\| + \gamma_1^2 n \|J\|^2\\
    & + n(L_2 + \gamma_1\|J\|)^2\\
    & \lesssim (n\gamma_1L_2 + nL_3 + nL_2^2) \\
    & + (n(\gamma_1^2 + \gamma_3^2) + 3L_2\gamma_1n\sqrt n)\|J\|\\
    & + (n\gamma_3^2 + n\gamma_1^2)\|J\|^2 \\
    & \coloneqq \beta_0 + \beta_1 \|J\| + \beta_2 \|J\|^2.\numberthis\label{eq:finitevariation}
\end{align*}
On the other hand, for the martingale parts, applying Lemma~\ref{lemma:derivativebound} and using similar tricks as before
\begin{align*}
    \Big[ \int N_t^T dB_t\Big] & = \Big[ \int_{s=0}^t \langle DJAdB_s + g^{-1}Dg(J)AdB_s, J\rangle  \Big] \\
    & = \Big[ \int_{s=0}^t  (J^TgDJA + J^TDg(J)A)dB_s  \Big]\\
    & = \int_{s=0}^t  \Big(J^TgDJA^2(DJ)^TgJ + J^TDg(J)A^2Dg(J)J + 2J^TgDJA^2Dg(J)J \Big)ds\\
    & \leq \int \Big(2J^TgDJA^2(DJ)^TgJ + 2J^TDg(J)A^2Dg(J)J\Big) ds\\
    &  \leq \int \Big(2J^Tg^{1/2}(g^{1/2}DJg^{-1/2})(g^{1/2}DJg^{-1/2})^Tg^{1/2}J + 2\|J\|_g^2 J^TgJ\Big) ds\\
    & \leq 
    2\int (\|g^{1/2}DJg^{-1/2}\|_{op}^2 \|J\|_g^2 + 2\|J\|_g^4) ds\\
    & \leq 2\int ((L_2 + \gamma_1\|J\|)^2 \|J\|_g^2 + 2\|J\|_g^4) ds\\
    & \leq \int (c_1 + c_2\|J\|^4) ds,\numberthis\label{eq:quadraticobund} 
\end{align*}
for constants $c_1, c_2$, where in the last line we used the AM-GM inequality. 
On the other hand, from Ito isometry, we have 
\begin{align*}
    \E \|J_t\|^4 & =  (\|J_0\|^2 + \int N_s^T dB_t + \int Q_sds)^2\\ 
    & \leq 2 \|J_0\|^4 + 2\E(\int N_s^T dB_t)^2 + 2\E(\int Q_sds)^2\\
    & \leq 2 \|J_0\|^4 + 2\E\langle \int N_s^T dB_s\rangle + 2\E(\int \beta_0 + \beta_1 \|J\| + \beta_2\|J\|^2)^2 \\
    & \leq 2\|J_0\|^4 + 2\int (c_3 + c_4\E \|J\|^4) ds,
\end{align*}
where in the last line we applied~\eqref{eq:quadraticobund} plus AM-GM inequality. But this ODE means 
\begin{align}
    \E \|J_t\|^4 \leq (2\E \|J_0\|^4)e^{c_4 t} + \frac{c_3}{c_4}(e^{c_4 t} - 1) < \infty.\label{eq:jfour}
\end{align}
Moreover, from~\eqref{eq:quadraticobund}, this implies that 
\begin{align*}
    \E \Big[ \int N_s^T dB_s\Big]  < \infty,
\end{align*}
which implies that the local martingale part is indeed a martingale. Therefore
\begin{align*}
    \E \int N_t^T dB_t = 0.
\end{align*}

Applying this to Equation~\eqref{eq:finitevariation}, we can then write
\begin{align*}
   \E \|J_t\|^2 & = \|J_0\|^2 + \E \int N_s^T dB_s + \E \int_{s=0}^t Q_s ds \\
   & = \|J_0\|^2 + \E \int_{s=0}^t Q_s ds \\
   & \leq \|J_0\|^2 + \E \int_{s=0}^t Q_s ds\\
   & \leq \|J_0\|^2 + \E \int_{s=0}^t (\beta_0 + \beta_1\|J\| + \beta_2\|J\|^2) ds\\
   & \leq \|J_0\|^2 + \E \int_{s=0}^t (\beta_0 + \beta_1\|J\| + \beta_2\|J\|^2) ds\\
   & \lesssim \|J_0\|^2 + \E \int_{s=0}^t (\beta_0 + \beta_1 + (\beta_1 + \beta_2)\|J\|^2) ds,
\end{align*}
where in the last line we applied the AM-GM inequality. Hence, we define the following differential equation
\begin{align*}
    &f(0) = \|J_0\|,\\
    &f'(t) = af(t) + b,
\end{align*}
for
\begin{align*}
    &a = (\beta_1 + \beta_2),\\
    &b = \beta_0 + \beta_1,
\end{align*}
where $\beta_0, \beta_1, \beta_2$ are defined in Equation~\eqref{eq:finitevariation}. whose solution is
\begin{align*}
    f(t) = f(0)e^{at} + \frac{b}{a}(e^{at} - 1).
\end{align*}
Now to avoid the exponential functions from explosion, we impose the following condition on $t$:
\begin{align}
    t \leq \frac{1}{4\beta_1 + 4\beta_2}.\label{eq:condition1}
\end{align}
 Then, from the inequality for $e^{c} \leq 2c + 1$ for $0 \leq c \leq \frac{1}{2}$, we get
\begin{align}
    \E \|J_t\|^2 & \leq \|J_0\|^2 e^{at} + \frac{b}{a}(e^{at} - 1) \nonumber \\
    & \leq  \|J_0\|^2 (1 + 2at) + 2bt \nonumber\\
    & = \|J_0\|^2(2(\beta_1 + \beta_2) t + 1) + 2(\beta_0 + \beta_1)t \label{eq:ODEtrick}\\
    & \lesssim \|J_0\|^2 + (\beta_0+\beta_1)t.\nonumber
\end{align}
Defining $\xi = \beta_1 + \beta_2$ completes the proof of Lemma~\ref{lem:Jbound}.

\subsection{Bounding the difference}

Now we write our desired quantity, namely $\|V(X_t) - \nabla F(X_t)\|^2$, as a stochastic integral. For brevity, we refer to $\nabla F(X_t)$ by $J_t$. Here we get additional terms because the quadratic variation between $V$ and $J$ is not zero. 
We start by expanding $\|V_t - J_t\|^2$ using Ito rule:
\begin{align*}
    d\langle V - J, V - J\rangle(t)
    & = 2\langle dV - dJ ,V - J\rangle\\
    & + \sum_{i,j,p} (V_i - J_i) \partial_p g_{ij} (V_j - J_j) dX_p \\
    & + \sum_{i,j,p,r}\frac{1}{2}(V_i - J_i) (V_j - J_j)\partial_{p,r}g_{ij} d\langle X_p, X_r\rangle \\
    & + 
     \sum_{i,j,p} \partial_p g_{ij} d\langle V_i - J_i, X_p\rangle (V_j - J_j) \\
     & + \sum_{i,j} \frac{1}{2} g_{ij} d\langle V_i - J_i, V_j - J_j\rangle.\numberthis\label{eq:mequation}
\end{align*}
We have similar simplifications as in Equation~\eqref{eq:simplifications} here, for the third term
\begin{align*}
    &\frac{1}{2}(V_i - J_i)(V_j - J_j) \partial_{p,r} g_{i,j}d\langle X_p, X_r\rangle = (V_i - J_i)(V_j - J_j) \partial_{p,r} g_{i,j} g^{pr},\numberthis\label{eq:thirdtermm}
\end{align*}
For the fourth term
\begin{align*}
    \partial_p g_{ij} (V_j - J_j) d\langle V_i - J_i, X_p\rangle & = \partial_p g_{ij}(V_j - J_j) \Big((Dg^{-1}(V))_{ip} - \partial_r J_i d\langle X_r, X_p\rangle\Big) \\
    & = \partial_p g_{ij}(V_j - J_j) \Big((Dg^{-1}(V))_{ip} - 2\partial_r J_i g^{rp}\Big) .\numberthis\label{eq:fourthtermm}
\end{align*}
For the fifth term
\begin{align*}
    \frac{1}{2} g_{ij}d\langle V_i - J_i, V_j - J_j \rangle & = \frac{1}{2} g_{ij} {(DA(V))^2}_{ij}dt -  \partial_\ell J_{j} d\langle V_{i}, X_\ell\rangle - \partial_\ell J_{i} d\langle V_{j}, X_\ell\rangle +  \partial_\ell J_i \partial_r J_j d\langle X_r, X_\ell \rangle \\
    & = \frac{1}{2} g_{ij} \Big({(DA(V))^2}_{ij} -  \partial_\ell J_{j} (DA(V) A)_{i\ell} - \partial_\ell J_{i} (DA(V) A)_{j\ell} +  \partial_\ell J_i \partial_r J_j d\langle X_r, X_\ell \rangle\Big) dt\\ 
    & = \frac{1}{2} g_{ij} \Big({(DA(V))^2}_{ij} -  DA(V)A(DJ)^T - DJ DA(V)A +  (DJ g^{-1} {(DJ)}^T)_{ij} \Big)dt
\end{align*}

In the following, we handle the terms in~\eqref{eq:mequation} one by one.

\subsubsection{First term}
Opening the first term (the $i$ index below enumerate the entries of the vector in the dot product):
\begin{align*}
    & 2\langle dV, V - J\rangle - 2\langle dJ, V - J\rangle,\\
    & = 2\langle DA(V)dB_t + DZ(V)dt, V - J\rangle - 2\langle DJ dX_t + \frac{1}{2} \partial_{\ell, r} J_i g^{\ell r} dt , V - J \rangle \\
    & = 2\langle DA(V)dB_t + DZ(V)dt, V - J\rangle - 2\langle DJ AdB_t + DJ Zdt + \frac{1}{2} \partial_{\ell, r} J_i g^{\ell, r} dt , V - J \rangle\numberthis\label{eq:p2}
\end{align*}
Now we show that the local martingale parts in~\eqref{eq:p2} are indeed martingale. For the first local martingale term:
\begin{align*}
    \E\Big[ \int_{s=0}^t \langle DA(V)dB_s, V-J\rangle \Big] & =  
    \E\int (V-J)^Tg(DA(V))^2 g (V-J) ds\\
    & = \E\int Tr(g^{1/2}(V-J)(V-J)^Tg^{1/2} g^{1/2}(DA(V))^2 g^{1/2}) ds\\
    & \leq \E\int \|g^{1/2}(V-J)(V-J)^Tg^{1/2}\|_{op} \|g^{1/2}(DA(V))^2 g^{1/2}\|_1 ds,
\end{align*}
where recall that $\|.\|_1$ refers to matrix nuclear norm. Now using Equations~\eqref{eq:mtbound} and~\eqref{eq:V4bound}, we can further bound the above as
\begin{align*}
    LHS & \leq \E\int \|g^{1/2}(V-J)\|_2^2 n\gamma_1^2 \|V\|_g^2 ds\\
    & \leq  \E\int \|V-J\|_g^2 n\gamma_1^2 \|V\|_g^2 ds \\
    & \leq n\gamma_1^2 \int \sqrt{\E \|V-J\|_g^4}  \sqrt{\E \|V\|_g^4} dt < \infty.
\end{align*}
For the second local martingale term in~\eqref{eq:p2}:
\begin{align*}
    \E \Big[  \int \langle DJAdB_t, V - J\rangle \Big]
    & = \int (V-J)^T g DJ A^2 (DJ)^T g (V-J)\\
    & \leq \int (V-J)^T g^{1/2} (g^{1/2} DJ g^{-1/2})(g^{1/2} DJ g^{-1/2})^T g^{1/2} (V-J).
\end{align*}
But note that from Lemma~\ref{lemma:derivativebound}, we know \begin{align*}
    \|(g^{1/2} DJ g^{-1/2})(g^{1/2} DJ g^{-1/2})^T\|_{op} \leq \|g^{1/2} DJ g^{-1/2}\|_{op}^2 \leq (L_2 + \gamma_1\|J\|)^2,
\end{align*}
which implies
\begin{align*}
    \E \Big[  \int \langle DJAdB_t, V - J\rangle \Big] & \leq 
    \E \int (V-J)^T g^{1/2}g^{1/2} (V-J)(L_2 + \gamma_1\|J\|)^2\\
    & \leq \E \int \|V-J\|_g^2 (L_2 + \gamma_1\|J\|)^2 dt \\
    & \leq \int \sqrt{\E \|V-J\|_g^4} \sqrt{\E (L_2 + \gamma_1\|J\|)^4} dt\\
    & \lesssim \int \sqrt{\E \|V\|_g^4 + \|J\|_g^4} \sqrt{\E (L_2^4 + \gamma_1^4\|J\|^4)} dt < \infty.
\end{align*}
As a result, the local martingale parts are indeed martingales and their expectation is zero. 
Next, we handle the finite variation part. 
For the first term, we use the result we already proved in~\eqref{eq:dzvbound}:
\begin{align*}
    \langle V - J , DZ(V)\rangle \leq
    \|V - J\|\|DZ(V)\| \leq \|V - J\|_g\|V\|_g n^{3/2}(\gamma_1^2 + \gamma_2^2 + \gamma_3^2).\numberthis\label{eq:term1}
\end{align*}

We handle the second term $2\langle DJ Zdt, V - J \rangle$ combined with another term later on. For the last $dt$ term in~\eqref{eq:p2}, using Lemma~\ref{lem:DtwoJlemma2} (the $i$ index below enumerate the entries of the vector in the dot product):
\begin{align}
    \langle V - J, \sum_{\ell,r}\partial_{\ell, r} J_i g^{\ell, r}\rangle dt & \leq \|V - J\|\Big\|\langle D^2J, g^{-1}\rangle\Big\| \leq n\|V - J\|( \gamma_1L_2 + (\gamma_1^2 + \gamma_3^2) \|J\| + L_3). \label{eq:avali}
\end{align}

\subsubsection{Second term}
Next, we expand the second term in~\eqref{eq:mequation}:
\begin{align*}
   \sum_{i,j,p} (V_i - J_i) \partial_p g_{ij} (V_j - J_j) dX_p
   & = (V_i - J_i) \partial_p g_{ij} (V_j + J_j) A_{p\ell}dB_\ell \\
   & + (V_i - J_i) \partial_p g_{ij} V_j Z_p dt
    - (V_i - J_i) \partial_p g_{ij} J_j Z_p dt.
   \numberthis\label{eq:openedsec} 
\end{align*}
First, for the local martingale part, from~\eqref{eq:V4bound} and~\eqref{eq:jfour}, we have (the bracket indicates quadratic variation)
\begin{align*}
    \E \Big[ \int (V_i - J_i) \partial_p g_{ij} V_j A_{p\ell}dB_\ell \Big]
    & = \E \int_{s=0}^t (V-J)^TDg(V + J)A^2 Dg(V + J)(V-J) ds\\
    & \leq \E \int_{s=0}^t \|V + J\|_g^2 (V-J)^T g (V-J) ds\\
    & \leq \E \int (2\|V\|_g^2 + 2\|J\|_g^2)^2 ds < \infty.
\end{align*}
Therefore, it is indeed a martingale, so taking expectation it becomes zero.
Next, for the second term in~\eqref{eq:openedsec}, using Equation~\eqref{eq:firsttermm}:
\begin{align*}
   \sum_{i,j,p} (V_i - J_i) \partial_p g_{ij} V_j Z_p = (V - J)^TDg(V)Z \leq \|V-J\|_g \Big\|g^{-1} Dg(V) Z\Big\|_g \leq \|V-J\|_g\gamma_1^2 n\sqrt n\|V\|_g.\numberthis\label{eq:term3}
\end{align*}

Moreover, combining the remaining $dt$ term in~\eqref{eq:p2}, namely $2\langle DJ Zdt, V - J \rangle$, with the third term in~\eqref{eq:openedsec}, we get
\begin{align*}
    -\langle V - J, 2DJZ + 2J_j\Gamma_{pj}^kdZ_p\rangle 
    & \leq 2\|V-J\|_g\|DJZ + J_j\Gamma_{pj}^kdZ_p\|_g \\
    & \leq 2L_2\|V-J\|_g\|Z\|_g \leq 2L_2\|V-J\|_gn\sqrt n \gamma_1,\numberthis\label{eq:term4}
\end{align*}
where we used~\eqref{eq:Zbound}.

\subsubsection{Third term}
Next, for the third term in~\eqref{eq:mequation}, using~\eqref{eq:thirdtermm}
\begin{align*}
    \sum_{i,j,p,r}\frac{1}{2}(V_i - J_i) (V_j - J_j)\partial_{p,r}g_{ij} d\langle X_p, X_r\rangle & = \frac{1}{2}(V - J)^TD^2g(g^{-1})(V-J) \\
    & = \frac{1}{2}\sum_i (V - J)^TD^2g(u_i,u_i)(V-J) \\
    & \leq  \frac{1}{2}\gamma_3^2 (\sum_i\|u_i\|_g^2) \|V-J\|_g^2  = \frac{1}{2} n\gamma_3^2 \|V-J\|_g^2,\numberthis\label{eq:term5}
\end{align*}
where above $g^{-1} = \sum_i u_i u_i^T$ is a Choleskey factorization for $g^{-1}$. 

\subsubsection{Fourth Term}
Next, for the first term in~\eqref{eq:fourthtermm} which in turn comes from the fourth term in~\eqref{eq:mequation}, 
\begin{align*}
\sum_{i,j,p} \partial_p g_{ij}(V_j - J_j) (Dg^{-1}(V))_{ip} & =    
(V-J)^T \langle Dg, Dg^{-1}(V)\rangle\\
&\leq \|V-J\|_g \|g^{-1}\langle Dg, Dg^{-1}(V)\rangle\| \\
& \leq \gamma_1^2n\sqrt n \|V\|_g \|V-J\|_g.\numberthis\label{eq:combined1} 
\end{align*}
For the second term in~\eqref{eq:fourthtermm}:
\begin{align}
    \sum_{i,j,p,r}-2 \partial_p g_{ij}(V_j - J_j) \partial_r J_i g^{rp} = -2(V-J)^T \langle Dg, DJg^{-1} \rangle \leq 2\|V-J\|_g \|g^{-1}\langle Dg, DJg^{-1} \rangle\|_g\label{eq:plugback}
\end{align}

Applying Lemma~\ref{lem:repeatedterm} into Equation~\eqref{eq:plugback}:
\begin{align}
    \sum_{i,j,p,r} -2 \partial_p g_{ij}(V_j - J_j) \partial_r J_i g^{rp} \lesssim \gamma_1 n\sqrt n L_2\|V-J\| + \gamma_1^2 n \|J\|\|V - J\|.\label{eq:combined2}
\end{align}
Combining Equations~\eqref{eq:combined1} and~\eqref{eq:combined2}:
\begin{align}
    \sum_{i,j,p} \partial_p g_{ij} (V_j - J_j) d\langle V_i - J_i, X_p\rangle \lesssim \gamma_1^2n\sqrt n \|V\|_g \|V-J\|_g + \gamma_1 n\sqrt n L_2\|V-J\| + \gamma_1^2 n \|J\|\|V - J\|.\numberthis\label{eq:term6}
\end{align}

\subsubsection{Fifth Term}
For the last term in~\eqref{eq:mequation}, we bound the four terms in~\eqref{eq:extraterm}. The first term is directly bounded from Lemma~\ref{lem:critical} by $\gamma_1^4\|V\|_g^2 n$. The second term can be bounded by combining Lemma~\ref{lemma:derivativebound} and Lemma~\ref{lem:critical}:
\begin{align}
    \frac{1}{2} Tr(gDA(V)A(DJ)^T) & =\frac{1}{2}Tr(g^{-1/2}{DJ}^T g^{1/2} g^{1/2} DA(V)) \nonumber\\
    & \leq \|g^{-1/2}{DJ}^T g^{1/2}\|_{op} \|g^{1/2} DA(V)\|_1 \\
    & \leq (L_2 + \gamma_1\|J\|_g)\gamma_1\|V\|_g n.\label{eq:term7}
\end{align}
where above, $\|.\|_1$ is the matrix nuclear, which is just the sum of the singular values.
The third term is similar to the second term. The forth term can be bounded using Lemma~\ref{lem:repeatedterm2}:
\begin{align*}
    \frac{1}{2}\langle g, DJg^{-1}(DJ)^T\rangle 
     \leq n(L_2 + \gamma_1\|J\|)^2.\label{eq:term8} 
\end{align*}

\subsubsection{Getting to it}
Therefore, overall, Combining Equations~\eqref{eq:term1},\eqref{eq:avali},\eqref{eq:term3},\eqref{eq:term4},\eqref{eq:term5},~\eqref{eq:term6},~\eqref{eq:term7},~\eqref{eq:term8}, and taking expectation:
\begin{align*}
    d\E\|V - J\|_g^2 & \leq
   \E \|V - J\|_g\|V\|_g n^{3/2}(\gamma_1^2 + \gamma_2^2 + \gamma_3^2) 
   + n\|V - J\|( \gamma_1L_2 + (\gamma_1^2 + \gamma_3^2) \|J\| + L_3)dt\\
   & + \E \|V-J\|_g\|V\|_g \gamma_1^2 n\sqrt n 
   + 2L_2\|V-J\|_gn\sqrt n \gamma_1 dt\\
   & + \E \frac{1}{2} n\gamma_3^2 \|V-J\|_g^2 dt\\
   & + \E \gamma_1^2n\sqrt n \|V\|_g \|V-J\|_g + \gamma_1 n\sqrt n L_2\|V-J\| + \gamma_1^2 n \|J\|\|V - J\| dt\\
   & 
   + \E \gamma_1^2\|V\|_g^2 n
   + 2(L_2 + \gamma_1\|J\|_g)\gamma_1\|V\|_g n
   +  n(L_2 + \gamma_1\|J\|)^2 dt.
\end{align*}
Further applying Cauchy Swartz:
\begin{align*}
    d\E\|V - J\|_g^2 & \lesssim
   \sqrt{\E \|V - J\|_g^2} \sqrt{\E\|V\|_g^2} (n\sqrt n (\gamma_1^2 + \gamma_2^2 + \gamma_3^2) + \gamma_1^2 n\sqrt n)dt\\
   & + \sqrt{\E \|V - J\|^2}( n\gamma_1L_2 + n(\gamma_1^2 + \gamma_3^2) \sqrt{\E\|J\|^2} + nL_3 + 2L_2n\sqrt n \gamma_1) dt\\
   & + \frac{1}{2} n\gamma_3^2 \E \|V-J\|_g^2 dt\\
   & + \E \gamma_1^2\|V\|_g^2 n
   + 2 L_2\gamma_1 n \sqrt{\E \|V\|_g^2} dt \\
   & + n \gamma_1^2\sqrt{\E \|J\|_g^2}\sqrt{\E\|V\|_g^2} 
   + \gamma_1^2 \E\|J\|^2 + 2nL_2 dt\\
   & \lesssim (\kappa_0'' + \kappa_1'' \sqrt{\E \|V - J\|_g^2} + \kappa_2'' \E \|V - J\|_g^2) dt,
\end{align*}
for
\begin{align*}
    \kappa_0'' & = 2n\gamma_1^2 \|\nabla F(x_0)\|_g^2 + 4L_2\gamma_1 n \|\nabla F(x_0)\|
    \\
    & + 2 n \gamma_1^2
    \|J_0\|\|\nabla F(x_0)\| \\
    & + 2n \gamma_1^2 \sqrt{2(\beta_0 + \beta_1)t} 
    \|\nabla F(x_0)\|\\
    & + \gamma_1^2\|J_0\|^2 + 2\gamma_1^2(\beta_0 + \beta_1)t + 2nL_2,\\
    \\
    \kappa_1'' & = 2\|\nabla F(x_0)\|(n^{3/2} (\gamma_1^2 + \gamma_2^2 + \gamma_3^2))\\
    & + n\gamma_1 L_2 + nL_3 + 3L_2n\sqrt n \gamma_1 + n(\gamma_1^2 + \gamma_3^2)\|J_0\| + n(\gamma_1^2 + \gamma_3^2)\sqrt{2(\beta_0 + \beta_1)t},\\
    \\
     \kappa_2'' & = \frac{1}{2}n\gamma_3^2.
\end{align*}
Further applying AM-GM, this implies
\begin{align*}
    d\E\|V - J\|_g^2 \leq (\kappa_0' + \kappa_2' \E \|V-J\|_g^2) dt,
\end{align*}
For
\begin{align*}
    \kappa_0' & = 2n\gamma_1^2 \|\nabla F(x_0)\|_g^2 + 4L_2\gamma_1 n \|\nabla F(x_0)\|
    \\
    & + 2 n \gamma_1^2
    \|J_0\|\|\nabla F(x_0)\| \\
    & + 2 n \gamma_1^2 \sqrt{2(\beta_0 + \beta_1)t} 
    \|\nabla F(x_0)\|\\
    & 
    + 2\gamma_1^2(\beta_0 + \beta_1)t + 2nL_2 + \\
    & + 2\|\nabla F(x_0)\|^2 (n^{3/2} (\gamma_1^2 + \gamma_2^2 + \gamma_3^2))\\
    & +nL_3 + 3L_2n\sqrt n \gamma_1 + \sqrt n(\gamma_1^2 + \gamma_3^2)\|J_0\|^2 + n(\gamma_1^2 + \gamma_3^2)\sqrt{2(\beta_0 + \beta_1)t},
    \\
    \kappa_2' & = \frac{1}{2}n\gamma_3^2 + ((n^{3/2} (\gamma_1^2 + \gamma_2^2 + \gamma_3^2))) + n\sqrt n(\gamma_1^2 + \gamma_3^2)\\
    & +  n\gamma_1 L_2 + nL_3 + 3L_2n\sqrt n \gamma_1 + n(\gamma_1^2 + \gamma_3^2)\sqrt{2(\beta_0 + \beta_1)t}, 
\end{align*}
where for times $t \leq \epsilon$ ($\epsilon$ is the final step size), we have $\kappa_2' \lesssim \kappa_2$ for 
\begin{align}
    \kappa_2 = 3n\sqrt n(\gamma_1^2 + \gamma_2^2 + \gamma_3^2) + nL_3 + 4L_2n\sqrt n \gamma_1 + n(\gamma_1^2 + \gamma_3^2)\sqrt{2(\beta_0 + \beta_1)\epsilon}.
\end{align}
Now using the similar ODE trick as in~\eqref{eq:ODEbound}, assuming the stronger condition below compared to~\eqref{eq:criticcondition} that we imposed:
\begin{align}
    t \leq \frac{1}{2\kappa_2},\label{eq:secondcondition}
\end{align}
Note that this condition also implies~\eqref{eq:condition1}. We get
\begin{align}
    \E \|V - J\|^2 \leq (1 + t\kappa_2)\|V_0 - J_0\|^2 + t\kappa_0' \lesssim \|V_0 - J_0\|^2 + t\kappa_0' .\label{eq:tempbound}
\end{align}
On the other hand, combining $\|b(t,x_0)\| = \|\nabla F(x_0)\|$ with Lemma~\ref{lem:paralleldiffpart}, we have $\|J_0\| \leq \|\nabla F(x_0)\| + tL_2$. Combining this with the AM-GM inequality
$$2n \gamma_1^2 \sqrt{2(\beta_0 + \beta_1)t}\|\nabla F(x_0)\| \leq n\sqrt n \gamma_1^2 \|\nabla F(x_0)\|^2 + \sqrt n\gamma_1^2 (\beta_0 + \beta_1)t,$$
$$
2 n \gamma_1^2 \|J_0\|\|\nabla F(x_0)\| \leq \sqrt n \gamma_1^2 \|J_0\|^2 + n\sqrt n \gamma_1^2 \|\nabla F(x_0)\|^2,
$$
implies $\kappa_0' \lesssim \kappa_0$ for
\begin{align*}
    \kappa_0 = & (n^{3/2} (\gamma_1^2 + \gamma_2^2 + \gamma_3^2)) \|\nabla F(x_0)\|_g^2 + 4L_2\gamma_1 n \|\nabla F(x_0)\|\\
    & 
    + 2\sqrt n\gamma_1^2(\beta_0 + \beta_1)t + 2nL_2\\
    & +nL_3 + 3L_2n\sqrt n \gamma_1 + \sqrt n(\gamma_1^2 + \gamma_3^2)t^2L_2^2+ n(\gamma_1^2 + \gamma_3^2)\sqrt{2(\beta_0 + \beta_1)t}
\end{align*}

Now we need to bound $\|V_0 - J_0\|$. Note that $V_0 = b(t, x_0)$ and $J_0 = \nabla F(\gamma_t(x_0))$ by definition of $J$. Hence, using Lemma~\ref{lem:paralleldiffpart}, we have
\begin{align}
    \E \|V_t - J_t\|^2 \leq t^2L_2^2 + t\kappa_0.\label{eq:tempboundzero}
\end{align}

Note that the goal is to take expectation with respect $x_0$ and plugging this into~\eqref{eq:initialderivation} to get:
\begin{align}
    \DE \leq \frac{1}{4}I_v(\rho_t) + t^2L_2^2 + 4t\E\kappa_0.
\end{align}
Hence, we now compute $\E_{\rho_0} \kappa_0$. Applying Cauchy swartz as $\E \|\nabla F(x_0)\| \leq \sqrt{\E \|\nabla F(x_0)\|^2}$:
\begin{align*}
   \E \kappa_0 \leq & 2(n^{3/2} (\gamma_1^2 + \gamma_2^2 + \gamma_3^2))  \E \|\nabla F(x_0)\|_g^2 \\
   & + 4L_2\gamma_1 n \sqrt{\E\|\nabla F(x_0)\|^2}\\
    & 
    + 2\sqrt n\gamma_1^2(\beta_0 + \beta_1)t + 2nL_2 +\\
    & +nL_3 + 3L_2n\sqrt n \gamma_1 + \sqrt n (\gamma_1^2 +  \gamma_3^2)t^2L_2^2+ n(\gamma_1^2 + \gamma_3^2)\sqrt{2(\beta_0 + \beta_1)t}.
\end{align*}
To further simplify the above, We apply the inequality $$4L_2\gamma_1 n \sqrt{\E\|\nabla F(x_0)\|^2} \leq 8\gamma_1^2 n\sqrt n \E\|\nabla F(x_0)\|^2 + 8\sqrt nL_2^2,$$
and combine it with $\sqrt n(\gamma_1^2 + \gamma_3^2)t^2L_2^2 \leq \frac{1}{n} tL_2^2$ from~\eqref{eq:criticcondition} to obtain
\begin{align}
    \E \kappa_0 \lesssim & n^{3/2}(\gamma_1^2 + \gamma_2^2 + \gamma_3^2)  \E \|\nabla F(x_0)\|_g^2\nonumber\\
    & 
    + 2\sqrt n\gamma_1^2(\beta_0 + \beta_1)t + 2nL_2 + (\sqrt n + \frac{t}{n})L_2^2 \nonumber\\
    & +nL_3 + 3L_2n\sqrt n \gamma_1 + n(\gamma_1^2 + \gamma_3^2)\sqrt{2(\beta_0 + \beta_1)t}.\label{eq:kappazero}
\end{align}
On the other hand, note that from Lemma~\ref{lem:helper1}, we know the chi squared distance $\chi^2(\rho_0) < \infty$ at any step of the algorithm. This implies that the KL divergence is also bounded using Pinsker inequality. 

Hence, using Lemma~\ref{llem:globalgradientbound2}, we get:
\begin{align}
    \E \|\nabla F(x_0)\|^2 \leq 2nL_2 + 2L_2^2 H_\nu(\rho_0).\label{eq:helper2}
\end{align}
Plugging Equation
~\eqref{eq:helper2} into~\eqref{eq:kappazero}:
\begin{align*}
   \E \kappa_0 & \leq (n^{3/2} (\gamma_1^2 + \gamma_2^2 + \gamma_3^2))  H_\nu(\rho_0) \\
   & + (n^{3/2} (\gamma_1^2 + \gamma_2^2 + \gamma_3^2))nL_2\\
    &
    + 2\sqrt n\gamma_1^2(\beta_0 + \beta_1)t + 2nL_2 + (\sqrt n + \frac{t}{ n})L_2^2\\
    & +nL_3 + 3L_2n\sqrt n \gamma_1 + n(\gamma_1^2 + \gamma_3^2)\sqrt{2(\beta_0 + \beta_1)t}.
\end{align*}

Further applying the time restriction $t \leq \epsilon$, we define 
\begin{align}
    \bar \omega_0 & = \bar c_1(n^{3/2} (\gamma_1^2 + \gamma_2^2 + \gamma_3^2))\epsilon,\\
    \bar \omega & = \bar c_2\Big[ (n^{3/2} (\gamma_1^2 + \gamma_2^2 + \gamma_3^2))nL_2 \epsilon\\
    &
    + 2\sqrt n\gamma_1^2(\beta_0 + \beta_1)\epsilon^2 + 2nL_2\epsilon + (\sqrt n + \frac{\epsilon}{n})L_2^2\epsilon\\
    & +nL_3\epsilon + 3L_2n\sqrt n \gamma_1\epsilon + n(\gamma_1^2 + \gamma_3^2)\sqrt{2(\beta_0 + \beta_1)}\epsilon^{3/2} \Big],
\end{align}
we get for every $t \leq \epsilon$:
\begin{align*}
    \E t\kappa_0 \leq \bar \omega + \bar \omega_0 H_\nu(\rho_0).\numberthis
\end{align*}
Now simplifying these equations using the conditions $\epsilon \leq 1/(2\kappa_2), 1/\xi$, on can check the validity of upper bounds $\bar \omega_0 \leq \omega_0, \bar \omega \leq \omega$ for
\begin{align}
    \omega_0 & = c_1'(n^{3/2} (\gamma_1^2 + \gamma_2^2 + \gamma_3^2))\epsilon,\\
    \omega & = c_2'\Big[ (n^{3/2} (\gamma_1^2 + \gamma_2^2 + \gamma_3^2))nL_2
     + \sqrt nL_2^2 + nL_3 \Big]\epsilon,
\end{align}
then, for every $t \leq \epsilon$:
\begin{align*}
    \E t\kappa_0 \leq \omega + \omega_0 H_\nu(\rho_0).\numberthis\label{eq:before}
\end{align*}
Combining Equations~\eqref{eq:before} and~\eqref{eq:tempboundzero} implies
\begin{align*}
    \E_{\rho_{t|0}} \|V_t - J_t\|^2 \leq \epsilon^2L_2^2 + \omega + \omega_0 H_\nu(\rho_0) \lesssim \omega + \omega_0H_\nu(\rho_0).
\end{align*}

This completes the proof of Lemma~\ref{lem:VJbound} and hence Lemma~\ref{lem:main}.

\section{Rate of convergence}\label{sec:rate}

For convenience we assume $\alpha, \epsilon \le 1, \gamma_i, L_i \ge 1$. 

Combining Equations~\eqref{eq:initialderivation} and~\eqref{eq:klderivative} to the result of Lemma~\eqref{lem:VJbound}
~\eqref{eq:klderivative} then gives
\begin{align}
    \partial_t H_\nu(\rho_t) \leq -\frac{3}{4}I_\nu(\rho_t) + \omega + \omega_0 H_\nu(\rho_0).\label{eq:before2}
\end{align}
Applying the log sobolev inequality to~\eqref{eq:before2} implies for every $t \leq \epsilon$:
\begin{align*}
    \partial_t H_{\nu}(\rho_t) \leq -\frac{3}{4} \alpha H_{\nu}(\rho_t) +\omega + \omega_0 H_\nu(\rho_0).
\end{align*}

Now defining $\tilde H(t) = H_\nu(\rho_t) - \frac{4}{3\alpha}(\omega + \omega_0 H_\nu(\rho_0))$, this implies
\begin{align*}
    \partial_t \tilde H(t) \lesssim -\frac{3}{4}\alpha \tilde H(t),
\end{align*}
or
\begin{align*}
    \tilde H(t) \leq e^{-\frac{3}{4}\alpha t} \tilde H(0).
\end{align*}
Therefore
\begin{align*}
    H_\nu(\rho_t) & \leq e^{-\frac{3}{4}\alpha t} H_\nu(\rho_0) + (1 - e^{-\frac{3}{4}\alpha t})\frac{4}{3\alpha} (\omega + \omega_0 H_\nu(\rho_0)) \\
    & \leq e^{-\frac{3}{4}\alpha t} H_\nu(\rho_0) + t(\omega + \omega_0 H_\nu(\rho_0)).
\end{align*}
Now setting $t = \epsilon$:
\begin{align*}
    H_\nu(\rho_\epsilon) \leq e^{-\frac{3}{4}\alpha \epsilon} H_\nu(\rho_0) + \epsilon(\omega + \omega_0 H_\nu(\rho_0)).
\end{align*}
Taking $\epsilon$ small enough such that 
\begin{align*}
    \epsilon\omega_0 \leq \frac{3}{16}\alpha \epsilon
\end{align*}
i.e., 
\begin{align}
    \epsilon \lesssim \frac{\alpha}{n^{3/2}(\gamma_1^2 + \gamma_2^2 + \gamma_3^2)} \label{eq:w1condition}
\end{align}

we show
\begin{align}
    H_\nu(\rho_\epsilon) \leq e^{-\frac{3}{16}\alpha\epsilon} H_\nu(\rho_0) + \epsilon\omega.\label{eq:minterm}
\end{align}
From the assumption $\alpha, \epsilon \leq 1$, we have $3\alpha \epsilon/4 \leq 3/4$, then 
\begin{align*}
    H_\nu(\rho_\epsilon) & \leq (1 - \frac{3}{8}\alpha \epsilon)H_\nu(\rho_0) + \epsilon(\omega + \omega_0 H_\nu(\rho_0)) \leq (1 - \frac{3}{16}\alpha \epsilon)H_\nu(\rho_0) + \epsilon\omega \\
    & \leq e^{-\frac{3}{16}\alpha \epsilon}H_\nu(\rho_0) + \epsilon\omega.
\end{align*}

\begin{proof}
[Proof of Theorem~\ref{thm:generalmanifold}.]
Using Lemma~\ref{lem:main}, 
for one step of RLA with time parameter $\epsilon$, we have 
\[ 
H_\nu(\rho_{k+1}) \leq e^{-\frac{3}{16}\alpha \epsilon} H_\nu(\rho_k) + c_1'\omega\epsilon,
\]

Now opening up the recursion plus applying the inequality $e^{-c} \leq 1 - c/2$ for $0\leq c\leq 1/2$:
\begin{align*}
 H_\nu(\rho_k) &\lesssim e^{-\frac{3}{16}\alpha \epsilon k} H_\nu(\rho_0) + \omega\epsilon/(\epsilon\alpha)\\ 
 &\lesssim e^{-\frac{3}{16}\alpha \epsilon k} H_\nu(\rho_0) + \omega/\alpha.
\end{align*}

The bound on $k$ follows. 
Note that in order to use Lemma~\ref{lem:sectionalcurvature} in Lemma~\ref{lem:onestepboundedness}, the condition $t \leq \frac{1}{6\sqrt K + 8\sqrt L_2}$ translates into $t \lesssim \frac{1}{\sqrt{\gamma_1^2 + \gamma_3^2} + \sqrt L_2}$. Moreover, condition $\epsilon \leq 1/(2\kappa_2)$ in~\eqref{eq:secondcondition} boils down to
\begin{align}
    &\epsilon \lesssim \frac{1}{n\sqrt n (\gamma_1^2 + \gamma_2^2 + \gamma_3^2) + nL_3 + L_2n\sqrt n\gamma_1},\\
    &\epsilon \lesssim n^{-2/3}(\gamma_1^2 + \gamma_3^2)^{-2/3}(\beta_0 + \beta_1)^{-1/3}.\label{eq:conditions}
\end{align}
Combining the two implies
\begin{align}
    \epsilon \leq \frac{1}{n^{3/2}(\gamma_1^2 + \gamma_2^2 + \gamma_3^2) + nL_3 + L_2n\sqrt n\gamma_1 + L_2^2}.\label{eq:mainconstraint}
\end{align}
Note that picking $\epsilon$ as small as~\eqref{eq:mainconstraint} already covers the conditions in~\eqref{eq:criticcondition}, $\epsilon \leq 1/\xi$ in Lemma~\ref{lem:Jbound}, and finally $t \leq \frac{1}{6\sqrt K + 8\sqrt L_2}$. Hence, the proof is complete.
\end{proof}

\subsection{Additional Lemmas}
\begin{lemma}\label{lem:paralleldiffpart}
Given a point $x_0$ and its mapped version $\gamma_t(x_0)$ on a geodesic starting with initial speed $\nabla F(x_0)$, we have
\begin{align*}
    \|b(t, x_0) - \nabla F(\gamma_t(x_0))\| \leq tL_2,
\end{align*}
where recall that $b(t, x_0)$ is the parallel transport of $\nabla F(x_0)$ along the geodesic up to time $t$. 
\end{lemma}
\begin{proof}
We can write using fundamental theorem of calculus
\begin{align*}
    \|b(t,x_0) - \nabla F(\gamma_t(x_0))\|^2
    & = \int_{s=0}^t \partial_s \| \nabla F(\gamma_s(x_0)) - b(s,x_0)\|^2 ds\\
    & \leq \int_{s=0}^t \partial_s \| \nabla F(\gamma_s(x_0)) - b(s,x_0)\|^2 ds\\
    & = \int_{s=0}^t  \Big\langle \nabla F(\gamma_s(x_0)) - b(s,x_0), \nabla_{\gamma_s(x_0)} F(\gamma_s(x_0)) - \nabla_{\gamma_s(x_0)} b(s,x_0) \Big\rangle ds\\
    & \leq \int_{s=0}^t \|\nabla F(\gamma_s(x_0)) - b(s,x_0)\|\|\nabla_{\gamma_s(x_0)} F(\gamma_s(x_0))\| ds\\
    & = \int_{s=0}^t \|\nabla F(\gamma_s(x_0)) - b(s,x_0)\|\|\text{Hess}F(\gamma_s(x_0) , \gamma_s(x_0))\|ds\\
    & \leq L_2 \int \|\nabla F(\gamma_s(x_0)) - b(s,x_0)\|ds.\numberthis\label{eq:initialcond}
\end{align*}
Taking derivatives from both sides implies
\begin{align*}
    \frac{d}{ds} \|\nabla F(\gamma_s(x_0)) - b(s,x_0)\| \leq L_2,
\end{align*}
which implies
\begin{align}
    \|\nabla F(\gamma_t(x_0)) - b(t,x_0)\| \leq tL_2.\label{eq:initialval} 
\end{align}
\end{proof}

\begin{lemma}\label{lem:globalgradientbound1}
For an arbitrary distribution with density $\rho$ over the manifold, we have
\begin{align}
    \E_{y \sim \rho}\|\nabla F(y)\|^2 \leq 2nL_2 + 2L_2^2W^2(\rho, \nu).
\end{align}
\end{lemma}
\begin{proof}
Let $\text{cop}$ be the optimal $W^2$ coupling between the distributions $\rho$ and $\nu$. Then, 
\begin{align*}
    \E_{y \sim \rho}\|\nabla F(y)\|^2 \leq 2\E_{y \sim \nu}|\|\nabla F(y)\|^2  + 2\E_{(y_1, y_2) \sim \text{cop}} \|\nabla F(y_1) - \nabla F(y_2)\|^2.
\end{align*}
Using the same derivation as in~\eqref{eq:initialcond}, using the gradient smoothness of $F$:
\begin{align}
   \E_{y \sim \rho}\|\nabla F(y)\|^2 \leq  2\E_{y \sim \nu}|\|\nabla F(y)\|^2 + 2\E L_2^2 d(y_1, y_2)^2.
\end{align}
On the other hand, for the first term, using integration by parts over the manifold, it is easy to see that
\begin{align}
    \E_{y \sim \nu}\|\nabla F(y)\|^2 \leq \E_{y \sim \nu} \Delta F(y) = \E_{y \sim \nu} Tr( \text{Hess}(F)(y)) \leq nL_2,
\end{align}
which completes the proof.
\end{proof}

\begin{lemma}\label{llem:globalgradientbound2}
 For an arbitrary distribution with density $\rho$ with bounded second moment over the manifold, we have
 \begin{align}
     \E_{y \sim \rho}\|\nabla F(y)\|^2 \leq 2nL_2 + 2L_2^2 H_\nu(\rho).
 \end{align}
\end{lemma}
\begin{proof}
Directly from Lemma~\ref{lem:globalgradientbound1} and the result of~\cite{otto2000generalization} for a Talagrand inequality on a general manifold.
\end{proof}

\subsection{Square root matrix Regularity: proofs of Lemmas~\ref{lem:critical} and~\ref{lem:secondvon}}\label{app:criticlem}
Here we prove Lemma~\ref{lem:critical}. For the convenience of the reader, we restate the Lemma below.
\begin{lemma}
Suppose the metric $g$ is $\gamma_1$ normal self-concordant. Then, for the square root matrix $A(x) = \sqrt{2g^{-1}}$, we have
\begin{align*}
    &\langle g, (DA(V))^2\rangle \leq \frac{1}{4}n\gamma_1^2\|V\|_g^2,\\
    &\langle g^{1/2}, DA(V)\rangle \leq \frac{1}{2}n\gamma_1\|V\|_g.
\end{align*}
\end{lemma}
\begin{proof}
A general trick that we use here is that we reduce the problem to the case where $A$ is diagonal at the fixed point $x$ for which we want to prove the inequality. We start with the first term. To show this reduction, let $g(x) = U \Gamma U^T$ be the eigendecomposition of the metric $g$ at point $x$. Then, we can write
\begin{align*}
    \langle g, (DA(V))^2\rangle & =  Tr(g (DA(V))^2)\\
    & = Tr(g DA(V)DA(V))\\
    & = Tr(U^Tg UU^TDA(V)UU^TDA(V)U)\\
    & = Tr(\Lambda (U^TDA(V)U)^2).
\end{align*}
But note that because $U$ is a fixed matrix, it can go inside the differentiation:
\begin{align*}
   \langle g, (DA(V))^2\rangle & =  Tr(\Lambda (D(U^TAU)(V))^2).
\end{align*}
On the other hand, note that 
$$A = \sqrt{g} = U \sqrt \Lambda U^T,$$
which means
$$U^T A U = \sqrt \Lambda.$$
Hence, if we define the metric $g'(y) = U^T g(y) U$ for every $y$ in the domain of our Euclidean chart (note that $U$ is fixed to an orthonormal eigendecomposition of $g$ at point $x$), then $g'$ is diagonal and equal to $\Lambda$ at $x$, and further
\begin{align*}
    \langle g, (DA(V))^2\rangle = Tr(\Lambda (Dg'(V))^2).
\end{align*}
The above equality means it is enough to bound the RHS instead of LHS, which means without loss of generality, we can assume the metric $g$ is some diagonal matrix. For ease of notation later on, let $\Lambda = g^{-1}(x)$ be the diagonal inverse matrix of $g$ at $x$. This way, $A = \sqrt \Lambda$. The point of having diagonal $g$ at $x$ is that the closed form of $DA(V)$ at $x$ now takes a simpler form. To compute it, it is enough to take derivative from the sides of
\begin{align*}
    A^2 = g^{-1},
\end{align*}
which implies
\begin{align*}
    DA(V)A + ADA(V) = Dg^{-1}(V) = H,
\end{align*}
where $H$ need not to be diagonal. Solving the above for $A$ gives for every $1 \leq i,j \leq n$:
\begin{align*}
    DA(V)_{ij} = H_{ij}/(\sqrt \Lambda_i + \sqrt \Lambda_j).
\end{align*}
Now note that the amtrix $\Big((\sqrt \Lambda_i + \sqrt \Lambda_j)^{-1}\Big)_{ij}$ is PSD. This is because it can be written as the integral over a class of PSD matrices as
\begin{align*}
    \Big((\sqrt \Lambda_i + \sqrt \Lambda_j)^{-1}\Big)_{ij} & = \int_{t = 0}^\infty \Big(e^{-(\sqrt \Lambda_i + \sqrt \Lambda_j)t}\Big)_{ij} dt \\
    & \int_{t = 0}^\infty \text{vec}(e^{-t \sqrt \Lambda_i}) \text{vec}(e^{-t\sqrt \Lambda_i})^T dt \succeq 0.
\end{align*}
On the other hand, using the self concordance of $g$, we can write
\begin{align*}
    H = Dg^{-1}(V) = -g^{-1}Dg(V)g^{-1} \preceq \|V\|_g g^{-1} g g^{-1} \leq \gamma_1\|V\|_g g^{-1} \leq \|V\|_g \Lambda,
\end{align*}
and similarly
\begin{align*}
    H \succeq -\gamma_1\|V\|_g \Lambda.
\end{align*}
Hence, applying the Schur product theorem, we see that
\begin{align*}
    DA(V) = \Big((\sqrt \Lambda_i + \sqrt \Lambda_j)^{-1}\Big)_{ij} \odot H \preceq  \|V\|_g  \Big((\sqrt \Lambda_i + \sqrt \Lambda_j)^{-1}\Big)_{ij} \odot \Lambda
    = \frac{1}{2}\gamma_1\|V\|_g\sqrt{\Lambda},
\end{align*}
and similarly
\begin{align*}
     DA(V) \succeq -\frac{1}{2} \gamma_1\|V\|_g \sqrt{\Lambda}.
\end{align*}
Now if we denote the $i$th eigenvalue and singular value of a matrix by $\lambda_i$ and $\sigma_i$ respectively, then the above inequalities imply
\begin{align*}
    &\lambda_i(DA(V)) \leq \frac{1}{2}\gamma_1 \|V\|_g\sqrt \Lambda_i,\\
    &\lambda_(DA(V)) \geq -\frac{1}{2}\gamma_1 \|V\|_g\sqrt \Lambda_i,
\end{align*}
which implies for all $1 \leq i\leq n$:
\begin{align}
    \sigma_i(DA(V))^2 \leq \frac{1}{4}\gamma_1^2\|V\|_g^2\Lambda_i.\label{eq:singular1}
\end{align}
Next, we apply the von Neumann's trace inequality, which states that for two PSD matrices $M$ and $N$:
\begin{align*}
    \langle M, N\rangle \leq \sum_i \sigma_i(M)\sigma_i(N).
\end{align*}
Applying this inequality for the matrices $M = g^{-1}$ and $N = (DA(V))^2$ at point $x$,  we get
\begin{align*}
    \langle g^{-1}, (DA(V))^2\rangle \leq \sum_i \sigma_i(\Lambda^{-1}) \sigma_i((DA(V))^2) \leq \sum \Lambda_i^{-1} \frac{1}{4}\gamma_1^2\Lambda_i \|V\|_g^2 = \frac{n\gamma_1^2}{4}\|V\|_g^2.
\end{align*}
where we used the fact that $g(x) = \Lambda^{-1}$. This completes the proof of the first inequality.

For the second and third inequality, we use a similar trick.
For the third inequality, we apply the von Neumann trace inequality and use~\eqref{eq:singular1}:
\begin{align*}
    &\langle g^{1/2}, DA(V)\rangle \leq \sum_i \sigma_i(g^{1/2}) \sigma_i(DA(V)) \leq \sum_i \Lambda_i^{-1/2} \frac{1}{2}\gamma_1 \Lambda_i^{1/2} = \frac{n\gamma_1}{2}\|V\|_g.
\end{align*}

For the third one, we can use the three matrix version of the von Neumann trace inequality:
\begin{align*}
    Tr(g^{1/2}(DA(V))^2g^{1/2}) \leq \sum_i \sigma_i(g^{1/2})^2 \sigma_i((DA(V))^2) \leq \|V\|_g^2 \sum_i \Lambda_i^{-1/2} \gamma_1^2\Lambda_i^{-1/2} \Lambda_i/4 \leq \frac{n\gamma_1^2}{4}\|V\|_g^2. 
\end{align*}
For the second inequality, we use a similar trick. To prove that it is enough to show the problem for when $g$ is diagonal at $x$, note that
\begin{align*}
    \langle g^{1/2}, DA(V)\rangle 
    & = Tr(g^{1/2}DA(V))\\
    & = Tr(U^Tg^{1/2}UU^TDA(V)U)\\
    & = Tr(\Lambda^{1/2} D(U^TAU)(V))\\
    & = Tr(\Lambda^{1/2} D\sqrt{g'^{-1}}(V)).
\end{align*}
Now assuming that $g$ is diagonal at $x$, we can repeat the arguments above and gain apply the Neumann trace inequality:
\begin{align*}
    Tr(g^{1/2}DA(V)) \leq \sum_i \sigma_i(g^{1/2})\sigma_i(DA(V)) \leq \sum_i  {\Lambda_i}^{-1/2}\|V_g\| \frac{1}{4}\sqrt \Lambda_i \leq \frac{n\gamma_1}{2}\|V\|_g.  
\end{align*}
where we used that 
\begin{align*}
    \sigma_i(DA(V)) \leq \frac{1}{2}\gamma_1\|V\|_g\sqrt \Lambda_i.
\end{align*}
\end{proof}

Next, we restate and prove Lemma~\ref{lem:secondvon}.
\begin{lemma}
Suppose the metric $g$ is $\gamma_1$ normal self-concordant. Then, for the square root matrix $A(x) = \sqrt{2g^{-1}}$, we have
\begin{align*}
    V^Tg(DA(V))^2gV \leq n\gamma_1^2\|V\|_g^4.
\end{align*}
\end{lemma}
\begin{proof}
We use the same approach as in the proof of Lemma~\ref{lem:critical}. First, assuming we want to prove the ineqaulity at point $x$, we show it is sufficient to assume $g$ is diagonal at $x$. using the same notation as the proof of Lemma~\ref{lem:critical}, let $g = U\Lambda U^T$ be the orthonormal decomposition of $g$. Then
\begin{align*}
    V^Tg(DA(V))^2gV  & = V^TU^TgU U^TDA(V)UU^TDA(V)UU^TgUV\\
    & = V^T\Lambda D(U^TAU)(V)D(U^TAU)(V) \Lambda V\\
    & = V^T \Lambda (D\sqrt{g'^{-1}}(V))^2 V,
\end{align*}
where we define for any $y$, $g'(y) = U^T g U$ for the orthonormal matrix fixed for $g$ at point $x$. This shows that we can assume $g$ is diagonal. Next, recall that we showed in the proof of Lemma~\ref{lem:critical}, namely in~\eqref{eq:singular1}, that
\begin{align*}
\sigma_i(DA(V))^2 \leq \frac{1}{4}\gamma_1^2\|V\|_g^2 \Lambda_i.
\end{align*}
Our goal is to apply the three matrix generalization of the von Neumann trace inequality, namely for matrices $M_1, M_2, M_3$, we have
\begin{align*}
    \sum_{i} \sigma_i(M_1M_2M_3) \leq \sum_i \sigma_i(M_1) \sigma_i(M_2)\sigma_i(M_3).
\end{align*}
This known generalization follows from an elegant technique called majorization. 
Now we rewrite the term as
\begin{align*}
     V^Tg(DA(V))^2gV = Tr(VV^T g(DA(V))^2g)\\
     & = Tr(g^{1/2}VV^T g^{1/2}g^{1/2}(DA(V))^2g^{1/2})\\
     & \leq \|g^{1/2}VV^T g^{1/2}\|_{op}\|g^{1/2}(DA(V))^2g^{1/2}\|_{1},
\end{align*}
where $\|\|_1$ is the Nuclear norm of the matrix, namely te sum of singular values. Now we apply the Neumann trace inequality as
\begin{align*}
   \|g^{1/2}(DA(V))^2g^{1/2}\|_{1}
   & = \sum_i \sigma_i(g^{1/2}(DA(V))^2g^{1/2})\\
   & = \sum_i \sigma_i(g^{1/2}) \sigma((DA(V))^2)\sigma_i(g^{1/2})\\
   & = \sum_i \Lambda_i^{-1/2} \frac{1}{4}\gamma_1^2\|V\|_g^2 \Lambda_i \Lambda_i^{-1/2} \\
   & = \frac{n\gamma_1^2}{4} \|V\|_g^2. 
\end{align*}
On the other hand, for the first term:
\begin{align*}
   \|g^{1/2}VV^T g^{1/2}\|_{op} = \|g^{1/2}V\|_2^2 = \|V\|_g^2,
\end{align*}
which completes the proof.
\end{proof}

\subsection{Properties of the Riemann Tensor}
Here, our goal is to show a bound on some terms obtained from the Riemann tensor of our self-concordant manifold $\mathcal M$, including the Ricci tensor and the sectional curvature. The lower bound on the Ricci curvature is required to guarantee that the process does not blow up in finite time. The bound on sectional curvature is required when we wish to apply Lemma~\ref{lem:useofsectional}.

We start by the following lemma which illustrates a bound on the subtensor of the Riemann tensor obtained by setting $\gamma'$ on the second and third slots, for some normal vector $\|\gamma'\| = 1$:

\begin{lemma}\label{lem:rimanntensorbound}
We have
\begin{align*}
    \|R(v, \gamma')\gamma'\|^2 \lesssim (\gamma_1^4 + \gamma_3^4)\|v\|^2\|\gamma'\|^4.
\end{align*}
\begin{proof}
Expressing the Riemman tensor in the following index form:
\begin{align*}
    R^\rho_{\sigma \mu \nu} & = dx^\rho (R(\partial_\mu, \partial_\nu)\partial \sigma),
\end{align*}
we have the following formula: 
\begin{align*}
    R^\rho_{\sigma \mu \nu} = \partial_\mu \Gamma^{\rho}_{\nu \sigma} - \partial_\nu \Gamma^{\rho}_{\mu \sigma} + \Gamma^{\rho}_{\mu \lambda}\Gamma^{\lambda}_{\nu \sigma} - \Gamma^\rho_{\nu \lambda}\Gamma^{\lambda}_{\mu \sigma}.
\end{align*}
Using this formula, one can see that by extending $\gamma'$ to a fixed vector field over $\mathbb R^n$, one can write $R(v, \gamma')\gamma'$ as
\begin{align*}
    R(v, \gamma')\gamma' & = D(g^{-1}Dg(\gamma')\gamma')(v) - D(g^{-1}Dg(
    \gamma')v)(\gamma')\\
    & + g^{-1}Dg(v)g^{-1}Dg(\gamma')\gamma' - g^{-1}Dg(\gamma')g^{-1}Dg(\gamma')v.
\end{align*}
But for the first term, using triangle inequality
\begin{align*}
    \|D(g^{-1}Dg(\gamma')\gamma')(v)\|_g^2
    & \lesssim \|g^{-1}Dg(v)g^{-1}Dg(\gamma')\gamma'\|^2
    + \|g^{-1}D^2g(\gamma',v)\gamma'\|^2\\
    & =\gamma'^TDg(\gamma')g^{-1}Dg(v)g^{-1}Dg(v)g^{-1}Dg(\gamma')\gamma'
    + \gamma'^TD^2g(\gamma',v)g^{-1}D^2g(\gamma',v)\gamma'\\
    & \leq \gamma_1^2\|v\|_g^2 \gamma'^T Dg(\gamma')g^{-1}Dg(\gamma')\gamma' + 
    \|v\|^2\|\gamma'\|^2\gamma'^T g \gamma'\\
    & \lesssim \gamma_1^4\|v\|_g^2\|\gamma'\|_g^4.
\end{align*}
Similarly for the second term:
\begin{align*}
    \|D(g^{-1}Dg(\gamma')v)(\gamma')\|^2 & \lesssim
    \|g^{-1}Dg(\gamma')g^{-1}Dg(\gamma')v\|^2
    + \|g^{-1}Dg(\gamma', \gamma')v\|^2\\
    & \lesssim \gamma_1^2\|\gamma'\|^2\|g^{-1}Dg(\gamma')v\|^2 + \gamma_3^4\|\gamma'\|^4\|v\|^2\\
    & \lesssim \gamma_1^4\|\gamma'\|^4\|v\|^2 + \gamma_3^4 \|\gamma'\|^4\|v\|^2.
\end{align*}
For the third term
\begin{align*}
    \|g^{-1}Dg(v)g^{-1}Dg(\gamma')\gamma'\|^2 \leq
    \gamma_1^4\|\gamma'\|^4\|v\|^2.
\end{align*}
For the forth term:
\begin{align*}
    \|g^{-1}Dg(\gamma')g^{-1}Dg(\gamma')v\|^2 & \leq
    \gamma_1^4\|\gamma'\|^4\|v\|^2,
\end{align*}
which completes the proof.
\end{proof}
\end{lemma}

\begin{lemma}\label{lem:sectionalcurvature}
For the sectional curvature, we have
\begin{align*}
    K(e_i, e_j) \lesssim (\gamma_1^4 + \gamma_3^4).
\end{align*}
\end{lemma}
\begin{proof}
Using Lemma~\ref{lem:rimanntensorbound}:
\begin{align*}
K(e_i, e_j) = \langle R(e_i, e_j)e_j, e_i\rangle \leq \|R(e_i, e_j)e_j\| \lesssim \gamma_1^4 + \gamma_3^4.    
\end{align*}
\end{proof}

\begin{lemma}\label{lem:Riccibound}
For the Ricci curvature, we have
\begin{align*}
    |Ricci(v,v)| \leq n(\gamma_1^2 + \gamma_3^2)\|v\|_g^2.
\end{align*}
\end{lemma}
\begin{proof}
Using the previous Lemma
\begin{align*}
|Ricci(v,v)| & = |\sum_i \langle R(v,e_i)e_i , v\rangle|\\
& \leq \sum_i |\langle R(v,e_i)e_i, v\rangle|\\
& \leq \sum_i \|R(v,e_i)e_i\|\|v\|\\
& \lesssim \sum_i (\gamma_1^2 + \gamma_3^2)\|v\|_g^2
= n(\gamma_1^2 + \gamma_3^2)\|v\|_g^2.
\end{align*}
\end{proof}

\subsection{Implication of the bound on the sectional curvature}

\begin{lemma}\label{lem:useofsectional}
For time at most 
\begin{align}
    t \leq \frac{1}{6\sqrt K + 8 \sqrt L_2}.\label{eq:ktcondition}
\end{align}
We have the bound
\begin{align}
    \E_{y \sim \rho_0} \logdet(J(\gamma_t)^{-1}(y)) \leq (3\sqrt K + 4\sqrt{L_2})nt.\label{eq:similard}
\end{align}
\end{lemma}
\begin{proof}
Lemma~\ref{lem:jacobibound} shows that the operator norm of $J(\gamma_t)^{-1}(y))$ is upper bounded by $\sqrt{1 - (3\sqrt K + 4\sqrt{L_2})t}$.
Now using the time condition~\eqref{eq:ktcondition}, we can write, using the inequality $\log(1-x) \geq - 2x$ for $x \leq \frac{1}{2}$:
\begin{align*}
    \E_{y \sim \rho_0} \logdet(J(\gamma_t)^{-1}(y)) \leq - \frac{n}{2}\log(1 - (3\sqrt K + 4\sqrt{L_2})t) \leq (3\sqrt K + 4\sqrt{L_2})nt.
\end{align*}
\end{proof}

\begin{lemma}\label{lem:jacobibound}
Under the assumption $$t \leq  1/\sqrt K,$$ for any unit vector $e \in T_{y}(\mathcal M)$, we have
 \begin{align*}
    \|J(\gamma_t)(e)\| \geq \sqrt{1 - (3\sqrt K + 4\sqrt{L_2})t}.
 \end{align*}
\end{lemma}
\begin{proof}
It is not hard to check that $J(\gamma_t)(e)$ is a Jacobi field $J(t)$ at time $t$ along the geodesic $\gamma_t(y)$ with initial condition $J(0) = e$ and $J'(0) = \nabla_{e}(\grad F)(y)$, for which we have using the $L_2$ gradient Lipschitz assumption
\begin{align*}
    &\|J(0)\| = \|e\| = 1,\\
    &\|J'(0)\| = \|\nabla_{e}(\grad F)(0)\| \leq  L_2.
\end{align*}

We write the Jacobi equation 
\begin{align*}
    J''(t) + R(J, \gamma')\gamma' = 0,
\end{align*}
Now we wish to lower bound $\|J\|_t$. For that, we use the Rauch comparison Theorem to conclude that $\|J(t)\| \geq \|\tilde J(t)\|$ where $\tilde J$ is another Jacobi field with $\|\tilde J(0)\| = \|J(0)\| = 1$, $\|\tilde J(0)'\| = \|J(0)'\|$ on another manifold $\tilde{\mathcal{M}}$ with sectional curvature $K$. Now we decompose $J$ into the part along $\gamma_t'$ and orthogonal to it:
\begin{align*}
    \tilde J(t) = \tilde J^{||}(t) + \tilde J^\perp(t),
\end{align*}
and let $\tilde J^{||}(0) = \theta_1 \gamma_{0}'$ and $\tilde J^\perp(t) = \theta_2 u$ for $\theta_1^2 + \theta_2^2 = 1$ and unit vector $e \perp \gamma_0'$.
Moreover, we also decompose the derivative $\tilde J'(0)$ into $\theta_3 \gamma_0' + \theta_4 w$ for unit vector $w \perp \gamma_0'$. Note that since $\nabla_{\gamma_0' \tilde J} = \nabla_{e}\grad F$, we have
\begin{align}
    \theta_3^2 + \theta_4^2 = \|\nabla_{\gamma_0' \tilde J}\|^2 = \|\nabla_{e}\grad F\|^2 \leq L_2^2.\label{eq:initialderivativebound}
\end{align}
It is well known that the Jacobi equation decomposes into the parallel and orthogonal parts:
\begin{align}
    & {\tilde J}^{\perp''}(t) + R({\tilde J}^\perp, \gamma')\gamma' = 0,\label{eq:jacobieq1}\\
    & {\tilde J}^{||''}(t) + R({\tilde J}^{||}, \gamma')\gamma' = 0.\label{eq:jacobieq}
\end{align}

On the other hand, because of having constant sectional curvature on $\tilde{\mathcal M}$, it is well known that the Riemann tensor simplifies into:
\begin{align*}
    R(X,Y)Z = C(\langle Y,Z\rangle X - \langle X, Z\rangle Y).
\end{align*}
Therefore, using the orthogonality of $\tilde J^\perp$ to $\gamma_t'$, the Jacobi equations along the curve and orthogonal to it decomposes. For the orthogonal part, Equation~\eqref{eq:jacobieq} becomes
\begin{align*}
    & 0 = \tilde {J^\perp}^{''} + K\tilde J^\perp,\\
    &\tilde {J^\perp}(0) = \theta_2 u,\\
    &\tilde {J^\perp}'(0) = \theta_4 w,\numberthis\label{eq:orthog}
\end{align*}
for $u,w \perp \gamma_0'$. In this case, it is easy to see that $\tilde {J^\perp}$ becomes scalar factor times the parallel transport of $u$ along $\gamma_t$, i.e. letting $u(t)$ be the parallel transport along $\gamma_t$ of $u$, for a scalar function $f$ we have
\begin{align*}
    \tilde {J^\perp} = f(t)u(t).
\end{align*}
Substituting in~\eqref{eq:orthog}:
\begin{align*}
    & 0 = f'' + Kf,\\
    & f(0) = \theta_2,\\
    & f'(0) = \theta_4,
\end{align*}
whose solution is
\begin{align}
    f(t) = \frac{\theta_4}{\sqrt K} \sin(\sqrt K t) + \theta_2 \cos(\sqrt K t).
\end{align}
Note that $\theta_1^2 + \theta_2^2 = 1$, so $\theta_1, \theta_2 \leq 1$.
We have
\begin{align*}
    f(t) \geq |\theta_2| \cos(\sqrt K t) - \frac{|\theta_4|}{\sqrt K} \sin (\sqrt K t) \geq |\theta_2| - |\theta_2|\sqrt K t - |\theta_2|(\sqrt K t)^2/2 - \frac{|\theta_4|}{\sqrt K} (\sqrt K t + (\sqrt K t)^2/2)\\
    \geq |\theta_2| - \frac{3}{2}|\theta_2|\sqrt K t - \frac{3}{2} |\theta_4| t = |\theta_2| - \frac{3}{2}(|\theta_2|\sqrt K + |\theta_4|) t.\numberthis\label{eq:firstlower}
\end{align*}
where we used the smoothness of the trigonomic functions and the assumption $t \leq 1/\sqrt K$.

On the other hand, in the tangential direction to the geodesic, the Jacobi equation becomes 
\begin{align*}
    & 0 = \tilde {J^{||}}^{''},\\
    &\tilde {J^{||}}(0) = \theta_1 \gamma_0',\\
    &\tilde {J^{||}}(0)' = \theta_3 \gamma_0'.
\end{align*}
Denoting $\tilde {J^{||}} = g(t)\gamma_t'$, the equation boils down to the following simple ODE:
\begin{align*}
    & 0 = g''(t),\\
    & g(0) = \theta_1,\\
    & g'(0) = \theta_3.
\end{align*}
Hence
\begin{align}
    g(t) = \theta_3 t + \theta_1 \geq |\theta_1| - |\theta_3|t t.\label{eq:secondlower}
\end{align}
Combining~\eqref{eq:firstlower} and~\eqref{eq:secondlower}, we have
\begin{align*}
    \|J(t)\|^2 & \geq (|\theta_2| - \frac{3}{2}(|\theta_2|\sqrt K + |\theta_4|) t)^2 + (|\theta_1| - |\theta_3|t)^2\\
    & \geq \theta_1^2 + \theta_2^2 - 3(|\theta_2|\sqrt K + |\theta_4|)t - 2|\theta_3|t \\
    & \geq 1 - (3\sqrt K + 4\sqrt{L_2})t.
\end{align*}
\end{proof}

\section{Euclidean smoothness to manifold smoothness}\label{app:euclideantomanifold}
The aim of this section is to translate Euclidean smoothness conditions into manifold ones. For this section, we assume a positive lower bound $\omega$ on the smallest eigenvalue of the metric. Let $F$ be a three times continuously differentiable function which is $L_1'$ Lipschitz,  $L_2'$ gradient Lipschitz, and $L_3'$ Hessian Lipshitz.

\begin{lemma}\label{lem:egradbound}
Function $F$ is $\frac{L_1'}{\sqrt\omega}$  lipshitz on the manifold, i.e.
\begin{align*}
    \|\grad F\|_g \leq \frac{L_1'}{\sqrt \omega}.
\end{align*}
\end{lemma}
\begin{proof}
\begin{align*}
    \|\grad F\|_g = \|g^{-1} DF\|_g \leq \|DF\|_{g^{-1}}\leq \frac{1}{\sqrt \omega}\|DF\|_2 \leq \frac{L_1'}{\sqrt \omega}.
\end{align*}
\end{proof}
\begin{lemma}\label{lem:ehessianbound}
Function $F$ is $\frac{L_1'\gamma_1}{\sqrt \omega} + \frac{L_2'}{\omega}$ gradient Lipshitz on the manifold, i.e.
\begin{align*}
    \|\nabla_{X_1} \grad F\| \leq (\frac{L_1'\gamma_1}{\sqrt \omega} + \frac{L_2'}{\omega})\|X_1\|.
\end{align*}
\end{lemma}
\begin{proof}
Given a vector field $X_1$, denoting $\grad F$ by $J$, again we have the conventional formula
\begin{align*}
    \nabla_{X_1}\grad F = DJ(X_1) + \frac{1}{2}g^{-1}Dg(X_1)J.
\end{align*}
For the second part, using self-concordance and the local coordinate formula for gradient
\begin{align}
    \frac{1}{2}\|g^{-1}Dg(X_1)J\| \leq \frac{1}{2}\gamma_1 \|X_1\|_g\|J\|_g = \frac{1}{2}\gamma_1 \|X_1\|_g \|DF\|_{g^{-1}} \leq \frac{1}{2}\|X_1\|_g \|DF\|_2 \frac{\gamma_1}{\sqrt \omega} \leq \frac{1}{2}\|X_1\|_g\frac{L_1'\gamma_1}{\sqrt \omega}.\label{eq:gradformula}
\end{align}
For the first part:
\begin{align}
    DJ(X_1) = D(g^{-1}DF)(X_1) = -g^{-1}Dg(X_1)g^{-1} DF + g^{-1}D^2F(X_1).\label{eq:buildupon}
\end{align}
The first part is bounded exactly as in~\eqref{eq:gradformula}. For the second part
\begin{align*}
    \|g^{-1}D^2F X_1\|_g & = \|g^{-1/2} g^{-1/2} D^2F g^{-1/2}g^{1/2}X_1\|_g\\
    & = \|g^{-1/2} D^2F g^{-1/2}g^{1/2}X_1\|_2 \\
    & \leq \|g^{-1/2}\|_{\text{op}} \|D^2F\|_{\text{op}} \|g^{-1/2}\|_{\text{op}} \|g^{1/2}X_1\|_2\\
    & \leq \frac{L_2'}{\omega}\|X_1\|_g.\numberthis\label{eq:similar}
\end{align*}
\end{proof}

\begin{lemma}\label{lem:hessianlip}
Function $F$ is $3\gamma_1^2 \frac{L_1'}{\sqrt \omega} + 2\gamma_1^2 \frac{L_1'}{\sqrt \omega} + 3\gamma_1\frac{L_2'}{\omega} + \frac{L_3'}{\omega\sqrt \omega} + \gamma_3^2\frac{L_1'}{\sqrt \omega}$ hessian Lipshitz on the manifold, i.e.
\end{lemma}
\begin{align*}
    \|\nabla^3F(X_1, X_2, X_3)\| \leq 3\gamma_1^2 \frac{L_1'}{\sqrt \omega} + 2\gamma_1^2 \frac{L_1'}{\sqrt \omega} + 3\gamma_1\frac{L_2'}{\omega} + \gamma'\frac{L_3'}{\omega\sqrt \omega} + \gamma_3^2\frac{L_1'}{\sqrt \omega}.
\end{align*}
\begin{proof}
Using the result of Lemma~\ref{lem:DtwoJlemma}, all of the terms in Equation~\eqref{eq:allterms} are bounded by $(\gamma_1L_2 + (\gamma_1^2 + \gamma_3^2) \|J\|)\|X_1\|\|X_2\|$, which with our current value of $L_2$ and our bounds on the gradient $J$ 

implies the coefficient
\begin{align}
     \gamma_1(\frac{L_1'\gamma_1}{\sqrt \omega} + \frac{L_2'}{\omega}) + (\gamma_1^2 + \gamma_3^2)\frac{L_1'}{\sqrt \omega}.\label{eq:primarybound}
\end{align}
Hence, it is enough to bound $\|D^2 J(X_1, X_2)\|_g$. To compute it, building upon~\eqref{eq:buildupon}, we can write assuming $DX_1(X_2) = 0$:
\begin{align*}
    D^2J(X_1, X_2) = D\Big[ -g^{-1}Dg(X_1)g^{-1} DF + g^{-1}D^2F(X_1)\Big](X_2). 
\end{align*}
For the first part, we have
\begin{align*}
    D( -g^{-1}Dg(X_1)g^{-1} DF)(X_2)
    & = g^{-1}Dg(X_2)g^{-1}Dg(X_1)g^{-1}DF -
    g^{-1}D^2g(X_1, X_2)g^{-1}DF\\
    & + g^{-1}Dg(X_1)g^{-1}Dg(X_2)g^{-1}DF -g^{-1}Dg(X_1)g^{-1}D^2 F X_2.
\end{align*}
For the second part:
\begin{align*}
  D(g^{-1}D^2F(X_1))(X_2) = -g^{-1}Dg(X_2)g^{-1}D^2F(X_1) + g^{-1}D^3F(X_1, X_2).
\end{align*}
Now we bound all of these terms one by one. For the first term, using Lemma~\ref{lem:egradbound}:
\begin{align}
    \|g^{-1}Dg(X_2)g^{-1}Dg(X_1)g^{-1}DF\|& \leq \|g^{-1}Dg(X_1)g^{-1}DF\|\gamma_1 \|X_2\|\\
    & \leq \|g^{-1}DF\|\gamma_1^2 \|X_1\|\|X_2\|\leq \gamma_1^2 \frac{L_1'}{\sqrt \omega}\|X_1\|\|X_2\|.\label{eq:efirstterm}
\end{align}

For the second term, using Lemma~\ref{lem:egradbound}
\begin{align}
   \|g^{-1}D^2g(X_1, X_2)g^{-1}DF\|_g \leq \|X_1\|\|X_2\|\gamma_3^2 \|DF\|_{g^{-1}} = \|X_1\|\|X_2\|\gamma_3^2 \|J\|_{g}\leq  \|X_1\|\|X_2\|\gamma_3^2 \frac{L_1'}{\sqrt \omega}.\label{eq:esecondterm}
\end{align}
Third term is similar to the first term. For the forth term, using the result of the derivation in~\eqref{eq:similar}:
\begin{align*}
    \|g^{-1}Dg(X_1)g^{-1}D^2 F X_2\|\leq \gamma_1\|X_1\|_g\|g^{-1}D^2FX_2\|_g \leq 
    \gamma_1\frac{L_2'}{\omega}\|X_1\|_g\|X_2\|_g.
\end{align*}
Fifth term is again the same as the forth term. Finally for the last term, using the same derivation as~\eqref{eq:similar} with $D^2F$ substituted by $D^3F(X_1)$, we get
\begin{align*}
    \|g^{-1}D^3F(X_1, X_2)\|& \leq \|D^3F(X_1)\|_{\text{op}} \frac{1}{\omega}\|X_2\|.
\end{align*}
Now using $L_3'$ hessian Lipschitz property of $F$:
\begin{align*}
  \|g^{-1}D^3F(X_1, X_2)\| \leq 
  L_3'\|X_1\|_2 \frac{1}{\omega}\|X_2\| \leq
  \frac{L_3'}{\omega\sqrt \omega}\|X_1\|\|X_2\|. 
\end{align*}
Hence, summing all these bounds with the bound in~\eqref{eq:primarybound}
\begin{align*}
     \|\nabla^3F(X_1, X_2, X_3)\|& \leq 2\gamma_1^2 \frac{L_1'}{\sqrt \omega} + \gamma_1^2 \frac{L_1'}{\sqrt \omega} + 2\gamma_1\frac{L_2'}{\omega} + \frac{L_3'}{\omega\sqrt \omega} \\
    & + \gamma_1(\frac{L_1'\gamma_1}{\sqrt \omega} + \frac{L_2'}{\omega}) + (\gamma_1^2 + \gamma_3^2)\frac{L_1'}{\sqrt \omega}\\
    & = 3\gamma_1^2 \frac{L_1'}{\sqrt \omega} + 2\gamma_1^2 \frac{L_1'}{\sqrt \omega} + 3\gamma_1\frac{L_2'}{\omega} + \frac{L_3'}{\omega\sqrt \omega} + \gamma_3^2\frac{L_1'}{\sqrt \omega},
\end{align*}
which completes the proof.
\end{proof}

As one can see above, to translate the Euclidean smoothness properties to manifold smoothness counterparts, the Euclidean Lipschitz parameter of $F$ enters the bounds, which is not ideal. We remind the reader that in section~\ref{lemma:derivativebound}, we translate the $L_2$ and $L_3$ properties into bounds on the norms of $DJ(X_1)$ and $DJ(X_1, X_2)$ in Lemmas~\ref{lemma:derivativebound} and~\ref{lem:DtwoJlemma}, which is then applied in our analysis. In the next two Lemmas, we directly show the argument in these Lemmas without the need of Euclidean Lipschitz property (i.e. without parameter $L_1'$).

\begin{lemma}
For the gradient vector field $J = \nabla F$, we have
\begin{align*}
    \|DJ(X_1)\| \leq (\frac{L_2'}{\omega} + \gamma_1\|J\|)\|X_1\|.
\end{align*}
\end{lemma}
\begin{proof}
We use the same derivation as in Lemma~\ref{lem:ehessianbound}:
\begin{align}
    DJ(X_1) = D(g^{-1}DF)(X_1) = -g^{-1}Dg(X_1)g^{-1} DF + g^{-1}D^2F(X_1).
\end{align}
For the first part, similar to there:
\begin{align*}
    \|g^{-1}Dg(X_1)J\| \leq \gamma_1 \|X_1\|_g\|J\|_g. 
\end{align*}
For the second part, we just reuse the bound in~\eqref{eq:similar}. Combining the two completes the proof.
\end{proof}

\begin{lemma}
For the gradient vector field $J = \nabla F$, we have
\begin{align}
\|D^2J(X_1, X_2)\| \leq ( \gamma_1\frac{L_2'}{\omega} + (\gamma_1^2 + \gamma_3^2) \|J\| + \frac{L_3'}{\omega\sqrt{\omega}})\|X_1\|\|X_2\|.  
\end{align}
\end{lemma}
\begin{proof}
We essentially use the derivation in Lemma~\ref{lem:hessianlip} for $DJ(X_1, X_2)$. For the first term(\ref{eq:efirstterm}), we use the same bound expect in the last step:
\begin{align*}
    \|g^{-1}Dg(X_2)g^{-1}Dg(X_1)g^{-1}DF\|
    & \leq \|g^{-1}DF\|\gamma_1^2 \|X_1\|\|X_2\|\leq \gamma_1^2 \|J\|\|X_1\|\|X_2\|.
\end{align*}
For the second term(\ref{eq:esecondterm}):
\begin{align*}
    \|g^{-1}D^2g(X_1, X_2)g^{-1}DF\|_g \leq \|X_1\|\|X_2\|\gamma_3^2 \|DF\|_{g^{-1}} = \|X_1\|\|X_2\|\gamma_3^2 \|J\|_{g}.
\end{align*}
Third term is similar to the first term. For the forth and fifth terms, we use the exact bound as in Lemma~\ref{lem:hessianlip}:
\begin{align*}
    & \|g^{-1}Dg(X_1)g^{-1}D^2 F X_2\|\leq 
    \gamma_1\frac{L_2'}{\omega}\|X_1\|_g\|X_2\|_g,\\
  & \|g^{-1}D^3F(X_1, X_2)\| \leq
  \frac{L_3'}{\omega\sqrt \omega}\|X_1\|\|X_2\|. 
\end{align*}
Combining these bounds completes the proof.
\end{proof}
The previous two Lemmas enables us to use Lemmas~\ref{lemma:derivativebound} and~\ref{lem:DtwoJlemma} with parameters $L_2 = L_2'/\omega$ and $L_3 = L_3'/(\omega\sqrt \omega)$.

\section{Log barrier geometry}\label{sec:polyope}
Here we highlight a few ideas needed to apply our main theorem to sampling functions supported on polytopes.
\subsection{Self concordance parameters}\label{sec:selfconcordantparams}
In this section, we show that the self concordant parameters of the log barrier inside a polytope are all constants. The parameter $\gamma_1 \le 2$ for the log barrier from standard self-concordance.

\begin{lemma}\label{lem:selfcon2}
For a polytope $Ax \geq b$ with the geometry of log barrier, we have $\gamma_2 = 4$, i.e.
\begin{align*}
     \big|\langle Dg(v), Dg^{-1}{w}\rangle\big| \leq 4\|v\|_g \|w\|_g.
\end{align*}
\end{lemma}
\begin{proof}
Let $A_x$ be the matrix $A$ whose $i$th row is reweighted by $1/(a_i^Tx - b_i)$. Moreover, let 
$s_v = (\frac{a_i^T v}{a_i^T x - b_i})_{i=1}^n$ and $S_v = \D(s_v)$. It is easy to check that 
\begin{align*}
    Dg(v) = -2A_x^T S_v A_x.
\end{align*}
Hence
\begin{align}
   \langle Dg(v), Dg^{-1}{w}\rangle & =
   -\text{Tr}(Dg(v) g^{-1} Dg(w) g^{-1})  \nonumber\\ 
   & = -4\text{Tr}(A_x^T S_v A_x g^{-1} A_x^T S_w A_x g^{-1})\\
   & = -4\text{Tr}(S_v P S_w P) \nonumber\\
   & = -4s_v^T P^{(2)} s_w,\label{eq:selfcon1}
\end{align}
where $P$ is the projection matrix of $A_x$ and $P^{(2)}$ is its Hadamard square. It is easy to check that $P^{(2)}$ is an unnormalized Laplacian with sum of its $i$th row equal to $\sigma_i$, i.e. the $i$th leverage score of $A_x$. As a result, we have
\begin{align}
    0 \leq P^{(2)} \leq I.\label{eq:Pupperbound}
\end{align}
Plugging this into~\eqref{eq:selfcon1}:
\begin{align*}
   \big|\langle Dg(v), Dg^{-1}{w}\rangle\big| \leq 4\big|s_v^T P^{(2)} s_w\big| \leq 4\sqrt{s_v^T P^{(2)} s_v} \sqrt{s_w^T P^{(2)} s_w} \leq 4\|s_v\|\|s_w\| = 4\|v\|_g \|w\|_g,
\end{align*}
\end{proof}
Similarly, we can bound $\gamma_3$.
\begin{lemma}\label{lem:gamma3}
For polytope $Ax \geq b$ with the geometry of log barrier, we have $\gamma_3 = 6$, i.e.
\begin{align*}
   - 6\|v|_g \|w\|_g g \preceq  D^2 g [v,w] \preceq   6\|v|_g \|w\|_g g.   
\end{align*}
\end{lemma}
\begin{proof}[Proof of Lemma~\ref{lem:gamma3}]
One can check
\begin{align}
    D^2 g[v,w] = 6A_x^T S_v S_w A_x.\label{eq:secondder}
\end{align}
On the other hand, note that 
\begin{align}
    & -\|s_v\|\|s_w\| I \leq S_vS_w \leq \|s_v\|\|s_w\| I.\label{eq:lowner1}
\end{align}
Combining~\eqref{eq:secondder} and~\eqref{eq:lowner1} completes the proof.
\end{proof}

We will also need to bound smoothness parameters in the manifold, and we will do so by utilizing both Lipshitzness of the given function's derivatives and self-concordance of the metric. For the particular case of uniform distribution in polytope with log barrier geometry, we investigate the Lipschitzness parameters in Section~\ref{sec:log-barrier-smoothness}.

\subsection{Smoothness}\label{sec:log-barrier-smoothness}

Next, note that for the uniform distribution in the polytope (uniform dist. in the Euclidean chart) we have $F(x) = \logdet(g(x))$. we investigate the Lipschitz properties of this function.

\begin{lemma}
The function $\logdet(g(x))$ over any $\gamma_1$ self-concordant Hessian manifold  has bounded gradient. Specifically
\begin{align*}
    \|\grad  \logdet(g(x))\|_g \leq \gamma_1 n\sqrt n.
\end{align*}
\end{lemma}
\begin{proof}
Let $g^{-1} = \sum u_i u_i^T$ be a Choleskey factorization of $g^{-1}$. Then,
\begin{align*}
\|\grad \logdet(g(x))\|_g &= \|g^{-1} D\log\det (g)\|\\
& = \|g^{-1} \langle g^{-1}, Dg\rangle\|\\
& = \|\langle g^{-1}, Dg\rangle\|_{g^{-1}}\\
& = \sqrt{\sum \langle g^{-1}, Dg(u_i)\rangle^2} \\
& \gamma_1 \sqrt{\langle g^{-1}, g\rangle^2\sum \|u_i\|^2}\\
& = \gamma_1 n\sqrt n.
\end{align*}
\end{proof}

\begin{lemma}\label{lem:gradlip}
For the geometry of log barrier in the polytope $Ax \geq b$ we have for an arbitrary $X_1 \in T_{x}(\mathcal M)$
\begin{align*}
    \|\nabla_{X_1} \grad \logdet(g(x))\|_g \leq (n(2\sqrt n + 1) + 18) \|X_1\|_g,
\end{align*}
which implies
\begin{align*}
    \text{Hess}(\logdet(g(x)))(X_1,X_2) \leq (n(2\sqrt n + 1) + 20) \|X_1\|_g\|X_2\|_g.
\end{align*}
Furthermore, we have
\begin{align*}
    \|DJ(X_1)\|_g \leq (n + 20)\|X_1\|_g.
\end{align*}
\end{lemma}
\begin{proof}
For brevity, we drop the $x$ indices and refer to $A_x$ as $A$.
Note that if we denote $\grad \logdet(g(x))$ by $J$, then
\begin{align*}
    \nabla_{X_1} J = DJ(X_1) + \frac{1}{2}g^{-1}Dg(J)X_1.
\end{align*}
Now for the first, note that 

using $\gamma_1$ self-concordance $g$ we already have
\begin{align*}
    \|g^{-1} Dg(J)X_1\| \leq \gamma_1 \|J\|\|X_1\|,
\end{align*}
which in the case of log barrier, because we also have bounded gradient 

implies
\begin{align*}
    \|g^{-1} Dg(J)X_1\| \leq 4 n\sqrt n \|X_1\|.
\end{align*}
Therefore, it is enough to bound the first term, i.e. $DJ(X_1)$. First, for the case of log barrier, we claim for the following is the representation of $\grad \logdet(g(x))$ in the Euclidean chart:
\begin{align*}
    J = \grad \logdet(g(x)) = g^{-1} A^T \sigma.\label{eq:Jrepresentation}
\end{align*}
This is because 
\begin{align*}
    \grad \logdet(g(x)) = g^{-1}\langle g^{-1}, Dg\rangle. 
\end{align*}
But note that 
\begin{align*}
    D_ig = A^T S_i A,
\end{align*}
where by $S_i$ we mean $S_{e_i}$. Hence
\begin{align*}
    \langle g^{-1}, D_ig\rangle = 
    \langle g^{-1}, A^T S_i A\rangle = 
    Tr(P S_i) = \sigma^T s_i, 
\end{align*}
where recall that $\sigma$ is the leverage score vector for $A_x$. On the other hand, note that $s_i$ is just the $i$th column of $A_x$, hence we have~\eqref{eq:Jrepresentation}. Next, note that
\begin{align}
    D(J)(X_1) = D(g^{-1}A^T \sigma)(X_1) & = -g^{-1}Dg(X_1)g^{-1}A^T \sigma + 
    g^{-1}A^T S_{X_1} \sigma\\
    & - g^{-1} A^T 4s_{X_1}.\sigma - g^{-1} A^T P^{(2)}s_{X_1}\\
    & = -g^{-1}Dg(X_1)g^{-1}A^T \sigma 
     -3 g^{-1} A^T S_{X_1}.\sigma - g^{-1} A^T P^{(2)}s_{X_1}.\label{eq:jderivative}
\end{align}
To see why the derivative of the leverage scores is as above, we take the derivative of the whole projection matrix:
\begin{align*}
    DP(X_1) & = D(A g^{-1}A^T)(X_1) = -2S_{X_1}A  g^{-1} A^T + -2A g^{-1}  A^TS_{X_1} -  A g^{-1} A^TS_{X_1}A g^{-1} A^T\\
    & = -2S_{X_1} P - 2PS_{X_1} - PS_{X_1}P.
\end{align*}
As a result, 
\begin{align}
    D\sigma(X_1) = -4S_{X_1}\sigma - P^{(2)}s_{X_1}.\label{eq:leveragederivative}
\end{align}

For the first term, using $4$th self concordance of log-barrier 
\begin{align*}
    \|g^{-1}Dg(X_1)g^{-1}A^T \sigma\|_g = 
    \|Dg(X_1)g^{-1}A^T \sigma\|_{g^{-1}} & \leq
    4 \|X_1\|_g\|A^T \sigma\|_{g^{-1}} = \|X_1\|_g \sqrt{\sigma^T P \sigma} \leq \|X_1\|_g \|\sigma\|_2 \\
    & \leq \|X_1\|_g \sqrt{\sum \sigma_i} = n \|X_1\|_g,
\end{align*}
where the last line is due to the property of the leverage scores.
For the second term
\begin{align*}
    3\|g^{-1}A^T S_{X_1} \sigma\|_g = 3\sqrt{Tr(\sigma^{T}S_{X_1}PS_{X_1}\sigma)} \leq 3\|\sigma . s_{X_1}\|_2 \leq 3\|s_{X_1}\| = 3\|X_1\|_g.
\end{align*}

For the third term,
\begin{align*}
    \|g^{-1} A^T P^{(2)}s_{X_1}\|_g = \sqrt{Tr((P^{(2)}s_{X_1})^TP P^{(2)}s_{X_1})} \leq \|P^{(2)}s_{X_1}\| \leq \|s_{X_1}\| = \|X_1\|_g. 
\end{align*}
where in the last line we used $P^{(2)} \leq I$. Combining all the above concludes the result.

\begin{lemma}\label{lem:hesslip}
For log barrier in the polytope $Ax \geq b$, we have for arbitrary $X_1,X_2 \in T_x(\mathcal M)$,
\begin{align*}
     |\nabla^3\ \mathrm{logdet}(g(x))(X_1,X_2,X_3)| \leq C_2 n\sqrt n \|X_1\|\|X_2\|\|X_3\|,
\end{align*}
which, if we assume $X_2(t)$ is the parallel transport vector field of $X_2$ along a curve with initial speed $X_1$, is equivalent to 
\begin{align*}
    \|\nabla_{X_1, X_2}\grad \logdet(g(x))\|\leq C_2 n\sqrt n \|X_1\|\|X_2\|.
\end{align*}
Moreover, we have
\begin{align*}
    D^2J(X_1, X_2) \leq C\sqrt n \|X_1\|\|X_2\|.
\end{align*}
\end{lemma}
\begin{proof}
Using our derivation 

all the other terms in the derivation in~\eqref{eq:dounblecovariantderivative} are bounded by $(\gamma_1L_2 + (\gamma_1^2 + \gamma_3^2) \|J\|)\|X_1\|\|X_2\|$, which for the case of log barrier is equal to $(2(n(\sqrt n + 1) + 20) + 80n\sqrt n)\|X_1\|\|X_2\| = (82n\sqrt n + 22)\|X_1\|\|X_2\|$, since we also have a bound on the norm of the gradient. Therefore, we only need to bound $D^2 J(X_1, X_2)$ (recall $J = \grad \logdet(g(x))$). Building upon~\eqref{eq:jderivative}, assuming $DX_2(X_1) = 0$, we need to compute
\begin{align}
    D^2J(X_1, X_2) = D(-g^{-1}Dg(X_1)g^{-1}A^T \sigma 
     -3 g^{-1} A^T S_{X_1}\sigma - g^{-1} A^T P^{(2)}s_{X_1})(X_2).\label{eq:logbarrierthird}
\end{align}
For the first term, 
\begin{align*}
    D(g^{-1}Dg(X_1)g^{-1}A^T \sigma)(X_2) & = -g^{-1}Dg(X_2)g^{-1}Dg(X_1)g^{-1}A^T \sigma\\
    & + g^{-1}Dg(X_1, X_2)g^{-1}A^T \sigma\\
    & - g^{-1}Dg(X_1)g^{-1}Dg(X_2)g^{-1}A^T\sigma\\
    & - g^{-1}Dg(X_1)g^{-1}A^T S_{X_2} \sigma\\
    & + g^{-1}Dg(X_1)g^{-1}A^T(-4S_{X_2} \sigma - P^{(2)}s_{X_2}),
\end{align*}
where for the last line we applied~\eqref{eq:leveragederivative}.
\end{proof}
Now for the first term above, using similar technique as before with self-concordance
\begin{align*}
   \|-g^{-1}Dg(X_2)g^{-1}Dg(X_1)g^{-1}A^T \sigma\|_g & \leq 2\|X_2\|\|g^{-1}Dg(X_1)g^{-1}A^T \sigma\|\\
   & \leq 4\|X_1\|\|X_2\|\|g^{-1}A^T \sigma\|\\
   & \leq 4\|X_1\|\|X_2\|\|P\sigma\|\\
   & \leq 4\|X_1\|\|X_2\|\|\sigma\|\\
    & \leq 4\|X_1\|\|X_2\|\sqrt n.
\end{align*}
For the second term, similar to first term
\begin{align*}
    \|g^{-1}Dg(X_1, X_2)g^{-1}A^T \sigma\| & \leq 6\|X_1\|\|X_2\|\|g^{-1}A^T \sigma\|\\
    & \leq 6\|X_1\|\|X_2\|\sqrt n.
\end{align*}
Third term is exactly as the first term. For the fourth term:
\begin{align*}
   \|g^{-1}Dg(X_1)g^{-1}A^T S_{X_2} \sigma\|_g & \leq \|X_1\|\|g^{-1}A^T S_{X_2} \sigma\|_g\\
   & \leq 2\|X_1\|\|P S_{X_2} \sigma\|_2\\
   & \leq 2\|X_1\|\|S_{X_2} \sigma\|_2\\
    & \leq 2\|X_1\|\|S_{X_2}\|_2 = 2\|X_1\|_g\|X_2\|_g.
\end{align*}
Fifth term is the same as fourth. For the seventh term:
\begin{align*}
    \|g^{-1}Dg(X_1)g^{-1}A^TP^{(2)}s_{X_2}\| & \leq 2\|X_1\|\|g^{-1}A^TP^{(2)}s_{X_2}\|\\
    & = 2\|X_1\|\|P P^{(2)}s_{X_2}\|_2\\
    & \leq 2\|X_1\|\|P^{(2)}s_{X_2}\|_2\\
    & \leq 2\|X_1\|\|P^{(2)}s_{X_2}\|_2\\
    & \leq 2\|X_1\|\|s_{X_2}\|_2\\
    & = 2\|X_1\|\|X_2\|.
\end{align*}
Next, differentiating the second term in~\eqref{eq:logbarrierthird}:
\begin{align*}
    D(g^{-1}A^TS_{X_1}\sigma)(X_2) & = 
    -g^{-1}Dg(X_2)g^{-1}A^TS_{X_1}\sigma\\
    & -2 g^{-1}A^T S_{X_2}S_{X_1}\sigma\\
    & + g^{-1}A^TS_{X_1}(-4S_{X_2}\sigma - P^{(2)}s_{X_2}).
\end{align*}
For the first term, we have
\begin{align*}
    \|g^{-1}Dg(X_2)g^{-1}A^TS_{X_1}\sigma\|_g \leq 2\|X_2\|\|g^{-1}A^TS_{X_1}\sigma\|& = 2\|X_2\|\|PS_{X_1}\sigma\|_2\\
    & \leq 2\|X_2\|\|S_{X_1}\sigma\|_2\\
    & \leq 2\|X_2\|\|X_1\|.
\end{align*}
For the second term
\begin{align*}
   \|g^{-1}A^T S_{X_2}S_{X_1}\sigma\| & = \|P S_{X_2}S_{X_1}\sigma\|_2 \\
   & \leq \|S_{X_2}S_{X_1}\sigma\|_2\\
   & \leq \|S_{X_2}s_{X_1}\|_2 \\
   & \leq \|s_{X_1}\|_2\|s_{X_2}\|_2 = \|X_1\|\|X_2\|.
\end{align*}
Third term is exactly similar to the second term. For the forth term
\begin{align*}
    \|g^{-1}A^TS_{X_1}P^{(2)}s_{X_2}\|
    & = \|PS_{X_1}P^{(2)}s_{X_2}\|_2\\
    & \leq \|S_{X_1}P^{(2)}s_{X_2}\|_2\\
    & \leq \|s_{X_1}\|\|P^{(2)}s_{X_2}\|\\
    & \leq \|s_{X_1}\|_2\|s_{X_2}\|_2\\
    & = \|X_1\|\|X_2\|,
\end{align*}
where we used $P^{(2)} \leq I$.
Finally, differentiating the third term in~\eqref{eq:logbarrierthird}:
\begin{align*}
    D(g^{-1} A^T P^{(2)}s_{X_1})(X_2)
    & = -g^{-1}Dg(X_2)g^{-1}A^T P^{(2)}s_{X_1}\\
    & - g^{-1}A^TS_{X_2} P^{(2)}s_{X_1}\\
    & + 2g^{-1}A^T\big(P \odot (-2S_{X_2}P - 2PS_{X_2} - PS_{X_2}P)\big)s_{X_1}\\
    & - g^{-1}A^T P^{(2)}S_{X_2}s_{X_1}.
\end{align*}
For the first term
\begin{align*}
   \|g^{-1}Dg(X_2)g^{-1}A^T P^{(2)}s_{X_1}\| & \leq 2\|X_2\|\|g^{-1}A^T P^{(2)}s_{X_1}\|\\
   & \leq 2\|X_2\|\|P P^{(2)}s_{X_1}\|_2\\
   & \leq 2\|X_2\|\|P^{(2)}s_{X_1}\|_2\\
   & \leq 2\|X_2\|\|s_{X_1}\|_2\\
    & =2\|X_2\|\|X_1\|.
\end{align*}
For the second term
\begin{align*}
    \|g^{-1}A^TS_{X_2} P^{(2)}s_{X_1}\|& \leq \|PS_{X_2} P^{(2)}s_{X_1}\|_2\\
    & \leq \|S_{X_2} P^{(2)}s_{X_1}\|_2\\
    & \leq \|s_{X_2}\|_2\|P^{(2)}s_{X_1}\|_2\\
    & \leq \|X_2\|\|X_1\|.
\end{align*}
For the third term,
\begin{align*}
    \|g^{-1}A^T\big(P \odot S_{X_2}P\big)s_{X_1}\|& \leq 
    \|P\big(P \odot S_{X_2}P\big)s_{X_1}\|_2\\
    & \leq \|\big(P \odot S_{X_2}P\big)s_{X_1}\|_2\\
    & = \|S_{X_2}P^{(2)}s_{X_1}\|_2\\
    & \leq \|s_{X_2}\|_2\|P^{(2)}s_{X_1}\|_2\\
    & \leq \|s_{X_2}\|_2\|s_{X_1}\|_2 = \|X_1\|\|X_2\|.
\end{align*}
For the fifth term
\begin{align*}
    \|g^{-1}A^T\big(P \odot PS_{X_2}\big)s_{X_1}\|
    & = \|P(P \odot PS_{X_2})s_{X_1}\|_2\\
    & \leq \|P^{(2)}S_{X_2}s_{X_1}\|\\
    & \leq \|S_{X_2} s_{X_1}\|\leq \|X_1\|\|X_2\|.
\end{align*}
For the last term,
\begin{align}
    \|g^{-1}A^T\big(P \odot PS_{X_2}P\big)s_{X_1}\|& \leq
    \|P\big(P \odot PS_{X_2}P\big)s_{X_1}\|_2 \nonumber\\
    &\leq \|\big(P \odot PS_{X_2}P\big)s_{X_1}\|_2 \nonumber\\
    & = \sqrt{s_{X_1}^T \big(P \odot PS_{X_2}P\big)^2 s_{X_1}}.\label{eq:lastterm}
\end{align}
But note that
\begin{align*}
 -\|S_{X_2}\|_2P \leq PS_{X_2}P \leq \|S_{X_2}\|_2P.
\end{align*}
On ther other hand, using the Schur product theorem, we know that if $A \geq 0$ and $B \geq C$, then
$A \odot B \geq A \odot C$. Applying above, we have
\begin{align*}
    -\|s_{X_2}\|_2P^{(2)} \leq P \odot (PS_{X_2}P) \leq \|s_{X_2}\|_2P^{(2)},
\end{align*}
which implies 
\begin{align*}
    -\|s_{X_2}\|_2I \leq P \odot (PS_{X_2}P) \leq \|s_{X_2}\|_2I.
\end{align*}
This in turn implies
\begin{align*}
    (P \odot (PS_{X_2}P))^2 \leq \|X_2\|^2 I.
\end{align*}
Applying this to~\eqref{eq:lastterm}:
\begin{align*}
    \|g^{-1}A^T\big(P \odot PS_{X_2}P\big)s_{X_1}\| \leq \|X_2\| \|s_{X_1}\|_2 = \|X_2\|\|X_2\|.
\end{align*}

\end{proof}
Combining all these equations and applying them to~\eqref{eq:logbarrierthird}, we have
\begin{align*}
    D^2J(X_1, X_2) \leq C\sqrt n \|X_1\|\|X_2\|,
\end{align*}
for some universal constant $C$. Therefore, we showed
\begin{align*}
    |\nabla^3 \ \logdet(g(x))(X_1,X_2)| \leq C_2 n\sqrt n \|X_1\|\|X_2\|,
\end{align*}
for some universal constant $C_2$.

\subsection{Smallest eigenvalue of the log barrier metric}
Here, we show that given a bound on the smallest eigenvalue of the metric imposed by the log barrier, given that the polytope $Ax \geq b$ is inside a ball of radius $R$.
\begin{lemma}\label{lem:diameter}
For polytope $Ax \geq b$ with diameter $R$, we have that the smallest eigenvalue of the metric is at least $1/R$.
\end{lemma}
\begin{proof}
It is enough to show that for any $x$ inside the polytope and unit vector $v$ we have $v^T A_x^T A_x v \geq \frac{1}{R^2}$, or equivalently
\begin{align*}
    \sum_i \Big(\frac{a_i^T v}{b_i - a_i^T x}\Big)^2 \geq \frac{1}{R^2}. 
\end{align*}
Let $I_i(x,v)$ be the intersection point of the line passing through $x$ in direction $v$ with the facet $a_i^Tx =b$. Note that the quantity $r_i = \frac{b_i - a_i^T x}{a_i^T v}$ is nothing but the length of the segment $[x,I_i(x,v)]$. Hence
\begin{align*}
    \sum_i \Big(\frac{a_i^T v}{b_i - a_i^T x}\Big)^2 = \sum_i \frac{1}{r_i^2} \geq \frac{1}{\min_i r_i^2} \geq \frac{1}{R^2},
\end{align*}
where the last inequality follows from the fact that $R$ is the diameter of the polytope.
\end{proof}

In the following lemma, we check the strong 1-completeness assumption for the log barrier geometry:
 \begin{lemma}\label{lem:strong1completeness}
Given a polytope $Ax \geq b$ with the log barrier geometry, the heat diffusion defined in~\eqref{eq:coresde} is strongly 1-complete.
 \end{lemma}
 
\begin{proof}
First, we check if the SDE~\eqref{eq:coresde} is complete. Using our lower bound for the Ricci curvature of the log barrier. The first implication of this is that the Brownian motion is non-explosive. Moreover, using Theorem 8.2 in~\cite{li1994strong} (note the the drift vector $Z$ there is zero in our case, since~\eqref{eq:coresde} is the pure Brownian motion on manifold), we see that~\eqref{eq:coresde} is complete. Next, we check if it is strongly 1-complete. Using Lemma A.2 in~\cite{elworthy1994formulae}, we see that for strongly 1-completeness under the completeness assumption for the SDE, it is enough we show $H_1(v,v) \leq c\|v\|^2$ for the tensor $H_1$. Writing the simplified formula of $H_1$ for the SDE in~\ref{eq:coresde}:
\begin{align*}
    H_1(v,v) & = -Ric(v,v) + \sum_i \|Dg^{i,}(v)\|_g^2 - \sum_i \frac{1}{\|v\|^2}\langle \nabla g^{i,}(v), v\rangle^2\\
    & \leq -Ric(v,v) + Tr(Dg^{-1}(v)^T g (Dg^{-1}(v))) + Tr(g^{-1}Dg(v)g^{-1} g g^{-1}Dg(v)g^{-1}) \\
    & \leq -Ric(v,v) + Tr(g^{-1/2}(g^{1/2}Dg^{-1}(v) g^{1/2})^2g^{-1/2}) + Tr(g^{-1/2}(g^{-1/2}Dg(v)g^{-1/2})^2g^{-1/2})\\
    & \leq -Ric(v,v) + n\|v\|_g^2 /\lambda_{min}(g) + n\|v\|_g^2/\lambda_{min}(g).
\end{align*}
But the Ricci curvature is bounded from below by Lemma~\ref{lem:Riccibound}, and $1/\lambda_{min}(g^{-1})$ is bounded from  above using Lemma~\ref{lem:diameter}, which shows the desired upper bound on $H_1$ and completes the proof.
\end{proof}

\subsection{Sampling from the Gibbs distribution restricted to a polytope}
In this section, we compute our complexity for sampling from the Gibbs distribution $e^{-f}$ inside the polytope.

\begin{lemma}\label{lem:polytopecomplexity1}
Let $\nu$ be the distribution $e^{-f}dx$ restricted to the polytope $Ax \geq b$, where $f$ is $L_2$ gradient Lipshcitz and $L_3$ Hessian Lipschitz with respect to the geometry of the log barrier. Then,  our algorithm reaches accuracy $\delta$ in KL divergence after $k = O(\log(2H_\nu(\rho_0)/\delta)(\alpha\epsilon))$ number of iterations, for $\epsilon$ as small as 
\begin{align*}
     \epsilon \lesssim 
    \frac{\delta \alpha}{ n^{5/2}L_2
     + \sqrt nL_2^2 + nL_3 }.
\end{align*}
\end{lemma}

\begin{proof}
Directly from Theorem~\ref{thm:generalmanifold} and Lemmas~\ref{lem:selfcon2} and~\ref{lem:gamma3}.
\end{proof}

\subsection{Sampling with Euclidean smoothness coefficients}
To obtain the complexity of sampling $e^{-f}dx = d^{-F}dv_g(x)$, we translate the Euclidean smoothness coefficients of $f$ to manifold smoothness coefficients
and then use Lemma~\ref{lem:polytopecomplexity1}. 

\begin{lemma}
Suppose we want to sample from the distribution $d\nu(x) = e^{-f} d(x)$ for smooth $f$ inside the polytope $Ax \geq b$ with diameter $R$, such that $f$ is $\ell_1$ lipschitz, $\ell_2$ gradient lipschitz, and $\ell_3$ hessian lipschitz. Moreover, if $\nu$ has satisfies a logarithmic sobolev inequality with constant $\alpha$ with respect to the log barrier geometry, then the query complexity of RLA to reach $\delta$ accuracy in KL divergence starting from density $\rho_0$ is $ O(\frac{1}{\alpha \epsilon}\log(2H_\nu(\rho_0)/\delta))$, with step size as small as

\begin{align*}
    \epsilon \lesssim \frac{\delta\alpha}{n^{7/2} + n^{5/2}\ell_2R + \sqrt n \ell_2^2 R^2 + n\ell_3 R^{3/2}}.
\end{align*}
\end{lemma}
\begin{proof}
First, note that from Lemma~\ref{lem:diameter}, we know that the smallest eigenvalue of the log barrier metric in a polytope with diameter $R$ is at least $1/R$. Hence,  using Lemmas~\ref{lem:ehessianbound} and~\ref{lem:hessianlip}, we get that Lemmas~\ref{lemma:derivativebound} and~\ref{lem:DtwoJlemma} hold with $L_2$ and $L_3$ replaced by $L_2'/\omega = O(\ell_2 R)$ and $L_3/(\omega\sqrt{\omega}) = O(\ell_3 R\sqrt R)$, respectively. 

Next, applying Lemmas~\ref{lem:gradlip} and~\ref{lem:hesslip},  we see that for $F = f + \frac{1}{2}\log \det g$, $J = \nabla F$,
we have
\begin{align*}
    &\|DJ(X_1)\| \leq \|D(\nabla f)(X_1)\| + \|D(\nabla \logdet(g)))(X_1)\| \lesssim (n + \ell_2 R + \gamma_1\|J\|)\|X_1\| = (n + \ell_2 R + \|J\|)\|X_1\|,\\
    &\|DJ(X_1, X_2)\| \leq \|D(\nabla f)(X_1, X_2)\| + 
    \|D(\nabla \logdet(g))(X_1, X_2)\| \lesssim (\sqrt n + \gamma_1\ell_2 R + \ell_3R\sqrt R + (\gamma_1^2 + \gamma_3^2)\|J\|)\|X_1\|\|X_2\|\\
    & \leq (\sqrt n + \ell_2 R + \ell_3R\sqrt R + \|J\|)\|X_1\|\|X_2\|,
\end{align*}
which means that Lemmas~\ref{lemma:derivativebound} and~\ref{lem:DtwoJlemma} hold with 
\begin{align}
&L_2 = O(n + \ell_2 R),\\
&L_3 = O(\ell_3 R\sqrt R).\label{eq:parameters}
\end{align}

In order to bypass the need for using Lemmas~\ref{lem:egradbound},~\ref{lem:ehessianbound}, and~\ref{lem:hessianlip}, hence bypassing the need for Euclidean lipschitzness, we claim that the only thing we need out of the $L_2$ and $L_3$ Lipschitz parameters for our analysis more or less boils down to Lemmas~\ref{lemma:derivativebound} and~\ref{lem:DtwoJlemma}. This is almost the case except that a couple of places in the proof should slightly change. We list these parts in order below. In this part until the rest of the proof of this Lemma, we assume $L_2$ and $L_3$ are the paraemters defined in~\eqref{eq:parameters} which makes Lemmas~\ref{lemma:derivativebound} and~\ref{lem:DtwoJlemma} hold. 

First, Lemma~\ref{lem:repeatedterm} will hold with a slightly worst bound. Namely, condition~\ref{eq:firstoff} holds with RHS $O(L_2 + \gamma_1\|J\|)$. This will add an additional $\sqrt n$ factor behind the second term $\gamma_1^2n\|J\|$ in Equation~\eqref{eq:lemmaresult1}. As a result of this, the bounds in Equations~\eqref{eq:combined2} and~\eqref{eq:4term} slightly change. On the other hand, the bound in Equation~\eqref{eq:2term} will have an additional factor of $\gamma_1^2 n\sqrt n\|J\|_g^2$. Overall, this and the change in Equation~\eqref{eq:4term} imply an additional term of $n\sqrt n \gamma_1^2$ in the definition of $\beta_2$ in~\ref{eq:finitevariation}, which in turn appear in the bound $\epsilon \leq 1/\xi$ Lemma~\ref{lem:Jbound}. On the other hand, the bound in Equation~\eqref{eq:term4} as well as Equation~\eqref{eq:combined2} will acquire an additional term of $n\sqrt n \gamma_1^2 \|V-J\|\|J\|$. This implies an additional term of $n\sqrt n \gamma_1^2 \|J_0\| + n\sqrt n\gamma_1^2 \sqrt{2(\beta_0+\beta_1)t}$, which in turn implies the same additional term in $\kappa_0'$ and the term $n\sqrt n \gamma_1^2 \sqrt{(\beta_0 + \beta_1)t}$ in $\kappa_2'$. This implies an extra factor of the term $n\sqrt n \gamma_1^2 \sqrt{(\beta_0 + \beta_1)t}$ in $\kappa_2$, $\kappa_0$, and $\bar \omega$. The change in $\kappa_2$ appears in condition $t \leq 1/(2\kappa_2)$ in~\eqref{eq:secondcondition}, which implies an additional conditional on $\epsilon$, namely 
\begin{align}
    \epsilon\lesssim n(\gamma_1^2 + \gamma_3^2)^{-2/3}(nL_2^2)^{-1/3},\label{eq:epscondition}
\end{align}
which comes from the second condition in~\eqref{eq:conditions}. It turns out the other additional factors does not change the order of terms, and so ultimately we end of with the change on the $\epsilon$ condition that we described in~\eqref{eq:epscondition}. 
 
 Next, we calculate the above additional factor for the case of interest here, namely the log barrier. From Lemma~\ref{lem:polytopecomplexity1}, we figure out the new condition on $\epsilon$ to reach accuracy $\delta$ can be further upper bounded as (note that we have assumed $\alpha, \delta \leq 1$)
 \begin{align*}
     \epsilon & \lesssim \min\{ \frac{\delta \alpha}{n^{5/2}L_2 + \sqrt n L_2^2 + nL_3} , \frac{1}{n(\gamma_1^2 + \gamma_3^2)^{2/3}(nL_2^2)^{1/3}}\} \lesssim \frac{\delta \alpha}{n^{5/2}L_2 + \sqrt n L_2^2 + nL_3 + n^{4/3}L_2^{2/3}}\\
     & \lesssim \frac{\delta \alpha}{n^{5/2}L_2 + \sqrt n L_2^2 + nL_3}.
 \end{align*}
 where we used the fact that $\gamma_1, \gamma_2, \gamma_3$ are constants for the case of log barrier.
 Now substituting the desired values of $L_2$ and $L_3$ from Equation~\eqref{eq:parameters}:
 \begin{align*}
     \epsilon \leq \frac{\delta\alpha}{n^{7/2} + n^{5/2}\ell_2R + \sqrt n \ell_2^2 R^2 + n\ell_3 R^{3/2}}.
 \end{align*}

\end{proof}

\section{Discussion}

We highlight two future research directions. 
\begin{enumerate}
    \item The first is whether there is a simple modulation method that allows one to convert LSI in one metric to LSI in another metric, e.g., by smoothly and rapidly decaying the target density for points ``near" constraints (i.e., with high values of the barrier function). A potential approach is multiplying the target density by a fixed density that is almost constant inside the polytope and rapidly decaying close to the boundary. One natural candidate for the fixed density in the log barrier case is $e^{-e^{\phi}}$, where $\phi$ is the log barrier. In the appendix, we verify that this distribution satisfies a log-Sobolev inequality. \cite{nesterov2002riemannian} gives several examples self-concordant Riemannian metrics which can also be adapted. 
    \item The second is the following algorithmic problem: Given a metric $g$, an initial point $x_0$ and a time interval $t$, sample from the distribution of Brownian motion on the manifold starting at $x_0$ for time $t$.  The question of sampling according to the natural manifold measure is a special case of this Brownian increment sampling problem for a sufficiently large time increment. 
\end{enumerate}

\noindent{\bf Acknowledgements.} We are grateful to Andre Wibisono, Ruoqi Shen, Xue-Mei Li and Sinho Chewi for helpful remarks. This work was supported in part by NSF awards CCF-2007443 and CCF-2134105.

\bibliographystyle{alpha}
\bibliography{sample}

\section*{Appendix: Log-Sobolev Inequality for log barrier metric}\label{app:fastdecaylogsobolev}
\begin{lemma}
Define the probability measure $\nu$ with support $(0,1)$ with density w.r.t to the Lebesgue measure as $d\nu(x) \coloneqq p(x)dx \sim \frac{dx}{e^{1/(x(1-x))}}$. Then, $\nu$ satisfies a Logarithmic Sobolev Inequality with respect to the log barrier with constant $\Omega(1)$.
\end{lemma}
\begin{proof}
 It is enough to check LSI for intervals of the form $S \coloneqq (0,y)$. We have
\begin{align*}
    \int_{0}^y \frac{1}{e^{1/(y(1-y))}} \lesssim \frac{y^2}{e^{1/(y(1-y))}} \sim y^2 p(y).
\end{align*}
(The notation $\sim$ implies equality up to constants).
On the other hand, we write the definition of the log-Sobolev constant using the $\epsilon$ neighborhood of the set $S$:
\begin{align*}
    \sqrt{\alpha} = \sup_{S} \lim_{\epsilon \rightarrow 0} \frac{\nu(S^{+\epsilon})}{d\nu(S)\sqrt{\log(1/\nu(S))}},
\end{align*}
where $S^{+d}$ is the set of points outside $S$ within distance at most $d$ of it.
Now since the metric at point $y$ is $g(y) = 1/y^2$, we get that the length of the interval $S^{+d}$ is $\Theta(y)$, hence
\begin{align}
    \nu(S^{+d}) = \Theta(\epsilon y p(y)).
\end{align}
Overall, this implies
\begin{align*}
    \sqrt{\alpha} = \Omega(\frac{\epsilon y p(y)}{\epsilon y^2p(y) \sqrt{\log(y^2p(y))}}).
\end{align*}
Also note that 
\begin{align*}
    \log(1/(y^2p(y))) \lesssim \log(1/y) + \log(1/e^{1/(y(1-y))}) \lesssim \log(1/y) + 1/y \lesssim 1/y,
\end{align*}
which means
\begin{align*}
    \alpha = \Omega(1).
\end{align*}
\end{proof}

\end{document}